\documentclass[10pt,journal,compsoc]{IEEEtran}

\usepackage{cite}

\ifCLASSINFOpdf
\else
\fi
\usepackage{amsmath}
\usepackage{amsthm}
\usepackage{mathtools}
\usepackage{bm}
\usepackage{bbm}
\usepackage{algorithm}
\usepackage{algorithmic}
\usepackage{amssymb}  
\usepackage{amsfonts}  
\usepackage{algorithmic}
\usepackage{algorithm}
\usepackage{subcaption}
\usepackage{longtable}
\usepackage{array}

\usepackage{url}

\usepackage{pgfplots}
\usepackage{tikz} 
\usepackage{bigints}
\usepackage{blindtext}

\pgfplotsset{compat=newest}
\usetikzlibrary{decorations.pathmorphing} % noisy shapes
\usetikzlibrary{fit} % fitting shapes to coordinates
\usetikzlibrary{backgrounds} % drawing the background after the foreground
\usetikzlibrary{pgfplots.groupplots}
\usetikzlibrary{positioning}
\usepackage{amsfonts}       % blackboard math symbols
\usepackage{wrapfig}
\usepackage{tabularx}
\usepackage{array,multirow}
\usepackage{colortbl}  
\newcolumntype{L}{>{\raggedright\arraybackslash}X}

\newcommand{\myalg}{NOPG}
\newcommand*\de{\mathop{}\!\mathrm{d}}

\DeclareMathOperator*{\EV}{\mathbb{E}}

\newcommand{\gradt}[0]{\nabla_{\theta}}

\newcommand{\Aset}[0]{\mathcal{A}}
\newcommand{\Sset}[0]{\mathcal{S}}
\newcommand{\svec}[0]{\mathbf{s}}
\newcommand{\xvec}[0]{\mathbf{x}}
\newcommand{\avec}[0]{\mathbf{a}}
\newcommand{\rvec}[0]{\mathbf{r}}
\newcommand{\qvec}[0]{\mathbf{q}_{\pi}}

\newcommand{\phivec}[0]{\bm{\phi}}
\newcommand{\pis}[0]{\pi_{\theta}}

\newcommand{\approxx}[1]{\hat{#1}}
\newcommand{\grad}[1]{\nabla_{#1}}
\newcommand{\kerres}{\bm{\Lambda}_{\pi}}
\newcommand{\pargrad}[1]{\frac{\partial}{\partial {#1}}}

\newcommand{\state}[0]{\mathbf{s}}
\newcommand{\action}[0]{\mathbf{a}}
\newcommand{\nextstate}[0]{\mathbf{s}'}
\newcommand{\Trans}[3]{p({#1}|{#2},{#3})}

\newcommand{\defeq}[0]{\coloneqq}

\newcommand{\bvec}{\mathbf{b}}
\newcommand{\zvec}{\mathbf{z}}

\newcommand{\hvec}{\mathbf{h}}
\newcommand{\deltavec}{\bm{\delta}}

\newcommand{\bigO}{\mathcal{O}}

\colorlet{darkgreen}{green!50!black!100!}

\newcommand{\reviewA}[1]{#1}	% Reviewer 1
\newcommand{\reviewB}[1]{#1}  % Reviewer 2
\newcommand{\reviewC}[1]{#1} % Reviewer 3
\newcommand{\reviewD}[1]{#1} % Reviewer 4
\newcommand{\reviewAll}[1]{#1} % More than one reviewer

\DeclareMathOperator\erf{erf}
\newcommand{\e}[1]{e^{\textstyle #1}}

\newtheorem{definition}{Definition}
\newtheorem{theorem}{Theorem}

\newtheorem{proposition}{Proposition}

\begin{document}
	\newlength\figureheight
	\newlength\figurewidth
	
	\setlength\figureheight{5.25 cm}
	\setlength\figurewidth{0.64 \columnwidth}
\title{Batch Reinforcement Learning with a Nonparametric Off-Policy Policy Gradient}

\author{Samuele~Tosatto,
        Jo\~ao~Carvalho
        and~Jan~Peters% <-this % stops a space
\IEEEcompsocitemizethanks{\IEEEcompsocthanksitem Samuele Tosatto is with the University of Alberta, Edmonton, Alberta, Canada. Department of Computing Science. E-mail: tosatto@ualberta.ca. 
\IEEEcompsocthanksitem Jo\~{a}o Carvalho and Jan Peters are with the Technische Universit\"at Darmstadt, ¨
Darmstadt, Germany, FG Intelligent Autonomous Systems. E-mail: \{name\}.\{surname\}@tu-darmstadt.de
}% <-this % stops an unwanted space
}


\markboth{Submitted to IEEE Transaction on Pattern Analysis and Machine Intelligence}%
{Shell \MakeLowercase{\textit{et al.}}: Bare Demo of IEEEtran.cls for Computer Society Journals}



\IEEEtitleabstractindextext{%
\begin{abstract}
% 1> What is our problem?
Off-policy Reinforcement Learning (RL) holds the promise of better data efficiency as it allows sample reuse and potentially enables safe interaction with the environment.
Current off-policy policy gradient methods either suffer from high bias or high variance, delivering often unreliable estimates.
% 2> Why is this important?
The price of inefficiency becomes evident in real-world scenarios such as interaction-driven robot learning, where the success of RL has been rather limited, and a very high sample cost hinders straightforward application.
% 3> Our suggestion to solve this!
In this paper, we propose a nonparametric Bellman equation, which can be solved in closed form. The solution is differentiable w.r.t the policy parameters and gives access to an estimation of the policy gradient.
% 4> Why does our contribution help
In this way, we avoid the high variance of importance sampling approaches, and the high bias of semi-gradient methods. 
We empirically analyze the quality of our gradient estimate against state-of-the-art methods, and show that it outperforms the baselines in terms of sample efficiency on classical control tasks.

\end{abstract}

}

\maketitle

\IEEEdisplaynontitleabstractindextext

% \fi
\IEEEpeerreviewmaketitle

\IEEEraisesectionheading{\section{Introduction}\label{sec:introduction}}

\IEEEPARstart{R}{einforcement learning} has made overwhelming progress in recent years, especially when applied to board and computer games, or simulated tasks \cite{mnih_human-level_2015,haarnoja_soft_2018,schulman_trust_2015}.
 However, in comparison, only a little improvement has been achieved on real-world tasks. 
One of the reasons of this gap is that the vast majority of reinforcement learning approaches are on-policy. On-policy algorithms require that the samples are collected using the optimization policy; and therefore this implies that a) there is little control on the environment and b) samples must be discarded after each policy improvement, causing high sample inefficiency.
In contrast, off-policy techniques are theoretically more sample efficient, because they decouple the procedures of data acquisition and policy update, allowing for the possibility of sample-reuse. Furthermore, off-policy estimation enables offline (or batch) reinforcement learning, which means that the algorithm extracts the optimal policy from a fixed dataset. This property is crucial in many real-world applications, since it allows a decoupled data-acquisition process, and, subsequently, a safer interaction with the environment. However, classical off-policy algorithms like Q-learning with function approximation and its offline version, fitted Q-iteration  \cite{ernst_tree-based_2005,riedmiller_neural_2005}, are not guaranteed to converge \cite{baird_residual_1995,lu_non-delusional_2018}, and allow only discrete actions. More recent semi-gradient\footnote{We adopt the terminology from \cite{imani_off-policy_2018}.} off-policy techniques, like Off-Policy Actor Critic (Off-PAC) \cite{degris_off-policy_2012-1}, Deep Deterministic Policy Gradient (DDPG) \cite{silver_deterministic_2014,lillicrap_continuous_2016}, Soft Actor Critic (SAC) \cite{haarnoja_soft_2018}, often perform sub-optimally, especially when the collected data is strongly off-policy, due to the biased semi-gradient update \cite{fujimoto_off-policy_2019}. 
\reviewAll{In the last years, offline techniques like Conservative Q-Learning (CQL) \cite{kumar_conservative_2020}, Bootstrapping Error Accumulation Reduction (BEAR) \cite{kumar_stabilizing_2019} and  Behavior Regularized Actor Critic (BRAC) \cite{wu_behavior_2019}, take care of a particular treatment of the out-of-distribution (OOD) policy improvement, however, still using the semi-gradient estimation}.
Another class of off-policy gradient estimation uses importance sampling \cite{shelton_policy_2001,meuleau_exploration_2001,peshkin_learning_2002}~to deliver an unbiased estimate of the gradient but suffer from high variance and is generally only applicable with stochastic policies. Moreover, these algorithms require the full knowledge of the behavioral policy, making them unsuitable when data stems from a human demonstrator. 
\begin{figure}[t]
	\centering
	\includegraphics[width=0.95 \columnwidth]{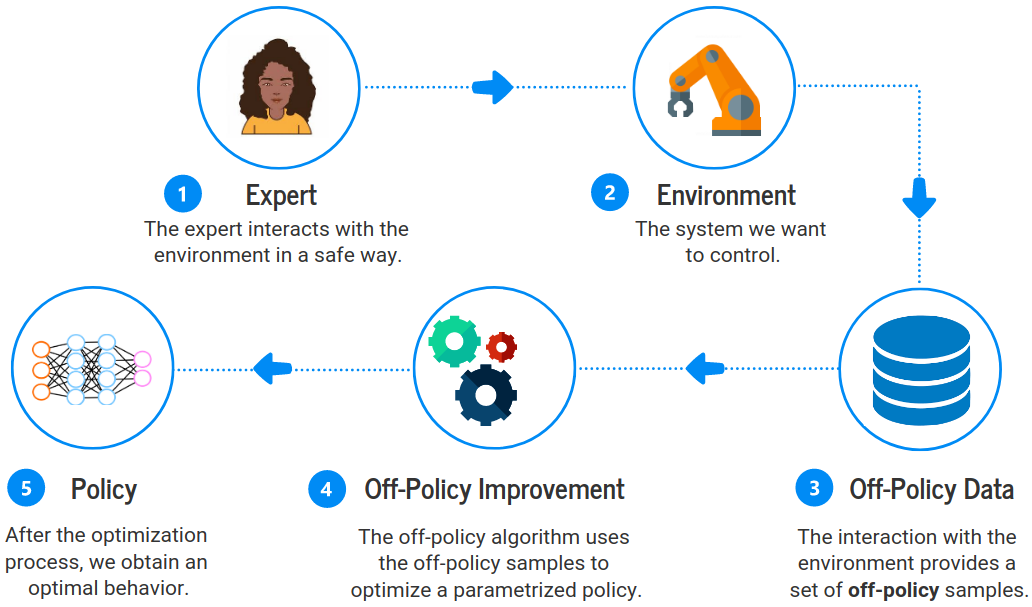}
	\caption{In the off-policy reinforcement learning scheme, the policy can be optimized using an off-policy dataset. This allows for safer interaction with the system and for better sample efficiency.}
	\label{figure:gradient}
\end{figure}
On the other hand, model-based approaches aim to approximate the environment's dynamics, allowing the generation of synthetic samples. The generative model can, in principle, be used by an on-policy reinforcement learning algorithm, therefore, encompassing the difficulty of off-policy estimation.
However, the generation of synthetic samples is also problematic, as they are affected by the error of the approximated dynamics. 
\reviewAll{To mitigate this problem, state-of-the-art techniques aim to quantify and penalize the most uncertain regions of the state-action space, like Model-based Offline Policy Optimization (MOPO) \cite{yu_mopo_2020} and Model-based Offline Reinforcement Learning (MOReL) \cite{kidambi_morel_2020}.} 
\reviewC{To address all previously highlighted issues, we propose a new algorithm, the Nonparametric Off-policy Policy Gradient (NOPG) \cite{tosatto_nonparametric_2020}. NOPG constructs a nonparametric model of the dynamics and the reward signal, and builds a full-gradient estimate based on the closed-form solution of a nonparametric Bellman equation. 
Our approach, in contrast to the majority of model-based approaches, does not require the generation of artificial data.
On the other hand, while model-free approaches are either built on semi-gradient estimation, or importance sampling, NOPG computes a full-gradient estimate without importance sampling estimators, and allows for the use of human demonstrations.} 
%  Unlike other nonparametric methods like PILCO, our approach allows for multimodal state-transitions, and can handle the infinite-horizon setting.
Figure~\ref{figure:gradient} shows the offline scheme of NOPG. A behavioral policy, represented by a human demonstrator, provides (possibly suboptimal) trajectories that solve a task. NOPG optimizes a policy from off-line and off-policy samples. 

\reviewC{\textbf{Contribution.}
This paper introduces a nonparametric Bellman equation, and the full-gradient  derived from its closed-form solution. We study the properties of the nonparametric Bellman equation, focusing on the bias of its closed-form solution.
We empirically analyze the bias and the variance of the proposed gradient estimator. We compare its effectiveness and efficiency w.r.t. state of the art online and offline techniques, observing that NOPG exhibits high sample efficiency.}

%\section{State-of-the-Art Off-Policy Policy Gradient}

\section{Problem Statement}
	Consider the reinforcement learning problem of an agent interacting with a given environment, as abstracted by a Markov decision process (MDP) and defined over the tuple $(\Sset, \Aset, \gamma, P, R, \mu_0)$ where $\Sset \equiv \mathbb{R}^{d_s}$ is the state space, $\Aset \equiv \mathbb{R}^{d_a}$ the action space, the transition-based discount factor $\gamma$ is a stochastic mapping between $\Sset\times\Aset\times\Sset$ to $[0, 1)$, which allows for unification of episodic and continuing tasks \cite{white_unifying_2017}. \reviewD{The discount factor allows the transormation of a particular MDP to an equivalent MDP where, at each state-action, there is $1-\gamma$ probability to transition to an absorbing state with zero reward. A variable discount factor, can, therefore, be interpreted as a variable probability of transitioning to an absorbing state. } To keep however the theory simple, we will assume that $\gamma(\state, \action, \state') \leq \gamma_c$ where $\gamma_c < 1$. The transition probability from a state $\state$ to $\nextstate$ given an action $\action$ is governed by the conditional density $\Trans{\nextstate}{\state}{\action}$. The stochastic reward signal $R$ for a transition $(\state, \action, \nextstate) \in \Sset \times \Aset \times \Sset$ is drawn from a distribution $R(\state, \action, \nextstate)$ with mean value $\EV_{\nextstate}[R(\state, \action, \nextstate)] = r(\state, \action)$. The initial distribution $\mu_0(\state)$ denotes the probability of the state $\state \in \Sset$ to be a starting state. 
	A policy $\pi$ is a stochastic or deterministic mapping from $\Sset$ onto $\Aset$, usually  parametrized by a set of parameters $\theta$.
	
	We define an \textsl{episode} as $\tau \equiv \{\state_t, \action_t, r_t, \gamma_t\}_{t=1}^\infty$ where
	\begin{eqnarray}
	& &\mathbf{s}_0 \sim \mu_{0}(\cdot);\quad  \mathbf{a}_t \sim \pi(\cdot \mid \mathbf{s}_t);\quad  \mathbf{s}_{t+1} \sim p(\cdot \mid \mathbf{s}_t, \mathbf{a}_t) \nonumber \\
	& & r_t \sim R(\mathbf{s}_t, \mathbf{a}_t, \state_{t+1});\quad  \gamma_t \sim \gamma(\state_t, \action_t, \state_{t+1}). \nonumber 
	\end{eqnarray}
	In this paper we consider the discounted infinite-horizon setting, where the objective is to maximize the expected return
	\begin{equation}
	J_{\pi}= \EV_\tau\left[\sum_{t=0}^{\infty} r_t \prod_{i=0}^{t}\gamma_i\right] \label{eq:return}.
	\end{equation}
	It is important to introduce two important quantities: the \textsl{stationary state visitation} $\mu_{\pi}$ and the \textsl{value function} $V_{\pi}$.
	We naturally extend the stationary state visitation defined by \cite{sutton_policy_2000} with the transition-based discount factor  
	\begin{eqnarray}
		\mu(\mathbf{s}) = \EV_\tau\left[ \sum_{t=0}^{\infty}  p(\mathbf{s}=\state_t | \pi, \mu_0) \prod_{i=1}^t\gamma_i\right], \nonumber
	\end{eqnarray}
	or, equivalently, as the fixed point of 
	\begin{eqnarray}
		\mu(\mathbf{s}) = \mu_0(\mathbf{s}) +  \int_{\Sset}\int_{\Aset} p_\gamma(\svec | \svec', \avec')\pi(\avec' | \svec')\mu_\pi(\nextstate)\de \svec'\de \avec' \nonumber
	\end{eqnarray}
	where, from now on, $p_\gamma(\nextstate | \state, \action)\! =\! p(\nextstate | \state, \action)\mathbb{E}[\gamma(\state, \action, \nextstate)]$.
	The value function
	\begin{eqnarray}
	V_{\pi}(\svec) = \EV_\tau\left[ \sum_{t=0}^\infty  r_t p(r_t | \svec_0 = \svec,  \pi) \prod_{i=0}^t\gamma_i\right], \nonumber
	\end{eqnarray}
	corresponds to the fixed point of the Bellman equation,
	\begin{equation}
	V_\pi(\svec)  =  \int_{\Aset} \pi(\avec| \svec)\bigg(r(\svec, \avec)+ \int_{\Sset}V_\pi(\svec')p_\gamma(\svec' | \svec, \avec)\de \svec' \bigg)\de \avec . \nonumber
	\end{equation}
	The state-action value function is defined as
	\begin{equation}
		Q_\pi(\svec, \action)  =  r(\svec, \avec) +  \int_{\Sset}V_\pi(\svec')p_\gamma(\svec' | \svec, \avec)\de \svec'. \nonumber
	\end{equation}

	The expected return \eqref{eq:return} can be formulated as
	\begin{eqnarray}
		J_{\pi} = \int_{\Sset} \mu_0(\svec)V_{\pi}(\svec)\de \svec = \int_{\Sset}\int_{\Aset}\mu_\pi(\svec)\pi(\avec | \svec)r(\svec, \avec)\de \avec\de \svec . \nonumber
	\end{eqnarray}

\textbf{Policy Gradient Theorem.} Objective \eqref{eq:return} is usually maximized via gradient ascent. The gradient of $J_{\pi}$ w.r.t. the policy parameters $\theta$ is

\begin{equation*}
\nabla_{\theta}J_{\pi} \!=\! \int\displaylimits_{\Sset}\int\displaylimits_{\Aset}\! \mu_{\pi}(\svec)\pi_{\theta}(\avec|\svec)Q_{\pi}(\svec,\avec)\nabla_{\theta}\log \pi_{\theta}(\avec | \svec) \de \avec \de \svec,
\end{equation*}
as stated in the policy gradient theorem \cite{sutton_policy_2000}. 
When it is possible to interact with the environment with the policy $\pi_\theta$, one can approximate the integral by considering the state-action as a distribution (up to a normalization factor) and use the samples to perform a Monte-Carlo (MC) estimation \cite{williams_simple_1992}. The $Q$-function can be estimated via Monte-Carlo sampling, approximate dynamic programming or by direct Bellman minimization. 
In the off-policy setting, we do not have access to the state-visitation $\mu_\pi$ induced by the policy, but instead we observe a different state distribution.
While estimating the $Q$-function with the new state distribution is well established in the literature \cite{watkins_q-learning_1992, ernst_tree-based_2005}, the shift in the state visitation $\mu_\pi(\state)$ is more difficult to obtain. State-of-the-art techniques either omit to consider this shift (we refer to these algorithms as \textsl{semi-gradient}), or they try to estimate it via importance sampling correction. These approaches will be discussed in detail in Section~\ref{sec:related_work}.

\section{Nonparametric Off-Policy Policy Gradient}
In this section, we introduce a nonparametric Bellman equation with a closed form solution, which carries the dependency from the policy's parameters. We derive the gradient of the solution, and discuss the properties of the proposed estimator.
\subsection{A Nonparametric Bellman Equation}
\reviewC{Nonparametric methods, require the storage of all samples to make predictions, but usually exhibit a low bias. They gained popularity thanks to their general simplicity and their theoretical properties.
In this paper, we focus on kernel density estimation, and Nadaraya-Watson kernel regression. In this context, it is usual to assume the kernel to be a symmetric, real function $\kappa$, with positive co-domain and with $\int k(x, y) \de x = 1$~$\forall y$.
The goal of kernel density estimation is to predict the density of a distribution $p(\state)$. After the collection of $n$ samples $\xvec_i \sim p(\cdot)$, it is possible to obtain the estimated density 
\begin{equation}
	\hat{p}(\xvec) = n^{-1} \sum_{i=1}^n\kappa(\xvec, \xvec_i). \nonumber
\end{equation}
Nadaraya-Watson kernel regression builds on the kernel density estimation to predict the output of a function $f:\mathbb{R}^d\to\mathbb{R}$, assuming to observe a dataset of $n$ samples $\{\xvec_i, y_i\}_{i=1}^n$ with $\xvec_i \sim p(\cdot)$ and $y_i = f(\xvec_i) + \epsilon_i$ (where $\epsilon_i$ is a zero-mean noise with finite variance). The prediction 
\begin{equation}
	\hat{f}(\xvec) = \frac{\int_{\mathbb{R}} y \hat{p}(\xvec, y)\de y}{\hat{p}(\xvec)} = \frac{\sum_{i=1}^n y_i \kappa(\xvec, \xvec_i)}{\sum_{i=1}^n\kappa(\xvec, \xvec_i)} \nonumber
\end{equation}
is constructed on the expectation of the conditional density probability, keeping in account that $\int_{\mathbb{R}} y \kappa(y, y_i) \de y = y_i$.
}
Nonparametric Bellman equations have been developed in a number of prior works. \cite{ormoneit_kernel-based_2002,xu_kernel-based_2007,engel_reinforcement_2005} used nonparametric models such as Gaussian Processes for approximate dynamic programming. \cite{taylor_kernelized_2009} have shown that these methods differ mainly in their use of regularization. \cite{kroemer_non-parametric_2011} provided a Bellman equation using kernel density-estimation and a general overview on nonparametric dynamic programming. In contrast to prior work, our formulation preserves the dependency on the policy, enabling the computation of the policy gradient in closed-form. Moreover, we upper-bound the bias of the Nadaraya-Watson kernel regression to prove that our value function estimate is consistent w.r.t. the classical Bellman equation under smoothness assumptions. \label{nppg}
We focus on the maximization of the average return in the infinite horizon case formulated as a starting state objective  $\int_\state \mu_0(\state)V_\pi(\state) \de \state $ \cite{sutton_policy_2000}.
\begin{definition}
	The discounted infinite-horizon objective is defined by $J_{\pi} = \int \mu_{0}(\state)V_{\pi}(\state)\de \state$. Under a stochastic policy the objective is subject to the Bellman equation constraint
	\label{definition:objective}
	\begin{align}
	V_{\pi}(\state)&\!=\!\int_{\Aset} \pi_{\theta}(\action|\state) \bigg(r(\state,\action)  +\gamma\!\int_{\Sset} V_{\pi}(\nextstate)p(\nextstate|\state, \action)\de \nextstate \bigg)\!\de \action,
	\end{align}
	while in the case of a deterministic policy the constraint is given as
	\begin{equation*}
	V_{\pi}(\state) = r(\state,\pi_{\theta}(\state)) + \gamma \int_{\Sset} V_{\pi}(\nextstate)p(\nextstate|\state, \pi_{\theta}(\state))\de \nextstate.
	\end{equation*}
\end{definition}
Maximizing the objective in Definition~\ref{definition:objective} analytically is not possible, excluding special cases such as under  linear-quadratic assumptions \cite{borrelli_predictive_2017}, or finite state-action space. Extracting an expression for the gradient of $J_{\pi}$ w.r.t. the policy parameters $\theta$ is also not straightforward given the infinite set of possibly non-convex constraints represented in the recursion over $V_{\pi}$. Nevertheless, it is possible to transform the constraints in Definition~\ref{definition:objective} to a finite set of linear constraints via nonparametric modeling, thus leading to an expression of the value function with simple algebraic manipulation \cite{kroemer_non-parametric_2011}.
\subsubsection{Nonparametric Modeling.} 
Assume a set of $n$ observations $D\!\equiv\!\{\state_i, \action_i, r_i, {\nextstate}_i, \gamma_i\}_{i=1}^{n}$ sampled from interaction with an environment, with $\state_i, \action_i \sim \beta(\cdot, \cdot)$, $\state_i' \sim p(\cdot |\state_i, \action_i)$, $r_i \sim R(\state_i, \action_i)$ and $\gamma \sim \gamma(\state_i, \action_i, \nextstate_i)$. We define the kernels $\psi:\Sset\times\Sset \to \mathbb{R}^+$, $\varphi:\Aset\times\Aset\to \mathbb{R}^+$ and $\phi:\Sset\times\Sset\to\mathbb{R}^+$, as normalized, symmetric and positive definite functions with bandwidths $\hvec_{\psi}, \hvec_{\varphi}, \hvec_{\phi}$ respectively. We define $\psi_i(\state) = \psi(\state, \state_i)$, $\varphi_i(\action) = \varphi(\action, \action_i)$, and $\phi_i(\state) = \phi(\state, {\nextstate}_i)$. Following \cite{kroemer_non-parametric_2011}, the mean reward $r(\state, \action)$ and the transition conditional $p(\nextstate|\state, \action)$ are approximated by the Nadaraya-Watson regression \cite{nadaraya_estimating_1964,watson_smooth_1964} and kernel density estimation, respectively
\begin{align}
\approxx{r}(\state, \action) &\! \defeq \! \frac{\sum_{i=1}^n\psi_i(\state)\varphi_i(\action)r_i}{ \sum_{i=1}^n\psi_i(\state)\varphi_i(\action)}\nonumber , \\ 
\approxx{p}(\nextstate | \state, \action) &\! \defeq \! \frac{\sum_{i=1}^n \phi_i(\nextstate)\psi_i(\state)\varphi_i(\action)}{\sum_{i=1}^n \psi_i(\state)\varphi_i(\action)} \nonumber , \\
\approxx{\gamma}(\state, \action, \nextstate) & \defeq
\frac{ \sum_{i=1}^n\gamma_i\psi_i(\state)\varphi_i(\action)\phi_i(\nextstate)}{\sum_{i=1}^n\psi_i(\state)\varphi_i(\action)\phi_i(\nextstate)}\nonumber 
\end{align}
and, therefore, by the product of $\approxx{p}$ and $\approxx{\gamma}$ we obtain
\begin{eqnarray}
	\approxx{p}_{\gamma}(\nextstate |\state, \action) &\defeq& \approxx{p}(\nextstate | \state, \action)\approxx{\gamma}(\state, \action, \nextstate) \nonumber \\
	&=& \frac{ \sum_{i=1}^n\gamma_i\psi_i(\state)\varphi_i(\action)\phi_i(\nextstate)}{\sum_{i=1}^n\psi_i(\state)\varphi_i(\action)}. \nonumber 
\end{eqnarray}

Inserting the reward and transition kernels into the Bellman Equation for the stochastic policy case we obtain the nonparametric Bellman equation (NPBE)
\begin{align}
\approxx{V}_{\pi}(\svec) &\!=\!\int_{\Aset} \! \pis(\avec|\svec)\bigg(\approxx{r}(\state, \action)\!+\!\int_{\Sset} \approxx{V}_{\pi}(\svec')\approxx{p}_\gamma(\svec'|\svec, \avec) \de \svec'\bigg)\!\de \avec \nonumber \\
& \!=\!\sum_i \int_{\Aset}  \frac{\pis(\avec|\svec) \psi_i(\state)\varphi_i(\action)}{\sum_j\psi_j(\state)\varphi_j(\action)}\de \avec \nonumber \\
& \quad \quad \quad \times \bigg(r_i + \gamma_i \int_{\Sset} \phi_i(\svec')\approxx{V}_{\pi}(\svec') \de \svec' \bigg). \label{equation:npbe} 
\end{align}
Equation~\eqref{equation:npbe} can be conveniently expressed in matrix form by introducing the vector of responsibilities $\varepsilon_i(\state)\!=\!\int \pis(\avec|\svec)\! \psi_i(\state)\varphi_i(\action)/\sum_j\psi_j(\state)\varphi_j(\action)\de \avec$, which assigns each state $\state$ a weight relative to its distance to a sample $i$ under the current policy.
\begin{definition}{}
	\label{definition:npbe}
	The nonparametric Bellman equation on the dataset $D$ is formally defined as
	\begin{equation}
	\approxx{V}_{\pi}(\svec)\!=\!\bm{\varepsilon}_{\pi}^{\intercal}(\svec) \left(\rvec +  \int_{\Sset} \phivec_\gamma(\svec')\approxx{V}_{\pi}(\svec') \de \svec' \right), \label{eq:formalnpbe}
	\end{equation}
	with $\phivec_\gamma(\svec)  \!=\![\gamma_1\phi_1(\svec),\dots, \gamma_n\phi_n(\svec)]^{\intercal}, \rvec  \!=\![r_1,\dots, r_n]^{\intercal}$, $\bm{\varepsilon}_{\pi}(\svec)   \!=\![\varepsilon_1^\pi(\svec)$, \dots,  $\varepsilon_n^\pi(\svec)]^{\intercal}$,
	\begin{align}
		\varepsilon_i(\svec, \avec) & = \frac{\psi_i(\state)\varphi_i(\action)}{\sum_j\psi_j(\state)\varphi_j(\action)} \nonumber
	\end{align}
	and
	\begin{align}
		\varepsilon^{\pi}_i(\svec)\!  &= \! \begin{cases}
			\int \pis(\avec|\svec)\varepsilon_i(\svec, \avec)  \de \avec & \text{if $\pi$ is stochastic} \\  \varepsilon_i(\svec, \pis(\svec))  & \text{otherwise.} \nonumber
		\end{cases}
	\end{align}
	% 	\begin{equation}
	% 		\begin{cases}
	% 			\varepsilon^{\pi}_i(\svec) \defeq \int \pis(\avec|\svec) \frac{\psi_i(\state)\varphi_i(\action)}{\sum_j\psi_j(\state)\varphi_j(\action)}\de \avec & \text{if $\pi$ is stochastic}, \\
	% 			\varepsilon^{\pi}_i(\svec)  \defeq \frac{\psi_i(\state)\varphi_i(\pis(\svec))}{\sum_j\psi_j(\state)\varphi_j(\pis(\svec))}                                & \text{otherwise.}
	% 		\end{cases}\nonumber
	% 	\end{equation}
\end{definition}
From Equation~\eqref{eq:formalnpbe} we deduce that the value function must be of the form $\bm{\varepsilon}_{\pi}^\intercal(\state)\qvec$, indicating that it can also be seen as a form of Nadaraya-Watson kernel regression,
\begin{equation}
\bm{\varepsilon}_{\pi}^\intercal(\state)\qvec = \bm{\varepsilon}_{\pi}^{\intercal}(\svec) \left(\rvec +  \int_{\Sset} \phivec_\gamma(\svec')\bm{\varepsilon}_{\pi}^{\intercal}(\svec')\qvec \de \svec' \right). \label{eq:npbewresp}
\end{equation}
Notice that, trivially,  every $\qvec$ which satisfies 
\begin{equation}
\qvec = \rvec + \int_{\Sset} \phivec_\gamma(\svec')\bm{\varepsilon}_{\pi}^{\intercal}(\svec')\qvec \de \svec' \label{eq:npbeworesp}
\end{equation}
also satisfies Equation~\eqref{eq:npbewresp}. Theorem~\ref{theorem:fixedpoint} demonstrates that the algebraic solution of Equation~\eqref{eq:npbeworesp} is the \textsl{only} solution of the nonparametric Bellman Equation~\eqref{eq:formalnpbe}. 
\begin{theorem}{}
	\label{theorem:fixedpoint}
	The nonparametric Bellman equation has a unique fixed-point solution 
	\begin{equation}
	\approxx{V}^*_{\pi}(\svec) \defeq \bm{\varepsilon}^{\intercal}_{\pi}(\svec)\bm{\Lambda}_\pi^{-1}\rvec, \nonumber
	\end{equation} with $\kerres \defeq I-\approxx{\mathbf{P}}_{\pi}^\gamma$ and $\approxx{\mathbf{P}}_\pi^\gamma \defeq \int_{\Sset} \phivec_\gamma(\svec')\bm{\varepsilon}_{\pi}^{\intercal}(\svec') \de \svec' $, where $\bm{\Lambda}_{\pi}$ is always invertible since $\approxx{\mathbf{P}}^{\pi, \gamma}$ is a strictly sub-stochastic matrix (Frobenius' Theorem). The statement is valid also for $n \to \infty$, provided bounded $R$.
\end{theorem}
\reviewD{
The transition matrix $\approxx{\mathbf{P}}_{\pi}^\gamma$ is strictly sub-stochastic since each row $\approxx{\mathbf{P}}_{\pi, i}^\gamma = \gamma_i \int \phi_i(\state')\bm{\varepsilon}_{\pi}^{\intercal}(\svec') \de \svec'$ is composed by the convolution between $\phi_i$, which by definition integrates to one, and $0 \leq \bm{\varepsilon}_{\pi}^{\intercal}(\svec') \leq 1$, as can be seen in Definition~\ref{definition:npbe}.}
Proof of Theorem~\ref{theorem:fixedpoint} is provided in the supplementary material.
\subsection{Nonparametric Gradient Estimation}
\label{sec:gradient-estimation}
With the closed-form solution of $\approxx{V}^*_{\pi}(\svec)$ from Theorem~\ref{theorem:fixedpoint}, it is possible to compute the analytical gradient of \reviewA{$\hat{J}_{\pi}$} w.r.t. the policy parameters
\begin{align}
\grad{\theta} \approxx{V}_{\pi}^*(\svec) &=  \bigg( \pargrad{\theta}\bm{\varepsilon}_{\pi}^{\intercal}(\svec)\bigg)\kerres^{-1} \mathbf{r}  + \bm{\varepsilon}_{\pi}^{\intercal}(\svec) \pargrad{\theta} \kerres^{-1} \mathbf{r}  \nonumber \\
%&= & \bigg( \pargrad{\theta}\bm{\varepsilon}_{\pi}^{\intercal}(\svec)\bigg)\kerres^{-1} \mathbf{r}  - \bm{\varepsilon}_{\pi}^{\intercal}(\svec)\kerres^{-1}\bigg(\pargrad{\theta}\kerres\bigg)\kerres^{-1}\mathbf{r} \nonumber \\
&=  \underbrace{\bigg( \pargrad{\theta}\bm{\varepsilon}_{\pi}^{\intercal}(\svec)\bigg)\kerres^{-1} \mathbf{r}}_{\text{A}} \nonumber \\ &  \quad \quad + \underbrace{  \bm{\varepsilon}_{\pi}^{\intercal}(\svec)\kerres^{-1}\bigg(\pargrad{\theta} \approxx{\mathbf{P}}^\gamma_{\pi}\bigg)\kerres^{-1} \mathbf{r}}_{\text{B}}. \label{equation:kergradv}
\end{align}
Substituting the result of Equation~\eqref{equation:kergradv} into the return specified in Definition~\ref{definition:objective}, introducing $\bm{\varepsilon}_{\pi,0}^{\intercal} \defeq \int\mu_0(\svec)\bm{\varepsilon}_{\pi}^{\intercal}(\svec) \de \svec$, $\mathbf{q}_{\pi} = \bm{\Lambda}^{-1}_\pi \rvec$, and $\bm{\mu}_{\pi} = \bm{\Lambda}_{\pi}^{-\intercal} \bm{\varepsilon}_{\pi,0}$, we obtain
\begin{equation}
\grad{\theta} \approxx{J}_{\pi} = \bigg( \pargrad{\theta}\bm{\varepsilon}_{\pi,0}^{\intercal}\bigg)\mathbf{q}_{\pi}  + \bm{\mu}_{\pi}^{\intercal}\bigg(\pargrad{\theta} \approxx{\mathbf{P}}_{\pi}^\gamma \bigg)\mathbf{q}_{\pi}, \label{equation:algorithmicgradient}
\end{equation}
where $\mathbf{q}_{\pi}$ and $\bm{\mu}_{\pi}$ can be estimated via conjugate gradient to avoid the inversion of $\bm{\Lambda}_\pi$.
It is interesting to notice that \eqref{equation:algorithmicgradient} is closely related to the Policy Gradient Theorem \cite{sutton_policy_2000},
\begin{equation}
\nabla_\theta\hat{J}_\pi =\! \int\displaylimits_{\Sset \times \Aset} \underbrace{\left(\mu_0(\state) +  \bm{\phi}_\gamma^\intercal(\state)\bm{\mu}_\pi\right)}_\text{Stationary distribution}\underbrace{\bm{\varepsilon}^\intercal(\state, \action)\bm{q}_\pi}_{\hat{Q}_\pi(\state, \action)} \nabla_\theta \pi_\theta(\action | \state) \de \action\!\de \state \nonumber
\end{equation} 
and to the Deterministic Policy Gradient Theorem (in its onpolicy formulation) \cite{silver_deterministic_2014},
\begin{equation}
\nabla_\theta \hat{J}_\pi= \int_\Sset \underbrace{\left(\mu_0(\state) + \bm{\phi}_\gamma^\intercal(\state) \bm{\mu}_\pi\right)}_\text{Stationary distribution} \underbrace{\nabla_\action\bm{\varepsilon}^\intercal(\state, \action)\bm{q}_\pi}_{\nabla_\action \hat{Q}_\pi(\state, \action)}  \nabla_\theta \pi_\theta(\state) \de \state . \nonumber 
\end{equation}
In a later analysis, we will consider the state-action surrogate return $J_S^\pi(\state, \action) = (\mu_0(\state) + \Phi_\gamma^\intercal(\state) \bm{\mu}_\pi)\bm{\varepsilon}^\intercal(\state, \action)\bm{q}_\pi$.
The terms $\text{A}$ and $\text{B}$ in Equation~\eqref{equation:kergradv} correspond to the terms in Equation~\eqref{equation:qgrad}. In contrast to semi-gradient actor-critic methods, where the gradient bias is affected by both the critic bias and the semi-gradient approximation \cite{imani_off-policy_2018,fujimoto_off-policy_2019}, our estimate is the \textsl{full gradient} and the only source of bias is introduced by the estimation of $\approxx{V}_{\pi}$, which we analyze in Section~\ref{sec:nonparaestimation}.
The term $\bm{\mu}_\pi$ can be interpreted as the support of the state-distribution as it satisfies  $\bm{\mu}^{\intercal}_\pi =\bm{\varepsilon}_{\pi,0}^{\intercal} + \bm{\mu}_\pi^{\intercal} {\approxx{\mathbf{P}}^\gamma_{\pi}}$. In Section~\ref{section:experiments}, more specifically in Figure~\ref{figure:muv}, we empirically show that $\varepsilon_{\pi}^{\intercal}(\svec)\bm{\mu}_\pi$ provides an estimate of the state distribution over the whole state-space.
%However, we refer to \eqref{equation:algorithmicgradient} for the practical implementation.
The quantities $\bm{\varepsilon}_{\pi,0}^{\intercal}$ and $\approxx{\mathbf{P}}^{\pi}_{i,j}$ are estimated via  Monte-Carlo sampling, which is unbiased but computationally demanding, or using other techniques such as unscented transform or numerical quadrature. The matrix $\approxx{\mathbf{P}}_{\pi}^\gamma$ is of dimension $n \times n$, which can be memory-demanding. In practice, we notice that the matrix is often sparse. By taking advantage of conjugate gradient and sparsification we are able to achieve computational complexity of $\mathcal{O}(n^2)$ per policy update and memory complexity of $\mathcal{O}(n)$. Further details on the computational and memory complexity can be found in the supplementary material. A schematic of our implementation is summarized in Algorithm~\ref{alg:kbpgalg}.
\begin{algorithm}[t]
	\caption{Nonparametric Off-Policy Policy Gradient}
	\label{alg:kbpgalg}
	\begin{algorithmic}
		\STATE \textbf{input:} dataset $\{\state_i, \action_i, r_i,\nextstate_i, \gamma_i \}_{i=1}^{n}$ where $\pi_{\theta}$ indicates the policy to optimize and $\psi, \phi, \varphi$ the kernels respectively for state, action and next-state.
		\WHILE{ \text{not converged}}
		\STATE Compute $\bm{\varepsilon}^{\intercal}_{\pi}(\svec)$ as in Definition~\ref{definition:npbe} and  $\bm{\varepsilon}_{\pi,0}^{\intercal} \defeq \int\mu_0(\svec)\bm{\varepsilon}_{\pi}^{\intercal}(\svec) \de \svec$.
		\STATE Estimate $\approxx{\mathbf{P}}_{\pi}^\gamma$ as defined in Theorem~\ref{theorem:fixedpoint} using MC ($\phi(\svec)$ is a distribution).
		\STATE Solve $\rvec = \bm{\Lambda}_\pi\mathbf{q}_{\pi}$ and $\bm{\varepsilon}_{\pi,0} = \bm{\Lambda}_{\pi}^{\intercal}\bm{\mu}_{\pi}$ for $\mathbf{q}_{\pi}$ and $\bm{\mu}_{\pi}$ using conjugate gradient.
		\STATE Update $\theta$ using Equation~\eqref{equation:algorithmicgradient}.
		\ENDWHILE
	\end{algorithmic}
\end{algorithm}

\subsection{A theoretical Analysis}
\label{sec:nonparaestimation}
Nonparametric estimates of the transition dynamics and reward enjoy favorable properties for an off-policy learning setting. A well-known asymptotic behavior of the Nadaraya-Watson kernel regression,
\begin{align}
& \EV \left[\lim_{n\to \infty}\approxx{f}_n(x) \right] - f(x)  \approx \nonumber\\
& \quad\quad\quad  h^2_n \bigg(\frac{1}{2}f''(x) + \frac{f'(x){ \beta'(x)}}{ \beta(x)}\bigg) \int u^2 K(u)\de u, \nonumber %+o_p(h_n^2) \nonumber
\end{align}
shows how the bias is related to the regression function $f(x)$, as well as to the samples' distribution $\beta(x)$  \cite{fan_design-adaptive_1992,wasserman_all_2006}. 
However, this asymptotic behavior is valid only for infinitesimal bandwidth, infinite samples ($h \to 0, nh \to \infty$) and requires the knowledge of the regression function and of the sampling distribution.

In a recent work, we propose an upper bound of the bias that is also valid for finite bandwidths \cite{tosatto_upper_2020}. 
We show under some Lipschitz conditions that the bound of the Nadaraya-Watson kernel regression bias does not depend on the samples' distribution, which is a desirable property in off-policy scenarios. 
The analysis is extended to multidimensional input space. For clarity of exposition, we report the main result in its simplest formulation, and later use it to infer the bound of the NPBE bias.

\begin{theorem}{}
	\label{theorem:biasnadaraya}
	Let $f\!:\!\mathbb{R}^d\!\to\!\mathbb{R}$ be a Lipschitz continuous function with constant $L_{f}$.
	Assume a set $\{\xvec_i, y_i\}_{i=1}^n$ of i.i.d. samples from a log-Lipschitz distribution $\beta$ with a Lipschitz constant $L_{\beta}$. Assume $y_i = f(\xvec_i) + \epsilon_i$, where $f\!:\!\mathbb{R}^d\!\to\!\mathbb{R}$ and $\epsilon_i$ is i.i.d. and zero-mean. The bias of the Nadaraya-Watson kernel regression with Gaussian kernels in the limit of infinite samples $n\to \infty$ is bounded by
	\begin{equation*}
	\begin{split}
	& \left|\EV \left[\lim_{n\to \infty}\approxx{f}_n(\xvec) \right] - f(\xvec) \right| \leq \\
	& \quad \quad \quad \quad \frac{L_f \sum\limits_{k=1}^d \hvec_k \left(\prod\limits_{i\neq k}^d \chi_i\right)  \left( \frac{1}{\sqrt{2 \pi} } + \frac{L_{\beta}\hvec_k}{2} \chi_k \right) }{ \prod\limits_{i=1}^d \e{\frac{L_{\beta}^2 h_i^2}{2}}\left(1 - \erf\left(\frac{\hvec_i L_{\beta}}{\sqrt{2}} \right) \right)},\\
	\end{split}
	\end{equation*}
	where \begin{equation}
	\chi_i = e^{\frac{L_{\beta}^2\hvec^2_i}{2}}\bigg(1+\erf\bigg(\frac{\hvec_iL_\beta}{\sqrt{2}} \bigg)\bigg) \nonumber ,
	\end{equation}
	$\hvec > 0 \in \mathop{R}^d$ is the vector of bandwidths and $\erf{}$ is the error function.
\end{theorem}

	Building on Theorem~\ref{theorem:biasnadaraya} we show that the solution of the NPBE is consistent with the solution of the true Bellman equation. Moreover, although the bound is not affected directly by $\beta(\state)$, a smoother sample distribution $\beta(\state)$ plays favorably in the bias term (a low $L_{\beta}$ is preferred).
\begin{theorem}{}
	\label{theorem:ultimate}
	Consider an arbitrary MDP $\mathcal{M}$ with a transition density $p$ and a stochastic reward function $R(\svec, \avec) = r(\svec,\avec) + \epsilon_{\svec, \avec}$, where $r(\svec,\avec)$ is a Lipschitz continuous function with $L_R$ constant and $\epsilon_{\svec, \avec}$ denotes zero-mean noise. Assume $|R(\svec, \avec)|\!\leq\!R_{\text{max}}$ and a dataset $D_n$ sampled from a log-Lipschitz distribution $\beta$ defined over the state-action space with Lipschitz constant $L_{\beta}$. Let $V_D$ be the unique solution of a nonparametric Bellman equation with Gaussian kernels $\psi,\varphi,\phi$ with positive bandwidths $\hvec_{\psi},\hvec_{\varphi}, \hvec_{\phi}$ defined over the dataset $\lim_{n\to\infty} D_n$. 
	%$V_D$ is an estimator of the fixed point $V^*$ of the classical Bellman equation defined over $\mathcal{M}$. 
	Assume $V_D$ to be Lipschitz continuous with constant $L_V$. The bias of such estimator is bounded by
	\begin{equation}
	\big|\overline{V}(\svec) - V^*(\svec)\big| \leq  \frac{1}{1-\gamma_c} \bigg(\text{A}_\text{Bias} + \gamma_c L_{V} \sum_{k=1}^{d_s}\frac{h_{\phi,k}}{\sqrt{2 \pi}} \bigg),
	\end{equation}
	where $\overline{V}(\svec) = \EV_{D}[V_D(\svec)]$, $\text{A}_\text{Bias}$ is the bound of the bias provided in Theorem~\ref{theorem:biasnadaraya} with $L_f\!=\!L_R$, $\hvec\!=\![\hvec_{\psi},\hvec_{\varphi}]$, $d\!=\!d_s\!+\!d_a$ and $V^*(\state)$ is the fixed point of the ordinary Bellman equation. \footnote{Complete proofs of the theorems and precise definitions can be found in the supplementary material.}
\end{theorem}
%\vspace{2em}
Theorem~\ref{theorem:ultimate} shows that the value function provided by Theorem~\ref{theorem:fixedpoint} is consistent when the bandwidth approaches infinitesimal values. Moreover, it is interesting to notice that the error can be decomposed in $A_\text{Bias}$, which is the bias component dependent on the reward's approximation, and the remaining term that depends on the smoothness of the value function and the bandwidth of $\phi$, which can be read as the error of the transition's model.

The bound shows that smoother reward functions, state-transitions and sample distributions play favorably against the estimation bias. 
\reviewD{Notice that the bias persist even in the support points. This issue is known in Nadaraya-Watson kernel regression. However, the bias can be controlled and lowered by reducing the bandwidth.
The i.i.d. assumption required by Theorem~\ref{theorem:ultimate} is not restrictive, as even if the samples are collected by a sequential process that interacts with the MPD, the sample distribution will eventually converge to the stationary distribution of the MDP. Furthermore, our algorithm considers all the samples simoultaneously, therefore, some inter-correlations between samples does not affect the estimation.}
\begin{figure}[H]
	\centering
	\includegraphics[width=0.95 \columnwidth]{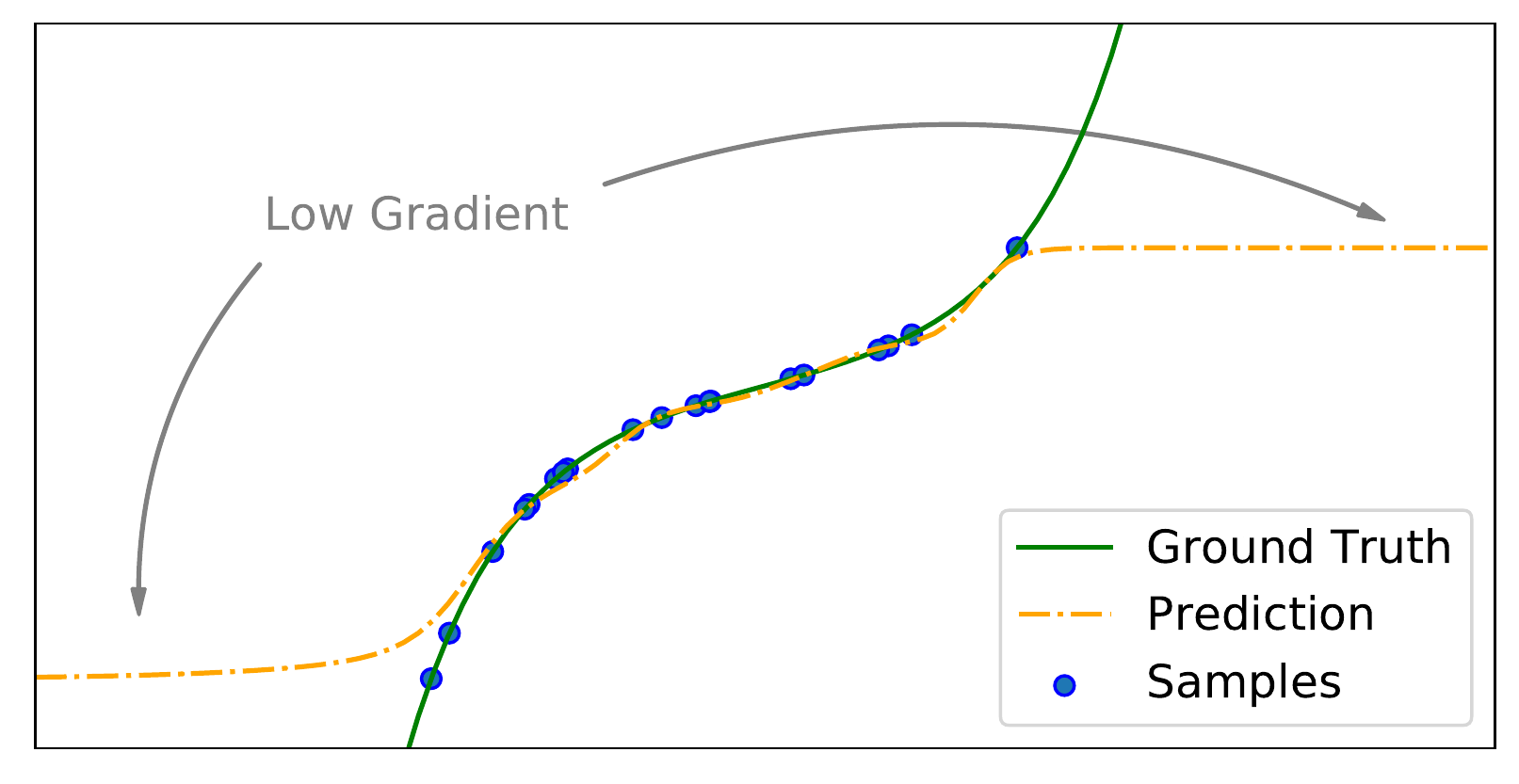}
	\caption{The classic effect (known as boundary-bias) of the Nadaraya-Watson regression predicting a constant function in low-density regions is beneficial in our case, as it prevents the policy from moving in those areas as the gradient gets close to zero. \label{fig:block}}
\end{figure}

\reviewB{
Beyond the bias analysis in the limit of infinite data is important to understand the variance and the bias for finite data.
We provide a finite sample analysis in the empirical section. It is known that the lowest mean squared error is obtained by setting a higher bandwidth in presence of scarce data, and decreasing it with more data. 
In our implementation of NOPG, we select the bandwidth using cross validation.}

\subsubsection{A Trust Region}
\label{set:trust-region}
Very commonly,  in order to prevent harmful policy optimization, the policy is constrained to stay close to the data \cite{peters_relative_2010}, to avoid taking large steps \cite{schulman_trust_2015, schulman_proximal_2017} or to circumvent large variance in the estimation \cite{metelli_policy_2018,chua_deep_2018}. These techniques, which keep the policy updates in a trusted-region, prevent incorrect and dangerous estimates of the gradient.
Even if we do not include any explicit constraint of this kind, the Nadaraya-Watson kernel regression automatically discourages policy improvements towards low-data areas. 
In fact, as depicted in Figure~\ref{fig:block}, the Nadaraya-Watson kernel regression, tends to predict a constant function in low density regions. Usually, this characteristic is regarded as an issue, as it causes the so-called boundary-bias. In our case, this effect turns out to be beneficial, as it constrains the policy to stay close to the samples, where the model is more correct.

\section{Related Work}
\label{sec:related_work}
\reviewAll{Off-policy and offline policy optimization became increasingly more popular in the recent years and have been explored in many different flavors. Batch reinforcement learning, in the model-free view, has been elaborated both as a value-based and a policy gradient technique (we include actor-critic in this last class). Recently, offline reinforcement learning has been proposed also in the mode-based formulation. In this section, we aim to detail both the advantages and disadvantages of the proposed approaches. As we will see, our solution shares some advantages typical of the model-free policy gradient algorithms while using an approximated model of the dynamics and reward.}
\subsection{Model Free}
\reviewAll{The model-free offline formulation aims to improve the policy based purely on the samples collected in the dataset, without generating synthetic samples from an approximated model. We can divide this broad category in two main families: the value-based approaches, and the policy gradients.}
\subsubsection{Value Based}
\reviewAll{Value-based techniques are constructed on the approximate dynamic theory, leveraging on the policy-improvement theorem and on the contraction property of the Bellman operator. Examples of offline approximate dynamic programming algorithms are Fitted Q-Iteration (FQI) and Neural FQI (NFQI) \cite{ernst_tree-based_2005, riedmiller_neural_2005}. The $\max$ operator used in the optimal Bellman operator, usually restricts the applicability to discrete action spaces (few exceptions, \cite{antos_fitted_2007}). Furthermore, the projected error can prevent the convergence to a satisfying solution \cite{baird_residual_1995}. These algorithms usually suffer also of the delusion bias \cite{lu_non-delusional_2018}.
Despite these issues, there has been a recent revival of value-based techniques in the context of offline reinforcement learning. Batch-Constrained Q-learning (BCQ) \cite{fujimoto_off-policy_2019}, at the best of our knowledge, firstly defined the extrapolation error, which is partially caused by optimizing unseen (out of distribution, OOD) state-action pairs in the dataset. }
This source of bias has been extensively studied in subsequent works, both in the model-free and model-based setting.
\reviewA{ Value-based methods using non-parametric kernel-based approaches have also been used \cite{ormoneit_kernel-based_2002,xu_kernel-based_2007,kroemer_kernel-based_2012}.  An interesting discussion in \cite{taylor_kernelized_2009} shows that kernelized least square temporal difference approaches are also approximating the reward and transition models, and therefore they can be considered model-based.}
\subsubsection{Policy Gradient}
\reviewAll{This class of algorithms leverages their theory on the policy gradient theorem \cite{sutton_policy_2000,silver_deterministic_2014}. Examples of policy gradients are REINFORCE \cite{williams_simple_1992}, G(PO)MDP \cite{baxter_infinite-horizon_2001}, natural policy gradient  \cite{kakade_natural_2001}, and Trust-Region Policy Optimization (TRPO) \cite{schulman_trust_2015}. Often, these approaches make use of approximate dynamic programming to estimate the $Q$-function (or related quantities), to obtain lower variance \cite{peters_natural_2008,lillicrap_continuous_2016,schulman_proximal_2017,haarnoja_soft_2018}. The policy gradient theorem, however, defines the policy gradient w.r.t. \textsl{on-policy samples}. To overcome this problem two main techniques have been introduced: semi-gradient approaches, which rely on the omission of one term in the computation of the gradient, introduce an irreducible source of bias, while importance sampling solutions, which are unbiased, suffer from high variance.}
%Semi-gradient approaches omit one term in the gradient computation, which causes an estimation bias \cite{imani_off-policy_2018}.
%The importance sampling correction, although unbiased, suffers from high variance, which makes it often unpractical \cite{owen_monte_2013}.
%Model based approaches rely on a model's estimation, and optimize the policy following this model. However, the model's error propagates in the number of steps, adding also a significant source of bias \cite{deisenroth_pilco:_2011}.
\begin{figure*}[pht]
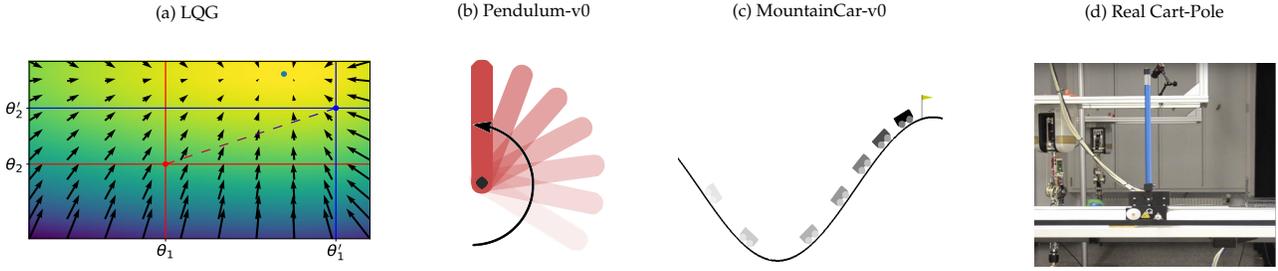

	\centering
	\begin{subfigure}[t]{.6\columnwidth}
		\input{plots/poles/lqg.tikz}
	\end{subfigure}
	\hspace{3pt}
	\begin{subfigure}[t]{0.35\columnwidth}
		\input{plots/poles/pendulum.tikz}
	\end{subfigure}
	\hspace{3pt}
	\begin{subfigure}[t]{.5\columnwidth}
		\input{plots/poles/mountain.tikz}
	\end{subfigure}
	\hspace{3pt}
	\begin{subfigure}[t]{.5\columnwidth}
		\input{plots/poles/real-cart.tikz}
	\end{subfigure}
%	\begin{subfigure}[t]{.5\columnwidth}
%		\input{plots/poles/valve.tikz}
%	\end{subfigure}
	\vspace{-0.5cm}
	\caption{Some of the benchmarking tasks. The return's landscape (a) of the LQG problem. In the gradient analysis, we obtain the gradient of the policy with parameters $\theta_1, \theta_2$ by sampling from a policy interpolated with the parameters $\theta_1', \theta_2'$. Sub-figures (b) and (d) depicts the OpenAI environment used. The real system in (c) has been used to evaluate the policy learned to stabilize the cart-pole task.\label{fig:tasks}}
\end{figure*}
\\
\textbf{Semi-Gradient:}
The off-policy policy gradient theorem was the first proposed off-policy actor-critic algorithm \cite{degris_off-policy_2012-1}. Since then, it has been used by the vast majority of state-of-the-art off-policy algorithms \cite{silver_deterministic_2014, lillicrap_continuous_2016, haarnoja_soft_2018}. Nonetheless, it is important to note that this theorem and its successors, introduce two approximations to the original policy gradient theorem \cite{sutton_policy_2000}. First, semi-gradient approaches consider a modified discounted infinite-horizon return objective $\approxx{J}_{\pi} = \int\rho_{\beta}(\svec)V_{\pi}(\svec) \de \state$, where $\rho_{\beta}(\svec)$ is state distribution under the behavioral policy $\pi_{\beta}$. Second, the gradient estimate is modified to be
\begin{align}
\gradt \approxx{J}_{\pi} & = \gradt\int_{\Sset} \rho_{\beta}(\state) V_{\pi}(\state)\de \state \nonumber \\
& = \gradt \int_{\Sset} \rho_{\beta}(\state) \int_{\Aset} \pi_{\theta}(\action|\state) Q_{\pi}(\state, \action)\de \action \de \state  \nonumber                                                 \\
& = \int\displaylimits_{\Sset} \rho_{\beta}(\state) \int\displaylimits_{\Aset} \underbrace{\gradt \pi_{\theta}(\action|\state) Q_{\pi}(\state, \action)}_{\text{A}} \nonumber \\
& \quad \quad + \underbrace{\pi_{\theta}(\action|\state) \gradt Q_{\pi}(\state, \action)}_{\text{B}} \de \action \de \state \label{equation:qgrad} \\
& \approx \int\displaylimits_{\Sset} \rho_{\beta}(\state) \int\displaylimits_{\Aset} \gradt \pi_{\theta}(\action|\state) Q_{\pi}(\state, \action) \de \action\de \state, \nonumber
\end{align}
where the term $\text{B}$ related to the derivative of $Q_{\pi}$ is ignored. The authors provide a proof that this biased gradient, or \textsl{semi-gradient}, still converges to the optimal policy in a tabular setting \cite{degris_off-policy_2012-1, imani_off-policy_2018}. However, further approximation (e.g., given by the critic and by finite sample size), might disallow the convergence to a satisfactory solution.
Although these algorithms work correctly sampling from the replay memory (which discards \textsl{the oldest} samples), they have shown to fail with samples generated via a completely different process \cite{imani_off-policy_2018,fujimoto_off-policy_2019}.
The source of bias introduced in the semi-gradient estimation depends fully on the distribution mismatch, and cannot be recovered by an increment of the dataset size or with a more powerful function approximator. Still, mainly due to its simplicity, semi-gradient estimation has been used in the off-line scenario.
The authors of these works do not tackle the distribution mismatch with its impact to the gradient estimation, but rather consider to prevent the policy taking OOD actions.
In this line of thoughts, Batch-Constrained Q-learning 
(BCQ) \cite{fujimoto_off-policy_2019} introduces a regularization that keeps the policy close to the behavioral policy,
Bootstraping Error Accumulation Reduction (BEAR) \cite{kumar_stabilizing_2019} considers instead a pessimistic estimate of the $Q$ function which penalizes the uncertainty, Behavior Regularized Actor Critic (BRAC) \cite{wu_behavior_2019} proposes instead a generalization of the aforementioned approaches.  
\\
\textbf{Importance-Sampling:}
 One way to obtain an unbiased estimate of the policy gradient in an off-policy scenario is to re-weight every trajectory via importance sampling \cite{meuleau_exploration_2001,shelton_policy_2001, peshkin_learning_2002}. An example of the gradient estimation via G(PO)MDP \cite{baxter_infinite-horizon_2001} with importance sampling is given by
\begin{align}
\nabla_{\theta} J_{\pi} = \mathbb{E}\left[\sum_{t=0}^{T-1} \rho_t \left(\prod_{j=0}^{t-1}\gamma_j\right)r_t \sum_{i=0}^t \nabla_{\theta} \log \pi_{\theta}(\action_i | \state_i) \right] \label{eq:pwis},
\end{align}
where $\rho_t = \prod_{z=0}^{t} \pi_{\theta}(\action_z | \state_z)/\pi_{\beta}(\action_z | \state_z)$.
This technique applies only to stochastic policies and requires the knowledge of the behavioral policy $\pi_\beta$. Moreover, Equation~\eqref{eq:pwis} shows that path-wise importance sampling (PWIS) needs a trajectory-based dataset, since it needs to keep track of the past in the correction term $\rho_t$, hence introducing more restrictions on its applicability.
Additionally, importance sampling suffers from high variance  \cite{owen_monte_2013}. 
Recent works have helped to make PWIS more reliable. For example, \cite{imani_off-policy_2018}, building on the emphatic weighting framework \cite{sutton_emphatic_2016},  proposed a trade-off between PWIS and semi-gradient approaches.  Another possibility consists in restricting the gradient improvement to a safe-region, where the importance sampling does not suffer from too high variance \cite{metelli_policy_2018}. Another interesting line of research is to estimate the importance sampling correction on a state-distribution level instead of on the classic trajectory level \cite{liu_breaking_2018, liu_off-policy_2019, nachum_algaedice:_2019}. We note that despite the nice theoretical properties, all these promising algorithms have been applied on low-dimensional problems, as importance sampling suffers from the curse of dimensionality. 
\\
\subsection{Model-Based}
\reviewAll{
While all model-free offline techniques rely purely on the samples contained in the offline dataset, model-based techniques aim to approximate the transitions from that  dataset. The approximated model is then used to generate new artificial samples.
These approaches do not suffer from the distribution mismatch problem described above, as the samples can be generated on-policy by the model (and therefore, in principle, any online algorithm can  be used). 
This advantage is overshadowed by the disadvantage of having unrealistic trajectories generated by the approximated model.  
For this reason, many recent works focus on discarding unrealistic samples, relying on a pessimistic approach towards uncertainty.
In some sense, we encounter again the problem of quantifying our uncertainty - given the limited information contained in the dataset - and to prevent the policy or the model to use or generate uncertain samples.
In this view, the Probabilistic Ensemble with Trajectory Sampling (PETS) \cite{chua_deep_2018} uses an ensemble of probabilistic models of the dynamics to quantify both the epistemic and the aleatoric uncertainties, and uses model predictive control to act in the real environment.
Model Based Offline Planning (MBOP) \cite{argenson_model-based_2020}  uses a behavioral cloning policy as a prior to the planning algorithm. Plan Online and Learn Offline (POLO) \cite{lowrey2018plan} proposes to learn the value function from a known model of the dynamics and to use model predictive control on the environment. Both Model Based Reinforcement Learning (MOReL) \cite{kidambi_morel_2020} and Model-based Offline Policy Optimization (MOPO) \cite{yu_mopo_2020} learn instead the model to train a parametric policy. MOReL learns a pessimistic version of the real MDP, by partitioning the state-action space in known and unknown and penalizing the policies visiting the unknown region. Instead, MOPO builds a pessimistic MDP by introducing a reward penalization toward model uncertainty.
}

\reviewAll{
NOPG, instead,  provides a theoretical framework built on a nonparametric approximation of the reward and the state transition. Such approximation is used to compute the policy gradient in closed-form. For this reason, our method differs from classic model-based solutions since it does not generate synthetic trajectories.
Most of the state-of-the-art model-free offline algorithms, on the other hand, utilize a biased and inconsistent gradient estimate. Instead, our approach delivers a full-gradient estimate that allows a trade-off between bias and variance. The quality of the gradient estimate results in a particularly sample-efficient policy optimization, as seen in the empirical section.}
%\subsection{Model Based}
%Another natural approach which comes to mind when thinking about off-policy optimization, is to use a learned model of the transition. This model allows to generate new samples and therefore to optimize the policy potentially off-line. The proclaimed efficiency of model-based techniques relies on the fact that they allow off-policy optimization. However, model-based techniques are also problematic: the model error propagates in the Bellman recursion (or in the number of steps, if we prefer), often resulting in bad policy improvements. PILCO \cite{deisenroth_pilco:_2011} aims to optimize the policy using probabilistic inference based on Gaussian Processes to model the estimation's uncertainty. However, it works on a finite horizon setting, is restricted to unimodal state-transitions, and a particular shape of reward. PETS \cite{chua_deep_2018}, an improved version of PILCO, builds a probabilitstic model using a bootstrapped ensemble of neural-networks, and propagates the state-distribution using particles. This method, still requires a finite horizon. Furthermore, PETS does not make use of a parametrized policy, but instead a model predictive control. The controller requires multiple neural network evaluations, which can result in an issue when interacting with a real-time system. Our method, in contrast, works on the infinite-horizon setting, and the usage of a parametrized policy is more suitable for real-time operations. 

\section{Empirical Evaluation}
\label{section:experiments}
In this section, we analyze our method. Therefore, we divide our experiments in two: the analysis of the gradient, and the analysis of the policy optimization using a gradient ascent technique.
The analysis of the gradient comprises an empirical evaluation of the bias, the variance and the gradient direction w.r.t. the ground truth, in relation to some quantities such as the size of the dataset or its degree of ``off-policiness''. 
In the policy optimization analysis, instead, we aim to both compare the sample efficiency of our method in comparison to state-of-the-art policy gradient algorithms, and to study its applicability to unstructured and human-demonstrated datasets.
\subsection{Benchmarking Tasks}
In the following, we give a brief description of the tasks involved in the empirical analysis.
\subsubsection{Linear Quadratic Gaussian Controller}\label{sec:LQG}
A very classical control problem consists of linear dynamics, quadratic reward and Gaussian noise. The main advantage of this control problem relies in the fact that it is fully solvable in closed-form, using the Riccati equations, which makes it appropriate for verifying the correctness of our algorithm.
In our specific scenario, we have a policy encoded with two parameters for illustration purposes. The LQG is defined as
\begin{align}
	\max_{\bm{\theta}} &  \sum_{t=0}^\infty \gamma^t r_t \nonumber \\
	\text{s.t.}\quad  &\state_{t+1} =  A \state_t + B \action_t ;\quad  r_t  = - \state_t^\intercal Q \state_t - \action_t^\intercal R \action_t \nonumber \\
	& \action_{t+1} = \Theta \state_t + \Sigma \epsilon_t; \quad \epsilon_t \sim \mathcal{N}(0, I), \nonumber  
\end{align}  
with $A$, $B$, $Q$, $R$, $\Sigma$ diagonal matrix and $\Theta = \mathrm{diag}(\bm{\theta})$ where $\bm{\theta}$ are considered the policy's parameters.  
%\begin{eqnarray}
%A & = &\begin{bmatrix}
%a_1 & 0 \\
%0 & a_2
%\end{bmatrix}; \quad B = \begin{bmatrix}
%b_1 & 0 \\
%0 & b_2
%\end{bmatrix}; \quad Q =  \begin{bmatrix}
%r_1 & 0 \\
%0 & r_2
%\end{bmatrix}; \nonumber \\
% \quad R  &  = &\begin{bmatrix}
%q_1 & 0 \\
%0 & q_2
%\end{bmatrix}; \quad 
%\Theta  =  \begin{bmatrix}
%\theta_1 & 0 \\
%0 & \theta_2
%\end{bmatrix}; \quad \Sigma  =  \begin{bmatrix}
%\sigma_1 & 0 \\
%0 & \sigma_2
%\end{bmatrix}. \nonumber 
%\end{eqnarray}
In the stochastic policy experiments, $\pi_\theta(\action|\state) = \mathcal{N}(\action | \Theta\state;\Sigma)$, while for the deterministic case $\Sigma=\mathbf{0}$ and $\pi_\theta(\state) = \Theta\state$.
For further details, please refer to the supplementary material.
\subsubsection{OpenAI Pendulum-v0}
The OpenAI Pendulum-v0 \cite{brockman_openai_2016} is a popular benchmark in reinforcement learning. It simulates a simple under-actuated inverted-pendulum. The goal is to swing the pendulum until it reaches the top position, and then to keep it stable.
The state of the system is fully described by the angle of the pendulum $\omega$ and its angular velocity $\dot{\omega}$. The applied torque $\tau \in [-2, 2]$ corresponds to the agent's action.
One of the advantages of such a system, is that its well-known value function is two-dimensional.      
\subsubsection{Quanser Cart-pole}
The cart-pole is another classical task in reinforcement learning. It consists of an actuated cart moving on a track, to which a pole is attached. The goal is to actuate the cart in a way to balance the pole in the top position. 
Differently from the inverted pendulum, the system has a further degree of complexity, and the state space requires the position on the track $x$, the velocity of the cart $\dot{x}$, the angle of the pendulum $\omega$ and its angular velocity $\dot{\omega}$.  
\subsubsection{OpenAI Mountain-Car}
The mountain-car (also known as car-on-hill), consists on an under-powered car that must reach the top of a hill. The car is placed in the valley connecting two hills. In order to reach the goal position, it must first go in opposite direction in order to gain momentum.
Its state is described by the $x$-position of the car, and by its velocity $\dot{x}$. The episodes terminate when the car reaches the goal.
In contrast to the swing-up pendulum, which is hardly controllable by a human-being, this car system is ideal to provide human-demonstrated data.
\reviewAll{
\subsubsection{U-Maze}
U-Maze is an environment from the D4RL dataset \cite{fu_d4rl_2020}. It consist of a simple maze with a 2d shape, where a ball should reach a goal position. The state representation is 4-dimensional (2d position and velocity), and the 2d action represents the velocity of the ball.  
\subsubsection{Hopper}
The Hopper is a popular one-legged robot, with 11-dimensional state and 3-dimensional action space, that should hop forward as fast as possible in a two-dimensional world. We use the implementation offered by MuJoCo \cite{todorov2012mujoco}. Also in this case, we test NOPG on a dataset provided by D4RL.  
}
%\subsubsection{Throttle-Valve System}
%A throttle-valve system is a technical device used for regulating a flow of a fluid or a gas. The control of the throttle-vale is challenging due to its non-linear dynamics.
%Its discretized dynamics are governed by 
%\begin{equation}
%	\mathbf{s}_{t+1} =
%	\state_t + \delta \left(
%	\begin{bmatrix}
%	0 & 1 \\
%	 K_s & K_d  
%	\end{bmatrix}
%	\mathbf{s}_t
%	+
%	\begin{bmatrix}
%	0 \\
%	C_s - K_f \mathrm{sgn}(\omega_t) + \pi_\theta(\mathbf{s_t})
%	\end{bmatrix}\right) \nonumber 
%\end{equation}
%where $\mathbf{s} = [\alpha_t, \omega_t]^\intercal$, $\delta \in \mathbb{R}^+$ represent the discretized time and the policy $\pi_\theta$ is a mapping between a two-dimensional state-space and a one-dimensional action-space.
%This system provides an interesting real-world application of our algorithm, and we will show that our algorithm is able to leverage on the data acquired with an classical optimal-control approach and to obtain better performance.
%\begin{figure}[b]
%	\includegraphics[width=0.975 \columnwidth]{plots/returndet.pdf}
%	\caption{In order to generate off-policy datasets, we interpolate the optimization policy's parameters $(\theta_1, \theta_2)$ with the parameters $(\theta_1', \theta_2')$. On the interpolation lines, the gradient changes both in magnitude and direction. \label{fig:returnlandscape} }
%\end{figure}
\begin{table}[]
	\begin{tabularx}{\linewidth}{|l|L|c|}
		\hline
		\rowcolor{gray!50} Acronym & Description & Typology \\
		\hline
		NOPG-D & Our method with deterministic policy. & \multirow{2}{*}{NOPG}  \\
		\rowcolor{gray!25}\rowcolor{gray!25} NOPG-S &  Our method with stochastic policy. & \cellcolor{white} \\ \hline
		G(PO)MDP+N & G(PO)MDP with normalized importance sampling. & \multirow{2}{*}{PWIS}\\
		\rowcolor{gray!25}G(PO)MDP+BN & G(PO)MDP with normalized importance sampling and generalized baselines. & \cellcolor{white}  \\ \hline
		DPG+Q  & Offline version of the deterministic policy gradient theorem with an oracle for the $Q$-function. & \multirow{2}{*}{SG}    \\
		\rowcolor{gray!25}DDPG & Deep Deterministic Policy Gradient. & \cellcolor{white}\\
		TD3 & Improved version of DDPG. &  \\
		\rowcolor{gray!25}SAC & Soft Actor Critic. &   \cellcolor{white}\\
		BEAR& Bootstrapping Error Accumulation Reduction. &  \\
		\rowcolor{gray!25}BRAC & Behavior Regularized Actor Critic. & \cellcolor{white}  \\
		\hline
		MOPO & Model-based Offline Policy Optimization. &  
		\multirow{2}{*}{MB} \\
		\rowcolor{gray!25}MOReL & Model-Based Offline Reinforcement Learning. &  \cellcolor{white}
		\\
		\hline
	\end{tabularx}
	\caption{Acronyms used in the paper to refer to practical implementation of the algorithms (SG: semi-gradient, PWIS: path-wise importance sampling, MB: model based).\label{table:algorithms}}
\end{table}
\begin{figure}[b]
	\centering
	\includegraphics[width=0.7 \columnwidth]{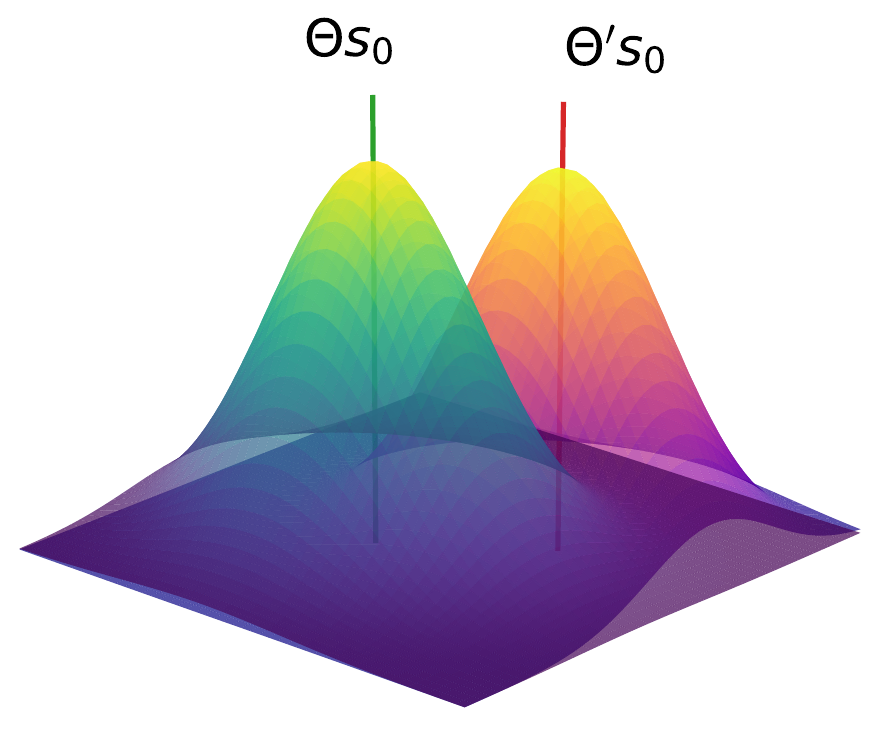}
	\caption{Evaluated in the initial state, the optimization policy having parameters $\theta_1, \theta_2$ and the behavioral policy having parameters $\theta_1', \theta_2'$ exhibit a fair distance in probability space.\label{fig:gaussian} }
\end{figure}
 \begin{figure*}[ht!]
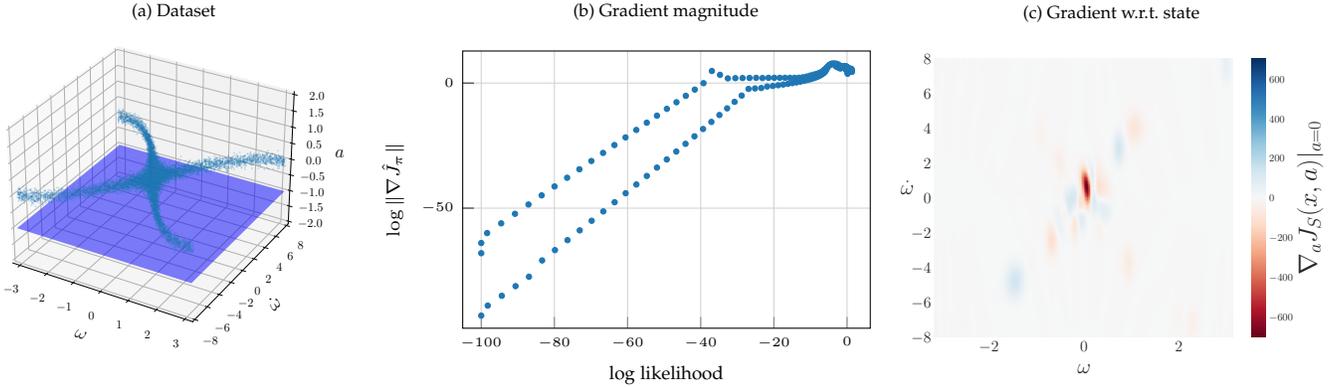

 	% \centering
 	\input{plots/pendulum-dataset.tikz}
 	% This file was created by tikzplotlib v0.9.8.
\begin{tikzpicture}

\definecolor{color0}{rgb}{0.917647058823529,0.917647058823529,0.949019607843137}
\definecolor{color1}{rgb}{0.12156862745098,0.466666666666667,0.705882352941177}

\begin{axis}[
%axis background/.style={fill=color0},
%axis line style={white},
title=\textsuperscript{(b) Gradient magnitude},
height=5.27cm,
width=7cm,
tick align=inside,
tick pos=left,
x grid style={white},
xlabel={log likelihood},
xmajorgrids,
xmin=-105.064175128937, xmax=6.34767770767212,
x grid style={white!82.74509803921568!black},
y grid style={white!82.74509803921568!black},
xlabel style={font=\scriptsize},
ylabel style={font=\scriptsize},
yticklabel style={font=\tiny},
xticklabel style={font=\tiny},
ylabel={$\log\|\nabla\hat{J}_\pi\|$},
ymajorgrids,
ymin=-97.9892494964593, ymax=12.7983847272458,
ytick style={color=white!15!black}
]
\addplot [draw=color1, fill=color1, mark=*, only marks, mark size=1]
table{%
x  y
-100 -68.0034844012202
-100 -63.9397150508587
-98.3885879516602 -59.9651292056444
-94.5401840209961 -56.0733634841627
-90.772834777832 -52.264660234152
-87.087532043457 -48.5389896900818
-83.4843444824219 -44.8964336193111
-79.9632110595703 -41.3369834765184
-76.5240631103516 -37.8606481769035
-73.1669006347656 -34.4674979236737
-69.8916320800781 -31.1575185886519
-66.6982345581055 -27.9307384127891
-63.5866355895996 -24.7872194357847
-60.5567474365234 -21.7269757639353
-57.6085357666016 -18.7500314464979
-54.7419128417969 -15.8564554995906
-51.9567909240723 -13.0462587715276
-49.2530975341797 -10.3195068102823
-46.6307334899902 -7.67615519583885
-44.0895805358887 -5.1152226366389
-41.6295623779297 -2.62447221041749
-39.2505416870117 -0.0644627069033886
-36.952392578125 4.8051794786341
-34.7349853515625 3.27453383803843
-32.5981674194336 1.86148151379735
-30.5417747497559 1.946935504285
-28.5656299591064 1.98537962760117
-26.6695423126221 2.02499413277082
-24.8532829284668 2.06634901354962
-23.116626739502 1.98831322942611
-21.4592895507812 2.08446030367575
-19.8809814453125 2.0938871765329
-18.3813343048096 2.08381202451233
-16.9599380493164 2.08508964147311
-15.6163339614868 2.07572329474785
-14.349946975708 2.07246734703959
-13.1600875854492 2.13098117183766
-12.0459451675415 2.30933939854973
-11.0065002441406 2.60695047784742
-10.0404949188232 3.00028546206342
-9.14636707305908 3.34939777556312
-8.32214832305908 3.67506950545886
-7.56537771224976 4.05582817047495
-6.87298440933228 4.56078867206782
-6.24118518829346 5.33859675783109
-5.66539573669434 6.09574654711169
-5.14025449752808 6.80936615449384
-4.65974283218384 7.39864271670373
-4.21750497817993 7.75056525707418
-3.80734944343567 7.76258317162284
-3.42382502555847 7.38333017474167
-3.06269955635071 6.68979723426853
-2.72112464904785 6.45528265108714
-2.39740657806396 6.20877768852222
-2.09046840667725 6.01430145886247
-1.79924917221069 5.81621925850373
-1.52231788635254 6.20586053821302
-1.25786375999451 6.44666098895348
-1.00405740737915 6.61983672837629
-0.759608209133148 6.72751364383237
-0.524213492870331 6.75360313431006
-0.298665791749954 6.68805703774885
-0.0845597311854362 6.53328332933312
0.116224482655525 6.28174621543322
0.302021592855453 5.9483099482206
0.471685439348221 5.57096702964207
0.6246577501297 5.23748518458775
0.760859727859497 5.0899117640238
0.880509912967682 5.25121789766796
0.983952641487122 5.46519858500164
1.07155096530914 5.61357159884928
1.14364993572235 5.6495380902167
1.20058107376099 5.55468529532023
1.24267041683197 5.31708943864437
1.27022695541382 4.97837873309629
1.28350257873535 4.67365898859983
1.28265142440796 4.48173233507079
1.26770842075348 4.51450051704709
1.23860108852386 4.69222516640101
1.19518744945526 4.86452620641851
1.13729894161224 5.0933116250884
1.06477677822113 5.40042525029758
0.977497458457947 5.60895438560784
0.875389039516449 5.70647644397844
0.758448004722595 5.69103525213669
0.626737892627716 5.53082994654444
0.480365514755249 5.15262944817574
0.319414556026459 4.31257996245134
0.143854752182961 3.69113713560379
-0.0465543754398823 4.7901823293223
-0.252317607402802 5.24692664456381
-0.474191397428513 5.66042460656357
-0.713093638420105 6.00750995021572
-0.969957590103149 6.29685266957715
-1.24557828903198 6.52434255259699
-1.5404931306839 6.68936474218009
-1.85493588447571 6.80376860216153
-2.18885278701782 7.17858887318516
-2.541916847229 7.5058650300136
-2.91339564323425 7.70218862183024
-3.30165767669678 7.76021109629621
-3.7031397819519 7.69107283293666
-4.11081552505493 7.54649676987859
-4.51291275024414 7.37718774218339
-4.89382457733154 7.0921399069861
-5.23950719833374 6.48678838878649
-5.54617977142334 5.64122552298358
-5.8252387046814 4.88596496317569
-6.09905433654785 4.37672979998866
-6.39200353622437 3.98938228522955
-6.72419452667236 3.62202203853192
-7.10984897613525 3.2523583167352
-7.55825901031494 2.88373813892731
-8.07532119750977 2.52261011640177
-8.66477298736572 2.17245087861107
-9.32904148101807 1.83311316746161
-10.0697631835938 1.5030606075824
-10.888072013855 1.1811475795998
-11.7848033905029 0.866640372231541
-12.7605962753296 0.559461314825605
-13.8159341812134 0.259428704363386
-14.9512214660645 -0.0339879698125498
-16.1667919158936 -0.320527133086939
-17.462911605835 -0.602579678708152
-18.8398246765137 -0.885126677099597
-20.2977237701416 -1.17982902875751
-21.8367824554443 -1.37643501469941
-23.4571647644043 -1.99975792017037
-25.1589984893799 -2.15823514276274
-26.9423999786377 -2.35498424597741
-28.8074932098389 -5.0115012857945
-30.7543487548828 -7.53227102111386
-32.7830619812012 -10.1205983580959
-34.8937187194824 -12.7891919767198
-37.0863609313965 -15.5409292911361
-39.3610725402832 -18.3758866880221
-41.7178993225098 -21.2942240896525
-44.1568908691406 -24.2958744214657
-46.6781005859375 -27.3808217464336
-49.2815475463867 -30.5489854753799
-51.9672775268555 -33.8003227475662
-54.7353477478027 -37.1347680832675
-57.5857353210449 -40.5522057478459
-60.5185050964355 -44.0525057939608
-63.5336875915527 -47.6354792928272
-66.6312713623047 -51.3007813075784
-69.8112869262695 -55.0479524361721
-73.0737380981445 -58.8762031650117
-76.4186706542969 -62.7843189035325
-79.8461227416992 -66.7698618396147
-83.3560180664062 -70.8931454284304
-86.9484252929688 -77.410357089343
-90.6233596801758 -81.1743647197045
-94.3808364868164 -85.0216072174177
-98.2226791381836 -88.9498774142074
-100 -92.9534479408364
};
\end{axis}

\end{tikzpicture}
 	\input{plots/gradient-state-plot.tikz}
 	\caption{\reviewD{(a) The dataset used for the experiment in Section~\ref{sec:trust-region-empirical}. The blue plane represent a policy with constant action. By setting different values of $\theta_3$ we can obtain different policies. For $\theta_3 \approx 0$, the policy is most fitting with the data. (b) When the policy has low log-likelihood, the gradient quickly approaches the zero (i.e., $\|\nabla_\theta \hat{J}_\pi\| \to 0 $). (c) The magnitude of the gradient decreases in low density regions, where the prediction is most uncertain.  }\label{fig:trust-region} }
 \end{figure*}
\subsection{Algorithms Used for Comparisons}
To provide an analysis of the gradient, we compare our algorithm against G(PO)MDP with importance sampling, and with offline DPG (DPG with fixed dataset). Instead of using the na{\"i}ve form of G(PO)MDP with importance sampling, which suffers from high variance, we used the normalized importance sampling \cite{shelton_policy_2013, rubinstein_simulation_2016} (which introduces some bias but drastically reduces the variance), and  the generalized baselines \cite{jie_connection_2010} (which also introduce some bias, as they are estimated from the same dataset). The offline version of DPG, suffers from three different sources of bias: the semi-gradient, the critic approximation and the improper use of the discounted state distribution \cite{thomas_bias_2014,nota_is_2020}. To mitigate these issues and focus more on the semi-gradient contribution to the bias, we provide an oracle $Q$-function (we denote this version as DPG+Q).
\reviewAll{
For the policy improvement, instead, we compare both with online algorithms (Figure~\ref{figure:comparison}) such as TD3 \cite{fujimoto_off-policy_2019} and SAC \cite{haarnoja_soft_2018}, and with offline algorithms (BEAR \cite{kumar_stabilizing_2019}, BRAC \cite{wu_behavior_2019}, MOPO \cite{yu_mopo_2020} and MOReL \cite{kidambi_morel_2020}) as depicted in Figure~\ref{figure:offline-comparison}.}
A full list of the algorithms used in the comparisons with a brief description is available in Table~\ref{table:algorithms}.
 \begin{figure*}[p]
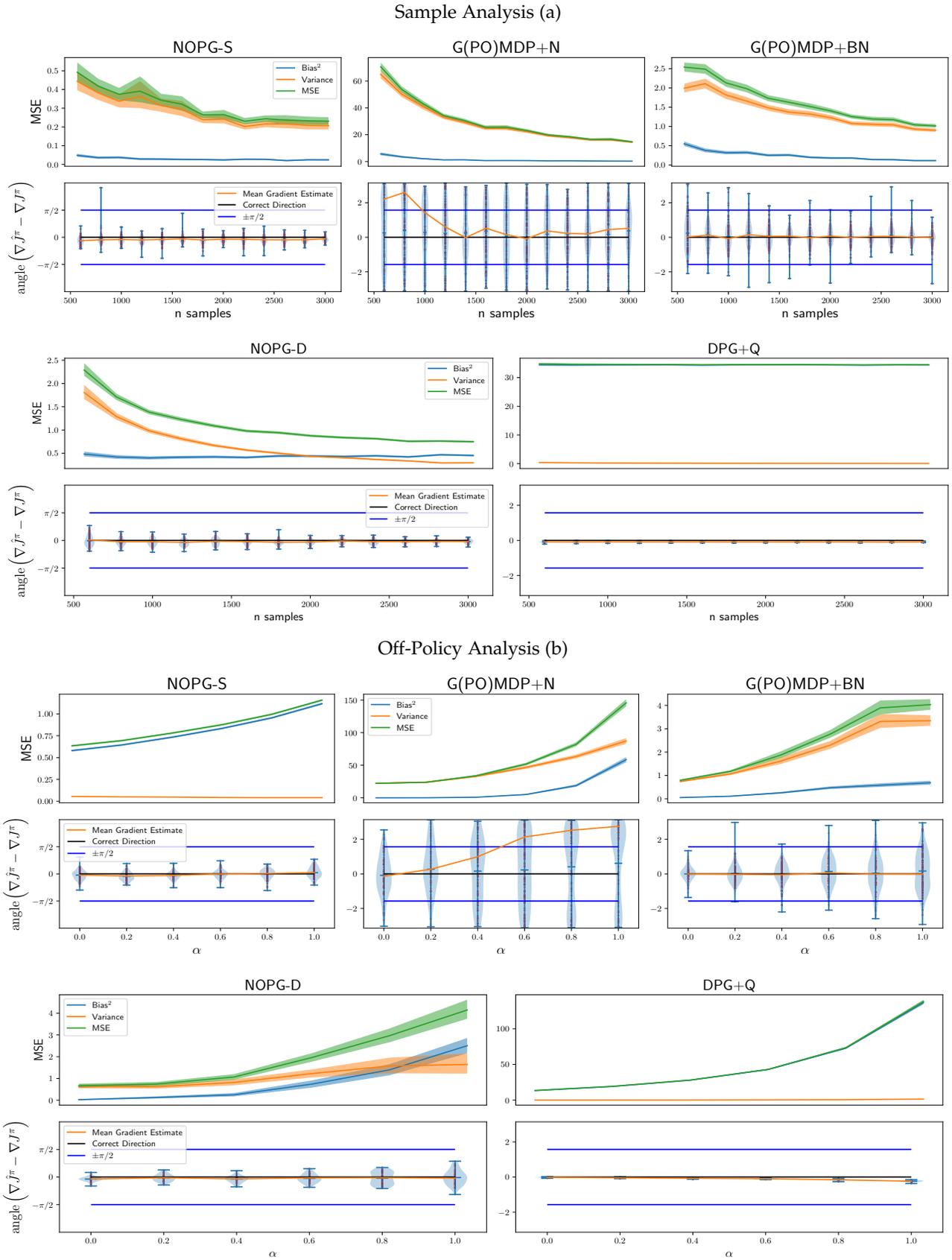

	\centering
	\input{plots/gradient/sample.tikz}\\
	%\vspace{-1em}
	\centering
	\input{plots/gradient/offpolicy.tikz}
	\vspace{-0.5em}
	\caption{Bias, variance, MSE and gradient direction analysis. The MSE plots are equipped with a $95\%$ interval using bootstrapping techniques. The direction analysis plots describe the distribution of angle between the estimates and the ground truth gradient. NOPG exhibits favorable bias, variance and gradient direction compared to PWIS and semi-gradient.\label{fig:sample_analysis}}
\end{figure*}

\subsection{Analysis of the Gradient}
\label{sec:gradient_analysis}
We want to compare the bias and variance of our gradient estimator w.r.t. the already discussed classical estimators.
Therefore, we use the LQG setting described in Section~\ref{sec:LQG}, which allows us to compute the true gradient. Our goal is to estimate the gradient w.r.t. the policy $\pi_\theta$ diagonal parameters $\theta_1, \theta_2$, while sampling from a policy which is a linear combination of $\Theta$ and $\Theta'$. The hyper-parameter $\alpha$ determines the mixing between the two parameters. When $\alpha=1$ the behavioral policy will have  parameters $\Theta'$, while when $\alpha=0$ the dataset will be sampled using $\Theta$.
In Figure~\ref{fig:gaussian}, we can visualize the difference of the
\begin{figure*}[t]
	\centering
	%\hspace{-1.0em}
	\begin{subfigure}[t]{0.5\columnwidth}
		\begin{tikzpicture}
    \begin{axis}[
        axis on top,% ----
        width=3.9cm,
        scale only axis,
        enlargelimits=false,
        xmin=-3.14,
        xmax=3.14,
        ymin=-8,
        ymax=8,
        title={$\tilde{\mu}_{\pi_0}$},
        xlabel={$\omega$},
        ylabel={$\dot{\omega}$},
        xlabel style ={font=\scriptsize},
        ylabel style ={font=\scriptsize, rotate=-90, xshift=7pt},
        yticklabel style={font=\tiny},
        xticklabel style={font=\tiny}
        ]
      \addplot[thick,blue] graphics[xmin=-3.14,ymin=-8,xmax=3.14,ymax=8] {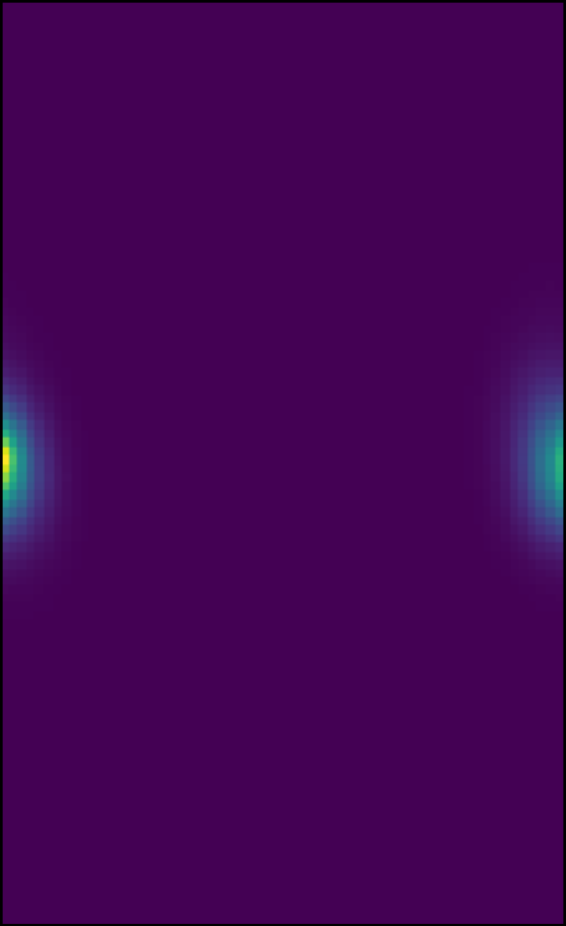};
    \end{axis}
  \end{tikzpicture}
	\end{subfigure}
	\hspace{0.8em}
	\begin{subfigure}[t]{0.5\columnwidth}
		\begin{tikzpicture}
    \begin{axis}[
        axis on top,% ----
        width=3.9cm,
        scale only axis,
        enlargelimits=false,
        xmin=-3.14,
        xmax=3.14,
        ymin=-8,
        ymax=8,
        title={$\tilde{V}_{\pi_0}$},
        xlabel={$\omega$},
        xlabel style ={font=\scriptsize},
        yticklabel =\empty,
        xticklabel style={font=\tiny}
        ]
      \addplot[thick,blue] graphics[xmin=-3.14,ymin=-8,xmax=3.14,ymax=8] {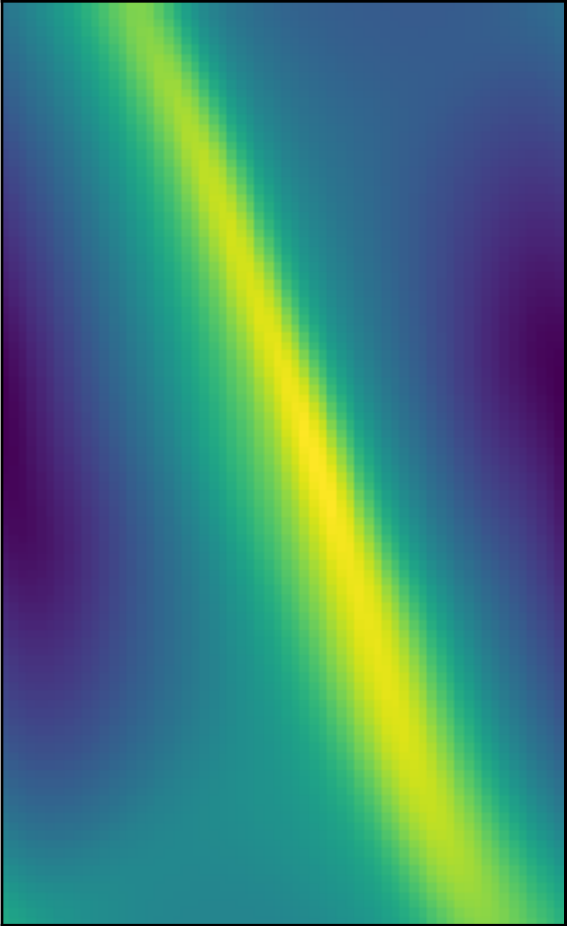};
    \end{axis}
  \end{tikzpicture}
	\end{subfigure}
	\hspace{-1em}
	\begin{subfigure}[t]{0.5\columnwidth}
		\begin{tikzpicture}
    \begin{axis}[
        axis on top,% ----
        width=3.9cm,
        scale only axis,
        enlargelimits=false,
        xmin=-3.14,
        xmax=3.14,
        ymin=-8,
        ymax=8,
        title={$\tilde{\mu}_{\pi_{300}}$},
        xlabel={$\omega$},
		xlabel style ={font=\scriptsize},
		yticklabel =\empty,
		xticklabel style={font=\tiny}
        ]
      \addplot[thick,blue] graphics[xmin=-3.14,ymin=-8,xmax=3.14,ymax=8] {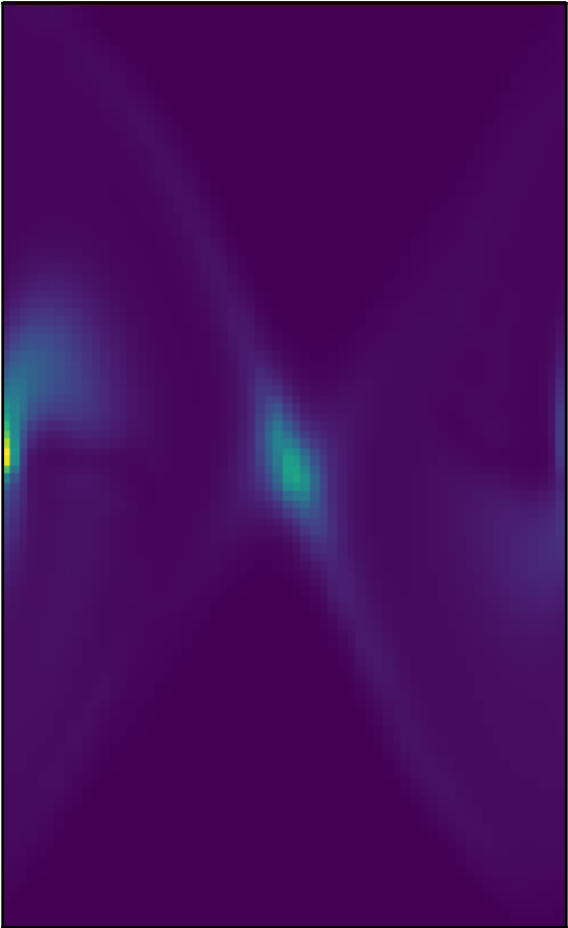};
    \end{axis}
  \end{tikzpicture}
	\end{subfigure}
	\hspace{-1em}
	\begin{subfigure}[t]{0.5\columnwidth}
		\begin{tikzpicture}
    \begin{axis}[
        axis on top,% ----
        width=3.9cm,
        scale only axis,
        enlargelimits=false,
        xmin=-3.14,
        xmax=3.14,
        ymin=-8,
        ymax=8,
        title={$\tilde{V}_{\pi_{300}}$},
        xlabel={$\omega$},
		xlabel style ={font=\scriptsize},
		yticklabel =\empty,
		xticklabel style={font=\tiny}
        ]
      \addplot[thick,blue] graphics[xmin=-3.14,ymin=-8,xmax=3.14,ymax=8] {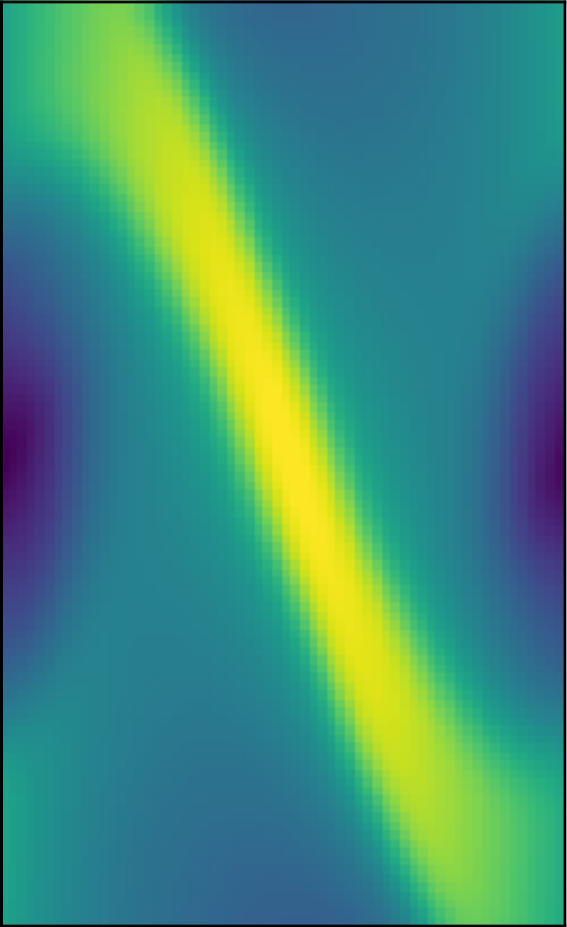};
    \end{axis}
  \end{tikzpicture}
	\end{subfigure}
	\vspace{-1em}
	\caption{A phase portrait of the state distribution $\tilde{\mu}_{\pi}$ and value function $\tilde{V}_{\pi}$ estimated in the swing-up pendulum task with NOPG-D. Green corresponds to higher values. The two leftmost figures show the estimates before any policy improvement, while the two rightmost show them after $300$ offline updates of NOPG-D. Notice that the algorithm finds a very good approximation of the optimal value function and is able to predict that the system will reach the goal state ($(\omega, \dot{\omega}) = (0, 0) $).}
	\label{figure:muv}
\end{figure*}
 \begin{figure}[t]
	% \centering
	\includegraphics[width=0.9 \columnwidth]{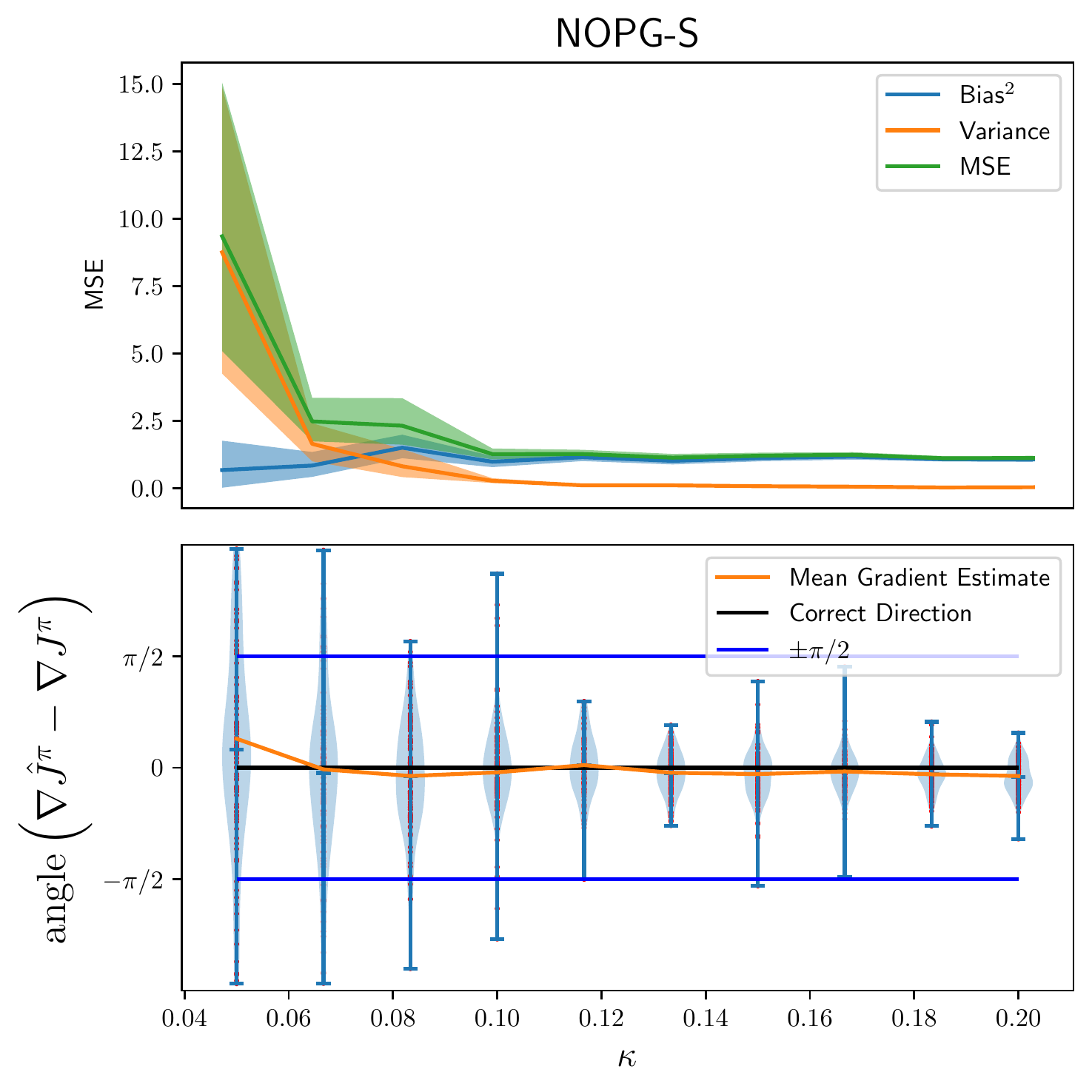}
	\caption{A lower bandwidth corresponds to higher variance, while higher bandwidth increases the bias up to a plateau.\label{fig:parameter} }
\end{figure}
 two policies with parameters $\Theta$ and $\Theta'$. Although not completely disjoint, they are fairly far in the probability space, especially if we take into account that such distance propagates in the length of the trajectories.
\subsubsection{Sample Analysis}
We want to study how the bias, the variance and the direction of the estimated gradient vary w.r.t. the dataset's size. We are particularly interested in the off-policy strategy for sampling, and in this set of experiments we will use constant $\alpha=0.5$. 
Figure~\ref{fig:sample_analysis}a depicts these quantities w.r.t. the number of collected samples. As expected, a general trend for all algorithms is that with a higher number of samples we are able to reduce the variance. The importance sampling based G(PO)MDP algorithms eventually obtain a low bias as well. Remarkably, NOPG has significantly both lower bias and variance, and its gradient direction is also more accurate w.r.t. the G(PO)MDP algorithms (note the different scales of the y-axis). Between DPG+Q and NOPG there is no sensible difference, but we should take into account the already-mentioned advantage of DPG+Q to have access to the true $Q$-function.
%\begin{figure*}
%	\centering
%	\includegraphics[width=1.9\columnwidth]{plots/gradient/stoch-samples.pdf}
%	\caption{Analysis of the gradient w.r.t. the dataset size. In the top row we observe the bias-variance  of the gradient's estimate, corredated with 05\% confidence intervals. On the bottom row, we can observe the distribution of the gradients' directions. It is possible to observe that NOPG exhibits better bias, variance and gradient direction when compared to the importance sampling techniques.}
%\end{figure*}
\subsubsection{Off-Policy Analysis}
We want to estimate the bias and the variance w.r.t. different degrees of ``off-policiness'' $\alpha$, \reviewC{as defined in the beginning of Section~\ref{sec:gradient_analysis}}.
We want to highlight that in the deterministic experiment the behavioral policy remains stochastic. This is needed to ensure the stochastic generation of datasets, which is essential to estimate the bias and the variance of the estimator.
As depicted in Figure~\ref{fig:sample_analysis}, the variance in importance sampling based techniques tends to increase when the dataset is off-policy. On the contrary, NOPG seems to be more subject to an increase of bias. This trend is also noticeable in DPG+Q, where the component of the bias is the one playing a major role in the mean squared error. The gradient direction of NOPG seems however unbiased, while DPG+Q has a slight bias but remarkably less variance (note the different scales of the y-axis). We remark that DPG+Q uses an oracle for the $Q$-function, which supposedly results in lower variance and bias\footnote{Furthermore, we suspect that the particular choice of a LQG task tends to mitigate the problems of DPG, as the fast convergence to a stationary distribution due to the stable attractor, united with the improper discounting, results in a coincidental correction of the state-distribution.}. The positive bias of DPG+Q in the on-policy case ($\alpha=0$) is caused by the improper use of discounting. 
In general, NOPG shows a decrease in bias and variance in order of magnitudes when compared to the other algorithms.
\subsubsection{Bandwidth Analysis}
In the previous analysis, we kept the bandwidth's parameters of our algorithm fixed, even though a dynamic adaptation of this parameter w.r.t. the size of the dataset might have improved the bias/variance trade-off. We are now interested in studying how the bandwidth impacts the gradient estimation. For this purpose, we generated datasets of $1000$ samples with $\alpha=0.5$. We set all the bandwidths of state, action and next state, for each dimension equal to $\kappa$. From Figure~\ref{fig:parameter}, we evince that a lower bandwidth corresponds to a higher variance, while a larger bandwidth approaches a constant bias and the variance tends to zero. This result is in line with the theory.

\reviewB{
\subsubsection{Trust Region}
\label{sec:trust-region-empirical}
In Section~\ref{set:trust-region}, we claimed that the nonparametric technique used has the effect to not consider OOD actions, and more in general, state-action pairs that are in a low density region: in fact the magnitude of the gradient in these regions is low. 
To appreciate this effect, we considered the Pendulum-v0. We generated the data using a Gaussan policy $\mathcal{N}(\bm{\mu}=0, \bm{\Sigma}= 0.2 \mathrm{I})$. Figure~\ref{fig:trust-region}a depicts the generated datataset. Subsequently, we generated a set of linear policies $\avec = \theta_0 \sin \omega + \theta_1 \cos\omega + \theta_2 \dot{\omega} + \theta_3 $ where
the parameters $\theta_0=\theta_1=\theta_2=0$, $\theta_3 \in [-2, 2]$ and $\omega$ and $\dot{\omega}$ represent the angle and the angular velocity of the pendulum. 
When $\theta_3$ is close to $0$, the policy is close to the samples contained in the data, while if $\theta_3$ is close either to $-2$ or $2$, then it is more distant.
In Figure~\ref{fig:trust-region}b shows the estimated policy gradient for different policies. In particular, each point represents the log-likelihood of the policy (w.r.t. the actions contained in the dataset) on the $x$-axis, and the logarithm of the magnitude of the gradient on the $y$-axis. There is a clear correlation between the log-likelihood and the magnitude of the gradient, which tells us that the gradient is close to zero for unlikely policies, supporting, therefore, our claim.
Furthermore, we investigate the contribution of the single state-action pairs to the gradient. The heatmap in Figure~\ref{fig:trust-region}c shows that the highest gradient magnitude appears on the diagonals of the state space. Looking at the generated data, in Figure~\ref{fig:trust-region}a, we notice that the majority of samples are also present on the aforementioned diagonals forming a ``X'' shape. Hence, our experiment shows that the magnitude of the gradient depends also on the density of the state-action space, with lower magnitude in correspondence of lower density of the state space.}
\begin{figure*}[t]
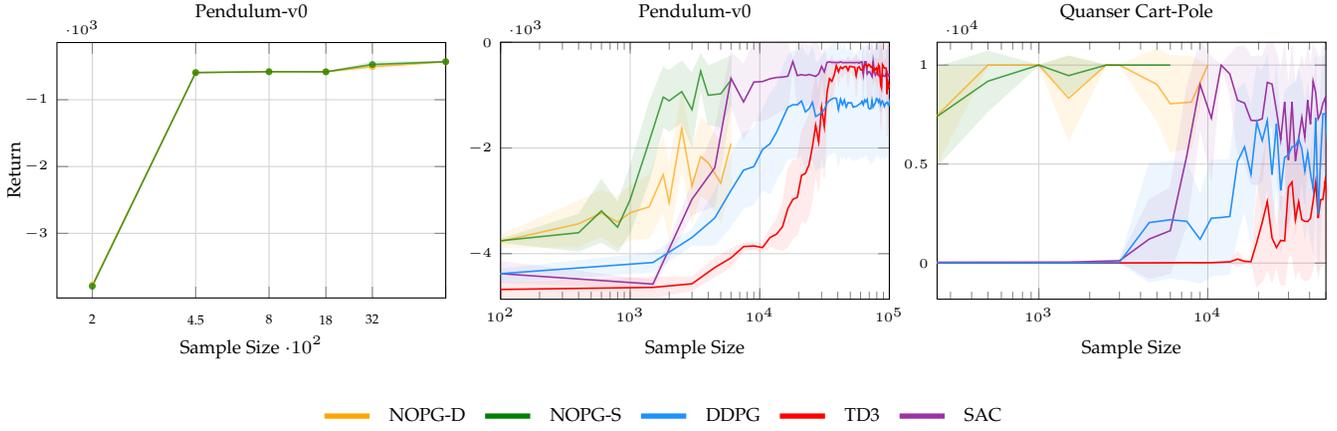

	\begin{subfigure}[t]{0.66\columnwidth}
		% This file was created by matplotlib2tikz v0.6.18.
\begin{tikzpicture}

\definecolor{color0}{rgb}{1,0.647058823529412,0}

\begin{axis}[
height=4.98cm,
width=6.75cm,
xmin=0, xmax=3200.0,
ymin=-3967.86533175936, ymax=-143.376848722761,
try min ticks=3,
tick align=inside,
x grid style={white!82.74509803921568!black},
y grid style={white!82.74509803921568!black},
xmajorticks=true,
ymajorticks=true,
xminorticks=true,
yminorticks=true,
xmode=log,
xmajorgrids,
ymajorgrids,
xtick={200,450, 800, 1250, 1800,3200},
xticklabels={2,4.5,8,18,32},
yticklabel style={font=\tiny},
xticklabel style={font=\tiny, yshift=-3pt},
scaled y ticks=base 10:-3,
scaled x ticks=base 10:-2,
% every x tick scale label/.style={at={(0.82, -0.21)}, anchor = south},
%every y tick scale label/.style={at={(yticklabel cs:0.1)}, anchor = north},
xlabel={Sample Size $\cdot 10^2$},
xlabel style ={font=\scriptsize, yshift=2pt},
ylabel style ={font=\scriptsize, yshift=-5pt},
ylabel={\scriptsize{Return}},
title={Pendulum-v0},
title style={font=\scriptsize},
log ticks with fixed point
]

\path [draw=color0, fill=color0, opacity=0.1] (axis cs:200,-3775.19959014973)
--(axis cs:200,-3775.18814914785)
--(axis cs:450,-591.90217816584)
--(axis cs:800,-576.783659376169)
--(axis cs:1250,-578.560411555018)
--(axis cs:1800,-459.985221504609)
--(axis cs:3200,-430.752099338165)
--(axis cs:3200,-433.108391695221)
--(axis cs:3200,-433.108391695221)
--(axis cs:1800,-545.400524422499)
--(axis cs:1250,-580.745819304394)
--(axis cs:800,-581.346848985241)
--(axis cs:450,-595.488795336325)
--(axis cs:200,-3775.19959014973)
--cycle;

\path [draw=green!50.19607843137255!black, fill=green!50.19607843137255!black, opacity=0.1] (axis cs:200,-3795.62418039938)
--(axis cs:200,-3791.22916095543)
--(axis cs:450,-590.239975072795)
--(axis cs:800,-578.148283359131)
--(axis cs:1250,-579.362788505597)
--(axis cs:1800,-433.869276677966)
--(axis cs:3200,-428.137769813099)
--(axis cs:3200,-431.767095934437)
--(axis cs:3200,-431.767095934437)
--(axis cs:1800,-510.136905039136)
--(axis cs:1250,-585.11098375442)
--(axis cs:800,-582.496690347455)
--(axis cs:450,-596.039198425476)
--(axis cs:200,-3795.62418039938)
--cycle;

\addplot [semithick, color0, opacity=0.7, mark=*, mark size=1, mark options={solid}]
table [row sep=\\]{%
	200	-3775.19386964879 \\
	450	-593.695486751082 \\
	800	-579.065254180705 \\
	1250	-579.653115429706 \\
	1800	-502.692872963554 \\
	3200	-431.930245516693 \\
};
\addplot [semithick, green!50.19607843137255!black, opacity=0.7, mark=*, mark size=1, mark options={solid}]
table [row sep=\\]{%
	200	-3793.42667067741 \\
	450	-593.139586749135 \\
	800	-580.322486853293 \\
	1250	-582.236886130008 \\
	1800	-472.003090858551 \\
	3200	-429.952432873768 \\
};
\end{axis}

\end{tikzpicture}
	\end{subfigure}\hspace{0.5em}
	\begin{subfigure}[t]{.66\columnwidth}
		\input{plots/poles/pendulumrand.tikz}
	\end{subfigure}\hspace{-0.5em}
	\begin{subfigure}[t]{.66\columnwidth}
		\input{plots/poles/cartpole.tikz}
	\end{subfigure}
	%	    \begin{subfigure}[t]{.5\columnwidth}
	%			\raisebox{0.87cm}{\includegraphics[height=3.5cm, width=4cm]{plots/poles/real-cart.png}}
	%		\end{subfigure}
	\vspace{-0.2cm}
	\centering
	\begin{tikzpicture}

\definecolor{color0}{rgb}{1,0.647058823529412,0}
\definecolor{color1}{rgb}{0.117647058823529,0.564705882352941,1}
\definecolor{color2}{rgb}{0.61,0.20,0.64}

\begin{axis}[
height=2cm,
width=8cm,
hide axis,
xmin=10,
xmax=50,
ymin=0,
ymax=0.5,
legend columns=-1,
legend entries={{NOPG-D},{NOPG-S}, {DDPG}, {TD3}, {SAC}},
legend style={at={(1.0,0.1)}, anchor=north, draw=none, font=\scriptsize, column sep=1ex, line width=2 pt},
]

\addlegendimage{no markers, color0}
\addlegendimage{no markers, green!50.19607843137255!black}
\addlegendimage{no markers, color1}
\addlegendimage{no markers, red}
\addlegendimage{no markers, color2}

\end{axis}

\end{tikzpicture}
	\caption{Comparison of NOPG in its deterministic and stochastic versions to state-of-the-art \textbf{online} algorithms on continuous control tasks: Swing-Up Pendulum with \textbf{uniform grid} sampling (left), Swing-Up Pendulum  with the \textbf{random agent} (center) and the Cart-Pole stabilization (right). The figures depict the mean and 95\% confidence interval over 10 trials. \myalg~ outperforms the baselines w.r.t the sample complexity. \textbf{Note the log-scale along the $x$-axis}.}
	\label{figure:comparison}
\end{figure*}
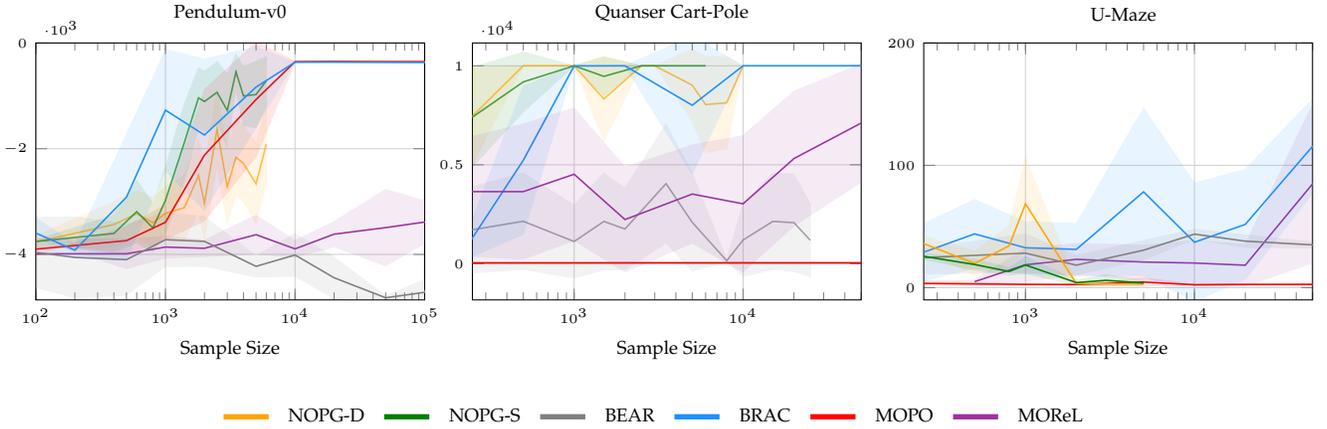
\begin{figure*}[t]
	\begin{subfigure}[t]{.66\columnwidth}
		% This file was created by matplotlib2tikz v0.6.18.
\begin{tikzpicture}

\definecolor{color0}{rgb}{1,0.647058823529412,0}
\definecolor{color1}{rgb}{0.117647058823529,0.564705882352941,1}
\definecolor{color2}{rgb}{0.61,0.20,0.64}

\begin{axis}[
height=5cm,
width=6.75cm,
xmin=100.0, 
xmax=100000,
ymin=-4861.06291217909, ymax=1.03579848121649,
try min ticks=3,
tick align=inside,
x grid style={white!82.74509803921568!black},
y grid style={white!82.74509803921568!black},
xmajorticks=true,
ymajorticks=true,
xminorticks=true,
yminorticks=true,
xmajorgrids,
ymajorgrids,
yticklabel style={font=\tiny},
xticklabel style={font=\tiny},
scaled y ticks=base 10:-3,
xmode=log,
xmajorgrids,
ymajorgrids,
xlabel={Sample Size},
xlabel style={font=\scriptsize},
title={Pendulum-v0},
title style={font=\scriptsize},
]
\addlegendimage{no markers, color0}
\addlegendimage{no markers, green!50.19607843137255!black}
\addlegendimage{no markers, color1}
\addlegendimage{no markers, red}
\addlegendimage{no markers, white!50.19607843137255!black}

\path [draw=color0, fill=color0, opacity=0.1] (axis cs:100,-3879.5718175294)
--(axis cs:100,-3652.07129868322)
--(axis cs:400,-3001.68436656388)
--(axis cs:600,-2797.39050155719)
--(axis cs:800,-2918.67804647365)
--(axis cs:1000,-2723.59697337297)
--(axis cs:1400,-2513.77559199741)
--(axis cs:1800,-1804.59800923639)
--(axis cs:2000,-2373.75207877783)
--(axis cs:2500,-947.249291596133)
--(axis cs:3000,-2032.57831330798)
--(axis cs:3500,-1428.39820782276)
--(axis cs:4000,-1552.56803239566)
--(axis cs:5000,-1922.086035905)
--(axis cs:6000,-1168.72395817863)
--(axis cs:6000,-2661.62730965616)
--(axis cs:6000,-2661.62730965616)
--(axis cs:5000,-3406.45353434095)
--(axis cs:4000,-3002.37728630533)
--(axis cs:3500,-2897.68280873886)
--(axis cs:3000,-3402.4534155717)
--(axis cs:2500,-2313.90452831998)
--(axis cs:2000,-3689.46921373146)
--(axis cs:1800,-3206.95551926132)
--(axis cs:1400,-3714.80887353538)
--(axis cs:1000,-3724.59258623981)
--(axis cs:800,-3881.82731096882)
--(axis cs:600,-3664.32785540097)
--(axis cs:400,-3867.73244500736)
--(axis cs:100,-3879.5718175294)
--cycle;

\path [draw=green!50.19607843137255!black, fill=green!50.19607843137255!black, opacity=0.1] (axis cs:100,-3784.01852118592)
--(axis cs:100,-3727.04294759646)
--(axis cs:400,-3256.4246191887)
--(axis cs:600,-2599.33116272189)
--(axis cs:800,-3037.8290542597)
--(axis cs:1000,-2356.03079351955)
--(axis cs:1400,-1150.55284735467)
--(axis cs:1800,-471.76174308283)
--(axis cs:2000,-535.950105104398)
--(axis cs:2500,-354.743288595078)
--(axis cs:3000,-591.962816707221)
--(axis cs:3500,-219.968688366979)
--(axis cs:4000,-446.028563143085)
--(axis cs:5000,-346.907162557308)
--(axis cs:6000,-261.775574868832)
--(axis cs:6000,-1166.89882424796)
--(axis cs:6000,-1166.89882424796)
--(axis cs:5000,-1601.41128448586)
--(axis cs:4000,-1557.75675791563)
--(axis cs:3500,-871.995218893654)
--(axis cs:3000,-1958.07607667869)
--(axis cs:2500,-1511.80138475548)
--(axis cs:2000,-1678.23130939295)
--(axis cs:1800,-1602.31282767571)
--(axis cs:1400,-2610.0538611951)
--(axis cs:1000,-3595.2041379206)
--(axis cs:800,-3959.15667418118)
--(axis cs:600,-3781.91446348335)
--(axis cs:400,-3950.62700395923)
--(axis cs:100,-3784.01852118592)
--cycle;

\addplot [semithick, color0, opacity=0.7]
table [row sep=\\]{%
	100	-3765.82155810631 \\
	400	-3434.70840578562 \\
	600	-3230.85917847908 \\
	800	-3400.25267872124 \\
	1000	-3224.09477980639 \\
	1400	-3114.2922327664 \\
	1800	-2505.77676424886 \\
	2000	-3031.61064625465 \\
	2500	-1630.57690995806 \\
	3000	-2717.51586443984 \\
	3500	-2163.04050828081 \\
	4000	-2277.47265935049 \\
	5000	-2664.26978512297 \\
	6000	-1915.17563391739 \\
};
\addplot [semithick, green!50.19607843137255!black, opacity=0.7]
table [row sep=\\]{%
	100	-3755.53073439119 \\
	400	-3603.52581157397 \\
	600	-3190.62281310262 \\
	800	-3498.49286422044 \\
	1000	-2975.61746572007 \\
	1400	-1880.30335427488 \\
	1800	-1037.03728537927 \\
	2000	-1107.09070724867 \\
	2500	-933.272336675277 \\
	3000	-1275.01944669296 \\
	3500	-545.981953630317 \\
	4000	-1001.89266052936 \\
	5000	-974.159223521582 \\
	6000	-714.337199558396 \\
};

\path [draw=red, fill=red, opacity=0.1]%, line width=0pt]
(axis cs:100,-3800.81910439733)
--(axis cs:500,-4006.46917685267)
--(axis cs:1000,-3715.89594689068)
--(axis cs:2000,-3360.63102824293)
--(axis cs:5000,-2153.73703565278)
--(axis cs:10000,-366.532750818473)
--(axis cs:20000,-362.884749929778)
--(axis cs:50000,-360.578312317223)
--(axis cs:100000,-362.061421358459)
--(axis cs:200000,-360.370658457379)
--(axis cs:200000,-344.787831532855)
--(axis cs:200000,-344.787831532855)
--(axis cs:100000,-343.958659207948)
--(axis cs:50000,-345.806306823402)
--(axis cs:20000,-336.858469064363)
--(axis cs:10000,-340.74320743348)
--(axis cs:5000,16.0885737387132)
--(axis cs:2000,-881.023506669175)
--(axis cs:1000,-3068.66020545307)
--(axis cs:500,-3472.81910439733)
--(axis cs:100,-4010.81910439733)
--cycle;

\addplot [semithick, red]
table {%
	100 -3901.01
	500 -3739.644140625
	1000 -3392.27807617188
	2000 -2120.82726745605
	5000 -1068.82423095703
	10000 -353.637979125977
	20000 -349.87160949707
	50000 -353.192309570312
	100000 -353.010040283203
	200000 -352.579244995117
};
\path [draw=color2, fill=color2, opacity=0.1]
(axis cs:100,-3707.97431040648)
--(axis cs:500,-4267.7140126487)
--(axis cs:1000,-4064.3722751027)
--(axis cs:2000,-4041.20590627114)
--(axis cs:5000,-3992.01506184734)
--(axis cs:10000,-4051.69425138449)
--(axis cs:20000,-3854.34897353053)
--(axis cs:50000,-4216.38594060887)
--(axis cs:100000,-3788.47277892867)
--(axis cs:200000,-3833.89705235399)
--(axis cs:200000,-3651.63026594428)
--(axis cs:200000,-3651.63026594428)
--(axis cs:100000,-2987.50262042722)
--(axis cs:50000,-2770.32668856256)
--(axis cs:20000,-3380.35226532747)
--(axis cs:10000,-3738.76701682904)
--(axis cs:5000,-3261.78020495076)
--(axis cs:2000,-3728.11203021625)
--(axis cs:1000,-3658.65476091518)
--(axis cs:500,-3707.97431040648)
--(axis cs:100,-3707.97431040648)
--cycle;

\addplot [semithick, color2]
table {%
	100 -3987.84416152759
	500 -3987.84416152759
	1000 -3861.51351800894
	2000 -3884.65896824369
	5000 -3626.89763339905
	10000 -3895.23063410677
	20000 -3617.350619429
	50000 -3493.35631458572
	100000 -3387.98769967795
	200000 -3742.76365914914
};
\path [fill=color1, fill opacity=0.1]
(axis cs:100,-3895.31324666405)
--(axis cs:100,-3299.09480997658)
--(axis cs:200,-3787.22457594069)
--(axis cs:500,-1709.72020586082)
--(axis cs:1000,-110.005158782317)
--(axis cs:2000,-294.876785519283)
--(axis cs:5000,-56.0489169597911)
--(axis cs:10000,-354.905530908581)
--(axis cs:20000,-354.197479481587)
--(axis cs:50000,-358.604393533243)
--(axis cs:100000,-356.189450791771)
--(axis cs:200000,-360.916392882918)
--(axis cs:500000,-359.882100449941)
--(axis cs:1000000,-365.107626545826)
--(axis cs:1000000,-376.228164469799)
--(axis cs:1000000,-376.228164469799)
--(axis cs:500000,-380.51498206959)
--(axis cs:200000,-381.805067078019)
--(axis cs:100000,-379.651992079322)
--(axis cs:50000,-372.847254415976)
--(axis cs:20000,-369.751739268413)
--(axis cs:10000,-370.358653661732)
--(axis cs:5000,-1612.29336819646)
--(axis cs:2000,-3186.07122473462)
--(axis cs:1000,-2430.5640672919)
--(axis cs:500,-4123.9738249009)
--(axis cs:200,-4053.77806077806)
--(axis cs:100,-3895.31324666405)
--cycle;

\addplot [semithick, color1]
table {%
	100 -3597.20402832031
	200 -3920.50131835938
	500 -2916.84701538086
	1000 -1270.28461303711
	2000 -1740.47400512695
	5000 -834.171142578125
	10000 -362.632092285156
	20000 -361.974609375
	50000 -365.725823974609
	100000 -367.920721435547
	200000 -371.360729980469
	500000 -370.198541259766
	1000000 -370.667895507812
};

\path [fill=white!50.19607843137255!black, fill opacity=0.1]
(axis cs:100,-4636.04772567563)
--(axis cs:100,-3287.42176820599)
--(axis cs:200,-3282.19115570587)
--(axis cs:500,-3414.53671281802)
--(axis cs:1000,-3205.9481033758)
--(axis cs:2000,-3271.56743839064)
--(axis cs:5000,-4015.75888130889)
--(axis cs:10000,-3596.05485254866)
--(axis cs:20000,-4037.12350772343)
--(axis cs:50000,-4771.77287985678)
--(axis cs:100000,-4466.83997973117)
--(axis cs:200000,-1338.52616524368)
--(axis cs:500000,-4348.35259047911)
--(axis cs:1000000,-3168.68045137547)
--(axis cs:1000000,-4668.24706614392)
--(axis cs:1000000,-4668.24706614392)
--(axis cs:500000,-4348.35259047911)
--(axis cs:200000,-4471.60011629168)
--(axis cs:100000,-4968.67271378178)
--(axis cs:50000,-4869.75738084608)
--(axis cs:20000,-4845.45834970559)
--(axis cs:10000,-4426.47734601259)
--(axis cs:5000,-4436.04805364076)
--(axis cs:2000,-4234.08109623874)
--(axis cs:1000,-4237.75751256382)
--(axis cs:500,-4783.3292104102)
--(axis cs:200,-4832.43616376943)
--(axis cs:100,-4636.04772567563)
--cycle;

\addplot [semithick, white!50.19607843137255!black]
table {%
	100 -3961.73474694081
	200 -4057.31365973765
	500 -4098.93296161411
	1000 -3721.85280796981
	2000 -3752.82426731469
	5000 -4225.90346747483
	10000 -4011.26609928063
	20000 -4441.29092871451
	50000 -4820.76513035143
	100000 -4717.75634675647
	200000 -2905.06314076768
	500000 -4348.35259047911
	1000000 -3918.46375875969
};
\end{axis}

\end{tikzpicture}
	\end{subfigure}\hspace{-0.5em}
	\begin{subfigure}[t]{.66\columnwidth}
		% This file was created by matplotlib2tikz v0.6.18.
\begin{tikzpicture}

\definecolor{color0}{rgb}{1,0.647058823529412,0}
\definecolor{color1}{rgb}{0.117647058823529,0.564705882352941,1}
\definecolor{color2}{rgb}{0.61,0.20,0.64}

\begin{axis}[
height=5cm,
width=6.75cm,
xmin=250.0, 
xmax=50000,
ymin=-1824.6301808112, ymax=11153.5351782129,
try min ticks=3,
tick align=inside,
x grid style={white!82.74509803921568!black},
y grid style={white!82.74509803921568!black},
xmajorticks=true,
ymajorticks=true,
xminorticks=true,
yminorticks=true,
xmajorgrids,
ymajorgrids,
yticklabel style={font=\tiny},
xticklabel style={font=\tiny},
%every y tick scale label/.style={at={(rel axis cs:1,0)},anchor=south west,inner sep=1pt},
xmode=log,
xmajorgrids,
ymajorgrids,
xlabel={Sample Size},
xlabel style ={font=\scriptsize},
title={Quanser Cart-Pole},
title style={font=\scriptsize, yshift=-1pt, xshift=2pt},
]
\addlegendimage{no markers, color0}
\addlegendimage{no markers, green!50.19607843137255!black}
\addlegendimage{no markers, color1}
\addlegendimage{no markers, red}
\addlegendimage{no markers, white!50.19607843137255!black}
\path [draw=color0, fill=color0, opacity=0.1] (axis cs:100,-336.291324648284)
--(axis cs:100,4547.97717927399)
--(axis cs:250,9900.32311943864)
--(axis cs:500,10000.6655560479)
--(axis cs:1000,9999.9745152272)
--(axis cs:1500,10446.2068521835)
--(axis cs:2500,10000.0015281523)
--(axis cs:3000,9999.97924566496)
--(axis cs:5000,10840.457536161)
--(axis cs:6000,10469.3907297757)
--(axis cs:8000,10446.8051247267)
--(axis cs:10000,9999.98328586249)
--(axis cs:10000,9999.93795909257)
--(axis cs:10000,9999.93795909257)
--(axis cs:8000,5816.05297343455)
--(axis cs:6000,5614.8914810035)
--(axis cs:5000,7203.50769381516)
--(axis cs:3000,9999.95233682166)
--(axis cs:2500,9999.8215386545)
--(axis cs:1500,6202.1414593361)
--(axis cs:1000,9999.90731372464)
--(axis cs:500,9997.38394969849)
--(axis cs:250,5018.75863454248)
--(axis cs:100,-336.291324648284)
--cycle;

\path [draw=green!50.19607843137255!black, fill=green!50.19607843137255!black, opacity=0.1] (axis cs:100,3445.17401657395)
--(axis cs:100,9115.4662608356)
--(axis cs:250,9905.84417940503)
--(axis cs:500,10698.8192178942)
--(axis cs:1000,9999.99188830604)
--(axis cs:1500,10457.9050944474)
--(axis cs:2500,9999.97058500291)
--(axis cs:3000,9999.98617800509)
--(axis cs:5000,9999.99071261827)
--(axis cs:6000,9999.97055579963)
--(axis cs:6000,9999.92467675862)
--(axis cs:6000,9999.92467675862)
--(axis cs:5000,9999.9591166621)
--(axis cs:3000,9999.93021720613)
--(axis cs:2500,9999.9401083207)
--(axis cs:1500,8476.35641171665)
--(axis cs:1000,9999.92950197707)
--(axis cs:500,7674.67205301745)
--(axis cs:250,4907.25217418546)
--(axis cs:100,3445.17401657395)
--cycle;

\path [draw=white!50.19607843137255!black, fill=white!50.19607843137255!black, opacity=0.1] (axis cs:100,-676.865508497188)
--(axis cs:100,2983.64566209359)
--(axis cs:500,4580.73671097934)
--(axis cs:1000,2959.04572588988)
--(axis cs:1500,4576.18529487338)
--(axis cs:2000,3803.87030649452)
--(axis cs:3500,7052.2055275987)
--(axis cs:5000,4544.85103093397)
--(axis cs:8000,216.228225773447)
--(axis cs:10000,3021.79472273405)
--(axis cs:15000,4576.76036650158)
--(axis cs:20000,4531.75363761676)
--(axis cs:25000,3011.12901984341)
--(axis cs:25000,-639.356676493951)
--(axis cs:25000,-639.356676493951)
--(axis cs:20000,-378.498618941438)
--(axis cs:15000,-295.286029782063)
--(axis cs:10000,-620.404753861983)
--(axis cs:8000,89.7743948401283)
--(axis cs:5000,-355.63141829979)
--(axis cs:3500,1047.36200405852)
--(axis cs:2000,-289.71816964416)
--(axis cs:1500,-298.109173280861)
--(axis cs:1000,-708.600517841058)
--(axis cs:500,-287.498574830883)
--(axis cs:100,-676.865508497188)
--cycle;

\addplot [semithick, color0, opacity=0.7]
table [row sep=\\]{%
	100	2105.84292731285 \\
	250	7459.54087699056 \\
	500	9999.02475287318 \\
	1000	9999.94091447592 \\
	1500	8324.17415575981 \\
	2500	9999.9115334034 \\
	3000	9999.96579124331 \\
	5000	9021.98261498809 \\
	6000	8042.1411053896 \\
	8000	8131.42904908061 \\
	10000	9999.96062247753 \\
};
\addplot [semithick, green!50.19607843137255!black, opacity=0.7]
table [row sep=\\]{%
	100	6280.32013870478 \\
	250	7406.54817679524 \\
	500	9186.74563545585 \\
	1000	9999.96069514155 \\
	1500	9467.13075308204 \\
	2500	9999.95534666181 \\
	3000	9999.95819760561 \\
	5000	9999.97491464019 \\
	6000	9999.94761627913 \\
};

\addplot [semithick, white!50.19607843137255!black, opacity=0.7]
table [row sep=\\]{%
	100	1153.3900767982 \\
	500	2146.61906807423 \\
	1000	1125.22260402441 \\
	1500	2139.03806079626 \\
	2000	1757.07606842518 \\
	3500	4049.78376582861 \\
	5000	2094.60980631709 \\
	8000	153.001310306788 \\
	10000	1200.69498443604 \\
	15000	2140.73716835976 \\
	20000	2076.62750933766 \\
	25000	1185.88617167473 \\
};

\path [draw=color1, fill=color1, opacity=0.1]
(axis cs:100,36.3080908745794)
--(axis cs:100,134.416726520631)
--(axis cs:200,60.0188240538325)
--(axis cs:500,8964.95388906766)
--(axis cs:1000,9998.65620902289)
--(axis cs:2000,9999.55412385443)
--(axis cs:5000,11481.6939803676)
--(axis cs:10000,9999.8745074475)
--(axis cs:20000,9999.98789772233)
--(axis cs:50000,9999.92895861425)
--(axis cs:100000,9999.96170819475)
--(axis cs:200000,9999.9499534918)
--(axis cs:500000,9999.97470347144)
--(axis cs:1000000,9999.96880520385)
--(axis cs:1000000,9999.92053646642)
--(axis cs:1000000,9999.92053646642)
--(axis cs:500000,9999.41955132983)
--(axis cs:200000,9999.93547179807)
--(axis cs:100000,9999.93559567379)
--(axis cs:50000,9999.20806123279)
--(axis cs:20000,9999.85807210723)
--(axis cs:10000,9999.731581192)
--(axis cs:5000,4534.62337057255)
--(axis cs:2000,9995.33144065996)
--(axis cs:1000,9989.89016157996)
--(axis cs:500,1485.22687277923)
--(axis cs:200,22.5279970546279)
--(axis cs:100,36.3080908745794)
--cycle;

\addplot [semithick, color1]
table {%
	100 85.3624086976051
	200 41.2734105542302
	500 5225.09038092345
	1000 9994.27318530142
	2000 9997.4427822572
	5000 8008.15867547005
	10000 9999.80304431975
	20000 9999.92298491478
	50000 9999.56850992352
	100000 9999.94865193427
	200000 9999.94271264494
	500000 9999.69712740064
	1000000 9999.94467083514
};
\path [draw=color2, fill=color2, opacity=0.1]
(axis cs:100,5600.76846513341)
--(axis cs:100,500.735268666589)
--(axis cs:500,238.735268666589)
--(axis cs:1000,1177.7719879302)
--(axis cs:2000,-402.440884481127)
--(axis cs:5000,1049.99170065785)
--(axis cs:10000,-411.609358028239)
--(axis cs:20000,1918.95190367183)
--(axis cs:50000,4129.91979096261)
--(axis cs:100000,9505.03544171967)
--(axis cs:200000,5070.54018320638)
--(axis cs:200000,10997.4759367936)
--(axis cs:200000,10997.4759367936)
--(axis cs:100000,10099.6646582803)
--(axis cs:50000,10086.2327440374)
--(axis cs:20000,8710.70693452817)
--(axis cs:10000,6470.12586702824)
--(axis cs:5000,5993.70975234215)
--(axis cs:2000,4871.91644828113)
--(axis cs:1000,7867.2316206698)
--(axis cs:500,7050.76846513341)
--(axis cs:100,5600.76846513341)
--cycle;

\addplot [semithick, color2]
table {%
	100 3644.7518669
	500 3644.7518669
	1000 4522.5018043
	2000 2234.7377819
	5000 3521.8507265
	10000 3029.2582545
	20000 5314.8294191
	50000 7108.0762675
	100000 9802.35005
	200000 8034.00806
};

\path [draw=red, fill=red, opacity=0.1]
(axis cs:100,50.343031607478)
--(axis cs:500,46.3478443996021)
--(axis cs:1000,44.3285505326215)
--(axis cs:2000,49.8492845973649)
--(axis cs:5000,49.922433062446)
--(axis cs:10000,49.8645496510289)
--(axis cs:20000,49.8647612195768)
--(axis cs:50000,49.8510987938674)
--(axis cs:100000,49.9121885891092)
--(axis cs:200000,49.8802471945745)
--(axis cs:200000,49.9894679238337)
--(axis cs:200000,49.9894679238337)
--(axis cs:100000,49.9738015537131)
--(axis cs:50000,49.9789472876756)
--(axis cs:20000,49.9835801500521)
--(axis cs:10000,49.9530063487269)
--(axis cs:5000,49.9714046755911)
--(axis cs:2000,49.9993367710433)
--(axis cs:1000,51.0866235701617)
--(axis cs:100,50.343031607478)
--cycle;

\addplot [semithick, red]
table {%
	100 48.34543800354
	1000 47.7075870513916
	2000 49.9243106842041
	5000 49.9469188690186
	10000 49.9087779998779
	20000 49.9241706848145
	50000 49.9150230407715
	100000 49.9429950714111
	200000 49.9348575592041
};
\end{axis}

\end{tikzpicture}
	\end{subfigure}
	\begin{subfigure}[t]{0.66\columnwidth}
		% This file was created by matplotlib2tikz v0.6.18.
\begin{tikzpicture}

\definecolor{color0}{rgb}{1,0.647058823529412,0}
\definecolor{color1}{rgb}{0.117647058823529,0.564705882352941,1}
\definecolor{color2}{rgb}{0.61,0.20,0.64}

\begin{axis}[
height=5cm,
width=6.75cm,
xmin=250.0, 
xmax=50000,
ymin=-10, ymax=200,
try min ticks=3,
tick align=inside,
x grid style={white!82.74509803921568!black},
y grid style={white!82.74509803921568!black},
xmajorticks=true,
ymajorticks=true,
xminorticks=true,
yminorticks=true,
xmajorgrids,
ymajorgrids,
yticklabel style={font=\tiny},
xticklabel style={font=\tiny},
%every y tick scale label/.style={at={(rel axis cs:1,0)},anchor=south west,inner sep=1pt},
xmode=log,
xmajorgrids,
ymajorgrids,
xlabel={Sample Size},
xlabel style ={font=\scriptsize},
title={U-Maze},
title style={font=\scriptsize, yshift=-1pt, xshift=2pt},
]
\addlegendimage{no markers, color0}
\addlegendimage{no markers, green!50.19607843137255!black}
\addlegendimage{no markers, color1}
\addlegendimage{no markers, red}
\addlegendimage{no markers, white!50.19607843137255!black}
\path [draw=color0, fill=color0, opacity=0.1] (axis cs:100,-336.291324648284)
--(axis cs:100,4547.97717927399)
--(axis cs:250,9900.32311943864)
--(axis cs:500,10000.6655560479)
--(axis cs:1000,9999.9745152272)
--(axis cs:1500,10446.2068521835)
--(axis cs:2500,10000.0015281523)
--(axis cs:3000,9999.97924566496)
--(axis cs:5000,10840.457536161)
--(axis cs:6000,10469.3907297757)
--(axis cs:8000,10446.8051247267)
--(axis cs:10000,9999.98328586249)
--(axis cs:10000,9999.93795909257)
--(axis cs:10000,9999.93795909257)
--(axis cs:8000,5816.05297343455)
--(axis cs:6000,5614.8914810035)
--(axis cs:5000,7203.50769381516)
--(axis cs:3000,9999.95233682166)
--(axis cs:2500,9999.8215386545)
--(axis cs:1500,6202.1414593361)
--(axis cs:1000,9999.90731372464)
--(axis cs:500,9997.38394969849)
--(axis cs:250,5018.75863454248)
--(axis cs:100,-336.291324648284)
--cycle;

\path [draw=green!50.19607843137255!black, fill=green!50.19607843137255!black, opacity=0.1] (axis cs:100,3445.17401657395)
--(axis cs:100,9115.4662608356)
--(axis cs:250,9905.84417940503)
--(axis cs:500,10698.8192178942)
--(axis cs:1000,9999.99188830604)
--(axis cs:1500,10457.9050944474)
--(axis cs:2500,9999.97058500291)
--(axis cs:3000,9999.98617800509)
--(axis cs:5000,9999.99071261827)
--(axis cs:6000,9999.97055579963)
--(axis cs:6000,9999.92467675862)
--(axis cs:6000,9999.92467675862)
--(axis cs:5000,9999.9591166621)
--(axis cs:3000,9999.93021720613)
--(axis cs:2500,9999.9401083207)
--(axis cs:1500,8476.35641171665)
--(axis cs:1000,9999.92950197707)
--(axis cs:500,7674.67205301745)
--(axis cs:250,4907.25217418546)
--(axis cs:100,3445.17401657395)
--cycle;

\addplot [semithick, color0, opacity=0.7]
table [row sep=\\]{%
	100	2105.84292731285 \\
	250	7459.54087699056 \\
	500	9999.02475287318 \\
	1000	9999.94091447592 \\
	1500	8324.17415575981 \\
	2500	9999.9115334034 \\
	3000	9999.96579124331 \\
	5000	9021.98261498809 \\
	6000	8042.1411053896 \\
	8000	8131.42904908061 \\
	10000	9999.96062247753 \\
};
\addplot [semithick, green!50.19607843137255!black, opacity=0.7]
table [row sep=\\]{%
	100	6280.32013870478 \\
	250	7406.54817679524 \\
	500	9186.74563545585 \\
	1000	9999.96069514155 \\
	1500	9467.13075308204 \\
	2500	9999.95534666181 \\
	3000	9999.95819760561 \\
	5000	9999.97491464019 \\
	6000	9999.94761627913 \\
};

\path [draw=color2, fill=color2, opacity=0.1]
(axis cs:500,6.50743259610917)
--(axis cs:500,3.31056740389083)
--(axis cs:1000,5.35428530382275)
--(axis cs:2000,10.2181837357252)
--(axis cs:5000,7.33542866870088)
--(axis cs:10000,1.85532190120219)
--(axis cs:20000,1.39080117368241)
--(axis cs:50000,20.2125080180298)
--(axis cs:100000,19.755289069092)
--(axis cs:200000,6.15349379036897)
--(axis cs:200000,34.658506209631)
--(axis cs:200000,34.658506209631)
--(axis cs:100000,67.572710930908)
--(axis cs:50000,149.41749198197)
--(axis cs:20000,35.3211988263176)
--(axis cs:10000,38.4006780987978)
--(axis cs:5000,34.5825713312991)
--(axis cs:2000,36.0518162642748)
--(axis cs:1000,32.0677146961773)
--(axis cs:500,6.50743259610917)
--cycle;

\addplot [semithick, color2]
table {%
	500 4.909
	1000 18.711
	2000 23.135
	5000 20.959
	10000 20.128
	20000 18.356
	50000 84.815
	100000 43.664
	200000 20.406
};

\path [draw=color1, fill=color1, opacity=0.1]
(axis cs:100,1.39920781506626)
--(axis cs:100,21.8807921849337)
--(axis cs:200,46.0550338307571)
--(axis cs:500,71.9076885276712)
--(axis cs:1000,53.8125787511802)
--(axis cs:2000,52.3284152041945)
--(axis cs:5000,146.667596415468)
--(axis cs:10000,85.3447623975484)
--(axis cs:20000,96.6196093891695)
--(axis cs:50000,153.931542505472)
--(axis cs:100000,166.388133788329)
--(axis cs:200000,111.425832904738)
--(axis cs:500000,27.302985547962)
--(axis cs:1000000,185.847737845135)
--(axis cs:1000000,23.7922621548645)
--(axis cs:1000000,23.7922621548645)
--(axis cs:500000,15.277014452038)
--(axis cs:200000,-5.7658329047382)
--(axis cs:100000,20.4918662116706)
--(axis cs:50000,77.0484574945281)
--(axis cs:20000,6.82039061083049)
--(axis cs:10000,-11.1047623975484)
--(axis cs:5000,9.95240358453202)
--(axis cs:2000,10.2715847958055)
--(axis cs:1000,11.4074212488198)
--(axis cs:500,16.0023114723288)
--(axis cs:200,2.26496616924286)
--(axis cs:100,1.39920781506626)
--cycle;

\addplot [semithick, color1]
table {%
	100 11.64
	200 24.16
	500 43.955
	1000 32.61
	2000 31.3
	5000 78.31
	10000 37.12
	20000 51.72
	50000 115.49
	100000 93.44
	200000 52.83
	500000 21.29
	1000000 104.82
};

\path [draw=white!50.19607843137255!black, fill=white!50.19607843137255!black, opacity=0.1]
(axis cs:100,10.12)
--(axis cs:100,24.12)
--(axis cs:1000,43.8881039242147)
--(axis cs:2000,26.1552277238957)
--(axis cs:5000,38.4980756110923)
--(axis cs:10000,47.7351735144589)
--(axis cs:20000,43.0206467662454)
--(axis cs:50000,37.43041175224)
--(axis cs:100000,36.4444396377284)
--(axis cs:200000,41.4211013373271)
--(axis cs:500000,41.6548416867866)
--(axis cs:1000000,43.5529690074966)
--(axis cs:1000000,32.3322809925034)
--(axis cs:1000000,32.3322809925034)
--(axis cs:500000,26.1251583132134)
--(axis cs:200000,36.8456486626729)
--(axis cs:100000,30.8618103622716)
--(axis cs:50000,32.51583824776)
--(axis cs:20000,33.0181032337546)
--(axis cs:10000,39.8145764855411)
--(axis cs:5000,22.7681743889077)
--(axis cs:2000,10.6395222761043)
--(axis cs:1000,12.5488960757853)
--(axis cs:100,10.12)
--cycle;

\addplot [semithick, white!50.19607843137255!black]
table {%
	100 22.12
	1000 28.2185
	2000 18.397375
	5000 30.633125
	10000 43.774875
	20000 38.019375
	50000 34.973125
	100000 33.653125
	200000 39.133375
	500000 33.89
	1000000 37.942625
};

\path [draw=red, fill=red, opacity=0.1]
(axis cs:500,5.3845636490188)
--(axis cs:100,2.58743625847479)
--(axis cs:1000,2.05689838919548)
--(axis cs:2000,1.84721429373931)
--(axis cs:5000,1.26589238926829)
--(axis cs:10000,1.97588135782019)
--(axis cs:20000,2.05543644501086)
--(axis cs:50000,1.60935587092779)
--(axis cs:100000,2.83097763035002)
--(axis cs:200000,2.66504625601929)
--(axis cs:200000,19.0489535875781)
--(axis cs:200000,19.0489535875781)
--(axis cs:100000,7.23902249362764)
--(axis cs:50000,3.77064417198756)
--(axis cs:20000,3.21656355785016)
--(axis cs:10000,2.80611873087152)
--(axis cs:5000,7.94010774806082)
--(axis cs:2000,3.22078576062013)
--(axis cs:1000,3.33710164036843)
--(axis cs:100,5.3845636490188)
--cycle;

\addplot [semithick, red]
table {%
	100 3.9859999537468
	1000 2.69700001478195
	2000 2.53400002717972
	5000 4.60300006866455
	10000 2.39100004434586
	20000 2.63600000143051
	50000 2.69000002145767
	100000 5.03500006198883
	200000 10.8569999217987
};
\path [draw=color0, fill=color0, opacity=0.1]
(axis cs:100,46.1410628079585)
--(axis cs:100,36.1589371920415)
--(axis cs:200,36.1589371920415)
--(axis cs:500,10.1100328732646)
--(axis cs:800,17.6250482117362)
--(axis cs:1000,30.9473706474164)
--(axis cs:2000,2.98138945515031)
--(axis cs:3000,2.58961695109708)
--(axis cs:5000,2.59446562261407)
--(axis cs:5000,3.49793437738593)
--(axis cs:5000,3.49793437738593)
--(axis cs:3000,3.48398304890292)
--(axis cs:2000,3.72661054484969)
--(axis cs:1000,106.104229352584)
--(axis cs:800,50.7037517882638)
--(axis cs:500,29.6887671267354)
--(axis cs:200,46.1410628079585)
--(axis cs:100,46.1410628079585)
--cycle;

\path [draw=green!50.1960784313725!black, fill=green!50.1960784313725!black, opacity=0.1]
(axis cs:100,29.1647487571613)
--(axis cs:100,26.3876512428387)
--(axis cs:200,26.3876512428387)
--(axis cs:500,14.4116094599666)
--(axis cs:800,10.1996337472427)
--(axis cs:1000,11.6206588343577)
--(axis cs:2000,2.12605434347902)
--(axis cs:3000,1.42382257633945)
--(axis cs:5000,1.52649582888631)
--(axis cs:5000,6.05670417111369)
--(axis cs:5000,6.05670417111369)
--(axis cs:3000,10.5873774236606)
--(axis cs:2000,6.14194565652097)
--(axis cs:1000,25.4497411656423)
--(axis cs:800,16.3135662527573)
--(axis cs:500,23.4471905400334)
--(axis cs:200,29.1647487571613)
--(axis cs:100,29.1647487571613)
--cycle;

\addplot [semithick, color0]
table {%
	100 41.15
	200 41.15
	500 19.8994
	800 34.1644
	1000 68.5258
	2000 3.354
	3000 3.0368
	5000 3.0462
};
\addplot [semithick, green!50.1960784313725!black]
table {%
	100 27.7762
	200 27.7762
	500 18.9294
	800 13.2566
	1000 18.5352
	2000 4.134
	3000 6.0056
	5000 3.7916
};
\end{axis}
\end{tikzpicture}
	\end{subfigure}\hspace{0.5em}
	%	    \begin{subfigure}[t]{.5\columnwidth}
	%			\raisebox{0.87cm}{\includegraphics[height=3.5cm, width=4cm]{plots/poles/real-cart.png}}
	%		\end{subfigure}
	\vspace{-0.2cm}
	\centering
	\begin{tikzpicture}

\definecolor{color0}{rgb}{1,0.647058823529412,0}
\definecolor{color1}{rgb}{0.117647058823529,0.564705882352941,1}
\definecolor{color2}{rgb}{0.61,0.20,0.64}

\begin{axis}[
height=2cm,
width=8cm,
hide axis,
xmin=10,
xmax=50,
ymin=0,
ymax=0.5,
legend columns=-1,
legend entries={{NOPG-D},{NOPG-S}, {BEAR}, {BRAC}, {MOPO}, {MOReL}},
legend style={at={(1.0,0.1)}, anchor=north, draw=none, font=\scriptsize, column sep=1ex, line width=2 pt},
]

% NOPG-D
\addlegendimage{no markers, color0}
% NOPG-S
\addlegendimage{no markers, green!50.19607843137255!black}
% BEAR
\addlegendimage{no markers, white!50.19607843137255!black}
% BRAC
\addlegendimage{no markers, color1}
% MOPO
\addlegendimage{no markers, red}
% MOReL
\addlegendimage{no markers, color2}

\end{axis}

\end{tikzpicture}
	\caption{Comparison of NOPG in its deterministic and stochastic versions to state-of-the-art \textbf{offline} algorithms on continuous control tasks: Swing-Up Pendulum with \textbf{random agent} (left), the Cart-Pole stabilization (center) and U-Maze with D4RL dataset (right). The figures depict the mean and 95\% confidence interval over 10 trials. \myalg~ is competitive with the sample efficiency of the considered baselines. \textbf{Note the log-scale along the $x$-axis}. }
	\label{figure:offline-comparison}
\end{figure*}
 
\subsection{Policy Improvement}
In the previous section, we analyzed the statistical properties of our estimator. Conversely, in this section, we use the NOPG estimate to fully optimize the policy. 
At the current state, NOPG is a batch algorithm, meaning that it receives as input a set of data, and it outputs an optimized policy, without any interaction with the environment. 
We study the sample efficiency of the overall algorithm. We compare it with both other batch and online algorithms. 
Please notice that online algorithms, such as DDPG-On, TD3 and SAC, can acquire more valuable samples during the optimization process. Therefore, in a direct comparison, batch algorithms are in disadvantage.
\subsubsection{Uniform Grid}
In this experiment we analyze the performance of \myalg~under a uniformly sampled dataset, since, as the theory suggests, this scenario should yield the least biased estimate of NOPG. We generate datasets from a grid over the state-action space of the pendulum environment with different granularities. We test our algorithm by optimizing a policy encoded with a neural-network for a fixed amount of iterations. The policy is composed of a single hidden layer with $50$ neurons and ReLU activations. This configuration is fixed across all the different experiments and algorithms for the remainder of this document.
The resulting policy is evaluated on trajectories of $500$ steps starting from the bottom position.
The leftmost plot in Figure~\ref{figure:comparison}, depicts the performance against different dataset sizes, showing that \myalg~is able to solve the task with $450$ samples. Figure~\ref{figure:muv} is an example of the value function and state distribution estimates of \myalg-D at the beginning and after $300$ optimization steps. The ability to predict the state-distribution is particularly interesting for robotics, as it is possible to predict in advance whether the policy will move towards dangerous states. Note that this experiment is not applicable to PWIS, as it does not admit non-trajectory-based data.
\subsubsection{Comparison with Online Algorithms}
\label{sec:offline-empirical}
In contrast to the uniform grid experiment, here we collect the datasets using trajectories from a random agent in the pendulum and the cart-pole environments. 
In the pendulum task, the trajectories are generated starting from the up-right position and applying a policy composed of a mixture of two Gaussians. The policies are evaluated starting from the bottom position with an episode length of $500$ steps. The datasets used in the cart-pole experiments are collected using a uniform policy starting from the upright position until the end of the episode, which occurs when the absolute value of the  angle $\theta$ surpasses $3\deg$. The optimization policy is evaluated for $10^4$ steps. The reward is $r_t = \cos \theta_t$. Since $\theta$ is defined as $0$ in the top-right position, a return of $10^4$ indicates an optimal policy behavior. 

We analyze the sample efficiency by testing \myalg~in an offline fashion with pre-collected samples, on a different number of trajectories. In addition, we provide the learning curve of DDPG, TD3 and SAC using the implementation in Mushroom \cite{deramo_mushroomrl_2020}. For a fixed size of the dataset, we optimize DDPG-Off and NOPG for a fixed number of steps.  For NOPG, which is offline, we select the policy from the last optimization step.
The two rightmost plots in Figure~\ref{figure:comparison} highlight that our algorithm has superior sample efficiency by more than one order of magnitude w.r.t. the considered online algorithms (note the log-scale on the x-axis).

To validate the resulting policy learned in the simulated cartpole (Figure~\ref{fig:tasks}), we apply the final learned controller on a real Quanser cart-pole, and observe a successful stabilizing behavior as can be seen in the supplementary video.

\reviewAll{
\subsubsection{Comparison with Offline Algorithms}
We use the same environments provided in Section~\ref{sec:offline-empirical} to compare against state-of-the-art offline algorithms. To this end, we used the same datasets on Cart-Pole and Pendulum-v0 to train BRAC, BEAR, MOPO and MOReL. Furthermore, to allow a fairer comparison using the D4RL dataset, we tested our algorithm with the aforementioned baselines on the U-Maze.
To perform the usual sample-analysis, we sub-sampled the dataset in smaller datasets. The results of such analysis can be viewed in Figure~\ref{figure:offline-comparison}. NOPG exhibits a competitive performance w.r.t. the baselines. In more detail, BEAR and MOReL perform suboptimally in our experiments, MOPO perform similarly to NOPG in Pendulum-v0, while failing in Quanser Cart-Pole. BRAC exhibit a similar behavior with our NOPG. All the algorithms seems, however, to fail with the U-Maze, probably due to the scarcity of data, and due to its sparse reward.
} 
\begin{figure}[t]
	\centering
	\begin{subfigure}{0.49\columnwidth}
		%\begin{tikzpicture}
%\node (cartpole) at (0,0) {\includegraphics[height=6.2cm, width=8.0cm]{plots/Mountain/mountaincar_samplesize.pdf}};
%\node [above=-0.5cm of cartpole.north] {\textsuperscript{Mountain Car}};
%%\node [below=.45cm of cartpole.south] {\textsuperscript{}};
%\end{tikzpicture}

% This file was created by matplotlib2tikz v0.6.18.
\begin{tikzpicture}

\definecolor{color0}{rgb}{1,0.647058823529412,0}

\begin{axis}[
height=4.6cm,
width=4.8cm,
legend cell align={left},
legend columns=1,
legend entries={{NOPG-D},{NOPG-S}, {Behavioral Cloning}, {Human baseline}},
legend style={at={(0.5,-0.5)}, anchor=center, draw=none, font=\scriptsize, line width=2pt},
tick pos=left,
x grid style={white!82.74509803921568!black},
xlabel={Number of trajectories},
xmajorgrids,
yticklabel style={font=\tiny},
xticklabel style={font=\tiny},
xmin=1, xmax=10,
y grid style={white!82.74509803921568!black},
ylabel={Return},
ymajorgrids,
ymin=-1000.243393203803, ymax=-53.0887427201282,
ytick={-900, -800, -700, -600, -500,-400,-300,-200,-100},
scaled y ticks=base 10:-2,
xlabel style ={font=\scriptsize},
ylabel style ={font=\scriptsize, yshift=-5.0pt},
title={Mountain Car},
title style={font=\scriptsize}
]
\addlegendimage{no markers, color0}
\addlegendimage{no markers, green!50.19607843137255!black}
\addlegendimage{no markers, red!50!blue}
\addlegendimage{no markers, black, dashed}
\path [draw=color0, fill=color0, opacity=0.1] (axis cs:1,-390.043726984244)
--(axis cs:1,-158.356273015756)
--(axis cs:2,-83.5504952505809)
--(axis cs:3,-100.087436344246)
--(axis cs:4,-94.0036470748486)
--(axis cs:5,-71.2321359239316)
--(axis cs:6,-75.2454775399103)
--(axis cs:7,-76.8341531201946)
--(axis cs:8,-72.0814498628687)
--(axis cs:9,-95.216)
--(axis cs:10,-78.8726402026378)
--(axis cs:10,-93.5273597973622)
--(axis cs:10,-93.5273597973622)
--(axis cs:9,-96.784)
--(axis cs:8,-85.9185501371313)
--(axis cs:7,-95.3658468798054)
--(axis cs:6,-90.3545224600897)
--(axis cs:5,-82.7678640760684)
--(axis cs:4,-119.396352925151)
--(axis cs:3,-128.712563655754)
--(axis cs:2,-118.449504749419)
--(axis cs:1,-390.043726984244)
--cycle;

\path [draw=green!50.19607843137255!black, fill=green!50.19607843137255!black, opacity=0.1] (axis cs:1,-410.956186102897)
--(axis cs:1,-155.488258341548)
--(axis cs:2,-81.3742096464664)
--(axis cs:3,-94.0336100314196)
--(axis cs:4,-82.7473071310021)
--(axis cs:5,-82.7377886289309)
--(axis cs:6,-78.8041515604406)
--(axis cs:7,-76.4060938206283)
--(axis cs:8,-85.089922003053)
--(axis cs:9,-76.0101712106379)
--(axis cs:10,-74.5548894076772)
--(axis cs:10,-92.3339994812117)
--(axis cs:10,-92.3339994812117)
--(axis cs:9,-93.3231621226955)
--(axis cs:8,-98.2434113302803)
--(axis cs:7,-93.3439061793717)
--(axis cs:6,-93.4458484395594)
--(axis cs:5,-115.706655815514)
--(axis cs:4,-110.363803980109)
--(axis cs:3,-102.21638996858)
--(axis cs:2,-109.514679242423)
--(axis cs:1,-410.956186102897)
--cycle;

\addplot [very thin, color0, opacity=0.7, mark=*, mark size=1, mark options={solid}]
table [row sep=\\]{%
	1	-274.2 \\
	2	-101 \\
	3	-114.4 \\
	4	-106.7 \\
	5	-77 \\
	6	-82.8 \\
	7	-86.1 \\
	8	-79 \\
	9	-96 \\
	10	-86.2 \\
};
\addplot [very thin, green!50.19607843137255!black, opacity=0.7, mark=*, mark size=1, mark options={solid}]
table [row sep=\\]{%
	1	-283.222222222222 \\
	2	-95.4444444444444 \\
	3	-98.125 \\
	4	-96.5555555555556 \\
	5	-99.2222222222222 \\
	6	-86.125 \\
	7	-84.875 \\
	8	-91.6666666666667 \\
	9	-84.6666666666667 \\
	10	-83.4444444444444 \\
};
\addplot [black, opacity=1., dashed, very thick]
table [row sep=\\]{%
	0	-434.1 \\
	10	-434.1 \\
};
\addplot [red!50!blue, opacity=1., very thick]
table [row sep=\\]{%
	0	-878.1 \\
	10	-878.1 \\
};
\path [draw=black, fill opacity=0] (axis cs:0,-452.243393203803)
--(axis cs:0,-53.0887427201282);

\path [draw=black, fill opacity=0] (axis cs:1,-452.243393203803)
--(axis cs:1,-53.0887427201282);

\path [draw=black, fill opacity=0] (axis cs:1,0)
--(axis cs:10,0);

\path [draw=black, fill opacity=0] (axis cs:1,1)
--(axis cs:10,1);
\end{axis}
\end{tikzpicture}
	\end{subfigure}\hspace{-0.1em}
	\begin{subfigure}{0.49\columnwidth}
		\input{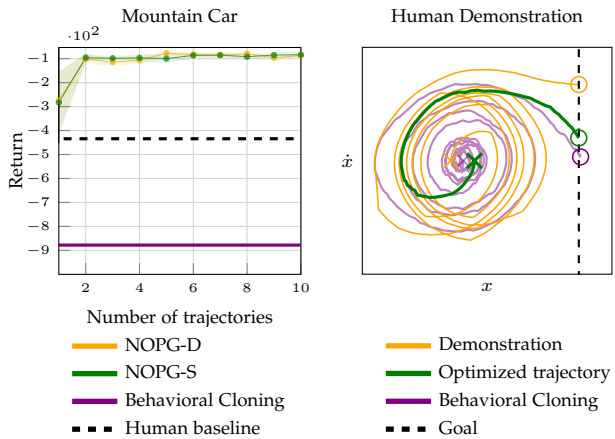}
	\end{subfigure}\vspace{-0.5em}
	\caption{With a small amount of data NOPG is able to reach a policy that surpasses the human demonstrator (dashed line) in the mountain car environment. Depicted are the mean and 95\% confidence over 10 trials (left). An example of a human-demonstrated trajectory and the relative optimized version obtained with \myalg~(right). Although the human trajectories in the dataset are suboptimal, \myalg~converges to an optimal solution (right).}
	\label{figure:mountain}
\end{figure} 
\subsubsection{Human Demonstrated Data}
In robotics, learning from human demonstrations is crucial in order to obtain better sample efficiency and to avoid dangerous policies. This experiment is designed to showcase the ability of our algorithm to deal with such demonstrations without the need for explicit knowledge of the underlying behavioral policy. The experiment is executed in a completely offline fashion after collecting the human dataset, i.e., without any further interaction with the environment. This setting is different from the classical imitation learning and subsequent optimization \cite{kober_policy_2009}.
As an environment we choose the continuous mountain car task from OpenAI. We provide $10$ demonstrations recorded by a human operator and assigned a reward of $-1$ to every step. A demonstration ends when the human operator surpasses the limit of $500$ steps, or arrives at the goal position. The human operator explicitly provides sub-optimal trajectories, as we are interested in analyzing whether \myalg~is able to take advantage of the human demonstrations to learn a better policy than that of the human, without any further interaction with the environment. To obtain a sample analysis, we evaluate NOPG on randomly selected sub-sets of the trajectories from the human demonstrations. Figure~\ref{figure:mountain} shows the average performance as a function of the number of demonstrations as well as an example of a human-demonstrated trajectory. \reviewD{A vanilla behavioral cloning approach trained with the whole dataset leads to a worse performance w.r.t. the human demonstrator. In fact, by simply replicating the demonstrator, the cloned behavior is most of the times not able to reach the flag.}
Notice that both \myalg-S and \myalg-D manage to learn a policy that surpasses the human operator's performance and reach the optimal policy with two demonstrated trajectories. 

\reviewAll{
\subsubsection{Test on a Higher-Dimension Task}
All the considered tasks, are relatively low dimensional, with the higher-dimension of 6, accounting for both states and actions, in the U-Maze.
We tested NOPG also on the Hopper, using the D4RL dataset. NOPG improves w.r.t. the number of samples provided, altough it does not reach a satisfactoy policy (Figure~\ref{figure:hopper}). 
The low performance is probably caused by the poor scaling to high dimensions of nonparametric methods, and the typical need of a large amount of samples due to the complexity of the considered task.}
\begin{figure}[t]
	\begin{subfigure}[t]{0.95\columnwidth}
		% This file was created by matplotlib2tikz v0.6.18.
\begin{tikzpicture}

\definecolor{color0}{rgb}{1,0.647058823529412,0}
\definecolor{color1}{rgb}{0.117647058823529,0.564705882352941,1}
\definecolor{color2}{rgb}{0.61,0.20,0.64}

\begin{axis}[
height=5cm,
width=9cm,
xmin=250.0, 
xmax=20000,
ymin=-10, ymax=200,
try min ticks=3,
tick align=inside,
x grid style={white!82.74509803921568!black},
y grid style={white!82.74509803921568!black},
xmajorticks=true,
ymajorticks=true,
xminorticks=true,
yminorticks=true,
xmajorgrids,
ymajorgrids,
yticklabel style={font=\tiny},
xticklabel style={font=\tiny},
%every y tick scale label/.style={at={(rel axis cs:1,0)},anchor=south west,inner sep=1pt},
xmode=log,
xmajorgrids,
ymajorgrids,
xlabel={Sample Size},
xlabel style ={font=\scriptsize},
title={Hopper},
title style={font=\scriptsize, yshift=-1pt, xshift=2pt},
]
\addlegendimage{no markers, color0}
\addlegendimage{no markers, green!50.19607843137255!black}
\addlegendimage{no markers, color1}
\addlegendimage{no markers, red}
\addlegendimage{no markers, white!50.19607843137255!black}
%\path [draw=color0, fill=color0, opacity=0.1]
%(axis cs:3000,32.6828985773412)
%--(axis cs:3000,-12.2736510585401)
%--(axis cs:5000,31.8383558116445)
%--(axis cs:10000,-117.841443954345)
%--(axis cs:20000,12.9725474335599)
%--(axis cs:20000,249.867683148773)
%--(axis cs:20000,249.867683148773)
%--(axis cs:10000,447.852919785949)
%--(axis cs:5000,136.151723112392)
%--(axis cs:3000,32.6828985773412)
%--cycle;

%\path [draw=green!50.1960784313725!black, fill=green!50.1960784313725!black, opacity=0.1]
%(axis cs:3000,199.613112270506)
%--(axis cs:3000,-53.9168495955546)
%--(axis cs:5000,39.9276702784132)
%--(axis cs:10000,-11.4863609623606)
%--(axis cs:20000,-85.9872720511631)
%--(axis cs:20000,289.771909660168)
%--(axis cs:20000,289.771909660168)
%--(axis cs:10000,177.554903036373)
%--(axis cs:5000,114.42695519871)
%--(axis cs:3000,199.613112270506)
%--cycle;

\addplot [semithick, color0]
table {%
	3000 10.2046237594006
	5000 83.995039462018
	10000 165.005737915802
	20000 131.420115291166
};
\addplot [semithick, green!50.1960784313725!black]
table {%
	3000 72.8481313374758
	5000 77.1773127385616
	10000 83.0342710370064
	20000 101.892318804502
};
\path [draw=color0, fill=color0, opacity=0.1]
(axis cs:100,46.1410628079585)
--(axis cs:100,36.1589371920415)
--(axis cs:200,36.1589371920415)
--(axis cs:500,10.0751033670617)
--(axis cs:800,17.6250482117362)
--(axis cs:1000,30.9473706474164)
--(axis cs:3000,-12.2736510585401)
--(axis cs:5000,31.8383558116445)
--(axis cs:10000,-117.841443954345)
--(axis cs:20000,12.9725474335599)
--(axis cs:20000,249.867683148773)
--(axis cs:20000,249.867683148773)
--(axis cs:10000,447.852919785949)
--(axis cs:5000,136.151723112392)
--(axis cs:3000,32.6828985773412)
--(axis cs:1000,106.104229352584)
--(axis cs:800,50.7037517882638)
--(axis cs:500,33.4852966329383)
--(axis cs:200,46.1410628079585)
--(axis cs:100,46.1410628079585)
--cycle;

\path [draw=green!50.1960784313725!black, fill=green!50.1960784313725!black, opacity=0.1]
(axis cs:100,29.1647487571613)
--(axis cs:100,26.3876512428387)
--(axis cs:200,26.3876512428387)
--(axis cs:500,12.1488241983539)
--(axis cs:800,10.1996337472427)
--(axis cs:1000,11.6206588343577)
--(axis cs:3000,-53.9168495955546)
--(axis cs:5000,39.9276702784132)
--(axis cs:10000,-11.4863609623606)
--(axis cs:20000,-85.9872720511631)
--(axis cs:20000,289.771909660168)
--(axis cs:20000,289.771909660168)
--(axis cs:10000,177.554903036373)
--(axis cs:5000,114.42695519871)
--(axis cs:3000,199.613112270506)
--(axis cs:1000,25.4497411656423)
--(axis cs:800,16.3135662527573)
--(axis cs:500,20.4883758016461)
--(axis cs:200,29.1647487571613)
--(axis cs:100,29.1647487571613)
--cycle;

\addplot [semithick, color0]
table {%
	100 41.15
	200 41.15
	500 21.7802
	800 34.1644
	1000 68.5258
	3000 10.2046237594006
};
\addplot [semithick, green!50.1960784313725!black]
table {%
	100 27.7762
	200 27.7762
	500 16.3186
	800 13.2566
	1000 18.5352
	3000 72.8481313374758
};
\end{axis}
\end{tikzpicture}
	\end{subfigure}
%	\begin{subfigure}[t]{0.45\columnwidth}
%		\input{plots/poles/hopper.tikz}
%	\end{subfigure}
	%	    \begin{subfigure}[t]{.5\columnwidth}
	%			\raisebox{0.87cm}{\includegraphics[height=3.5cm, width=4cm]{plots/poles/real-cart.png}}
	%		\end{subfigure}
	\vspace{-0.2cm}
	\centering
	\begin{tikzpicture}

\definecolor{color0}{rgb}{1,0.647058823529412,0}
\definecolor{color1}{rgb}{0.117647058823529,0.564705882352941,1}
\definecolor{color2}{rgb}{0.61,0.20,0.64}

\begin{axis}[
height=2cm,
width=8cm,
hide axis,
xmin=10,
xmax=50,
ymin=0,
ymax=0.5,
legend columns=-1,
legend entries={{NOPG-D},{NOPG-S}},
legend style={at={(1.0,0.1)}, anchor=north, draw=none, font=\scriptsize, column sep=1ex, line width=2 pt},
]

% NOPG-D
\addlegendimage{no markers, color0}
% NOPG-S
\addlegendimage{no markers, green!50.19607843137255!black}
\end{axis}

\end{tikzpicture}
	\caption{Performance of NOPG on the Hopper-v0, using the D4RL dataset.}
	\label{figure:hopper}
\end{figure}
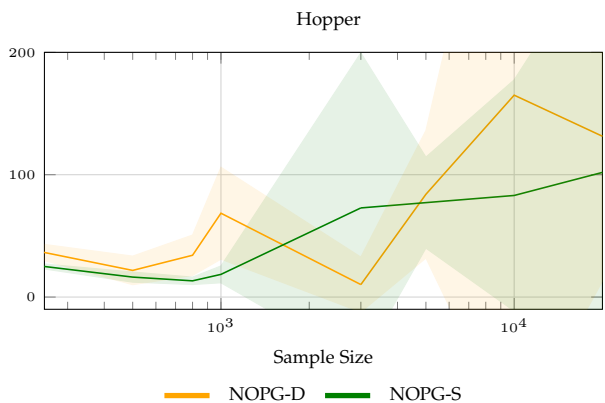

\reviewB{
\subsection{Computational Complexity}
Our nonparametric approach involves a layer of computation that is not usual in classic deep reinforcement learning solutions. In particular, the construction of the matrix $\approxx{\mathbf{P}}_{\pi}^\gamma$ and the inversion of $\bm{\Lambda}_{\pi}$ can be expensive. 
In the following, we analyze the computational resources required by NOPG.
In particular, we use different sizes of the Pendulum-v0 dataset and investigate the time required to compute an iteration of Algorithm~\ref{alg:kbpgalg}. 
As stated in Section~\ref{sec:gradient-estimation}, the iteration time grows quadratically w.r.t. the number of samples contained in the dataset. The computational cost can be lowered by lowering the number of non-zero element of the matrix $\approxx{\mathbf{P}}_{\pi}^\gamma$ (Figure~\ref{fig:complexity}).
However, while our algorithm considers all the samples at every iteration, classic deep reinforcement learning algorithms usually consider only a fixed amount of data (called "mini batch"), and the computational cost of the iterations results, therefore, to be constant. }

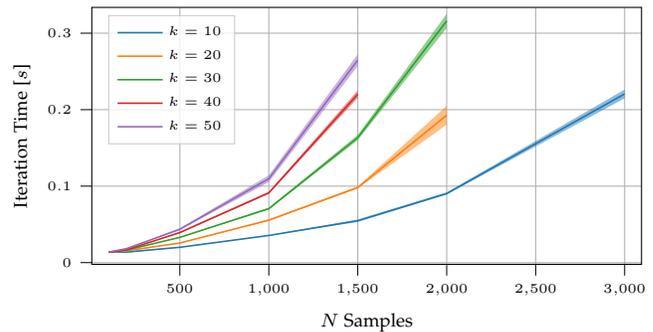
\begin{figure}
	% This file was created by tikzplotlib v0.9.8.
\begin{tikzpicture}

\definecolor{color0}{rgb}{0.12156862745098,0.466666666666667,0.705882352941177}
\definecolor{color1}{rgb}{1,0.498039215686275,0.0549019607843137}
\definecolor{color2}{rgb}{0.172549019607843,0.627450980392157,0.172549019607843}
\definecolor{color3}{rgb}{0.83921568627451,0.152941176470588,0.156862745098039}
\definecolor{color4}{rgb}{0.580392156862745,0.403921568627451,0.741176470588235}

\begin{axis}[
legend cell align={left},
legend style={
  fill opacity=0.8,
  draw opacity=1,
  text opacity=1,
  at={(0.03,0.97)},
  anchor=north west,
  draw=white!80!black,
  font=\tiny
},
height=5cm,
width=\columnwidth,
tick align=outside,
tick pos=left,
x grid style={white!69.0196078431373!black},
xlabel={$N$ Samples},
xmin=5, xmax=3095,
xtick style={color=black},
y grid style={white!69.0196078431373!black},
ylabel={Iteration Time [$s$]},
ymin=-0.00221767009727559, ymax=0.333729191019424,
xmajorgrids,
ymajorgrids,
yticklabel style={font=\tiny},
xticklabel style={font=\tiny},
xlabel style ={font=\scriptsize},
ylabel style ={font=\scriptsize},
ytick style={color=black}
]
\path [fill=color0, fill opacity=0.5]
(axis cs:100,0.0137863737520873)
--(axis cs:100,0.0141375281237265)
--(axis cs:200,0.0145744307539414)
--(axis cs:500,0.0203878986855752)
--(axis cs:1000,0.0357563429910915)
--(axis cs:1500,0.0563818518377724)
--(axis cs:2000,0.0918384657347321)
--(axis cs:3000,0.226368745368315)
--(axis cs:3000,0.214857419131603)
--(axis cs:3000,0.214857419131603)
--(axis cs:2000,0.0888995065247894)
--(axis cs:1500,0.0531576039999012)
--(axis cs:1000,0.0352263251438097)
--(axis cs:500,0.0197280935426149)
--(axis cs:200,0.012994333402824)
--(axis cs:100,0.0137863737520873)
--cycle;

\path [fill=color1, fill opacity=0.5]
(axis cs:100,0.0136462833259693)
--(axis cs:100,0.013798238260131)
--(axis cs:200,0.0160065346475615)
--(axis cs:500,0.0258540678141859)
--(axis cs:1000,0.056471843368973)
--(axis cs:1500,0.0995456495761558)
--(axis cs:2000,0.205205026285465)
--(axis cs:2000,0.180340704305355)
--(axis cs:2000,0.180340704305355)
--(axis cs:1500,0.0969307039632051)
--(axis cs:1000,0.0545247208603564)
--(axis cs:500,0.025385157785315)
--(axis cs:200,0.0150136828240911)
--(axis cs:100,0.0136462833259693)
--cycle;

\path [fill=color2, fill opacity=0.5]
(axis cs:100,0.0136249820459905)
--(axis cs:100,0.0139036218256729)
--(axis cs:200,0.0166643582845644)
--(axis cs:500,0.0333879561070444)
--(axis cs:1000,0.0716938447110711)
--(axis cs:1500,0.16702468217142)
--(axis cs:2000,0.325413117690066)
--(axis cs:2000,0.30669716012162)
--(axis cs:2000,0.30669716012162)
--(axis cs:1500,0.159985631498421)
--(axis cs:1000,0.0693869586302646)
--(axis cs:500,0.0326869980318173)
--(axis cs:200,0.0161317067280495)
--(axis cs:100,0.0136249820459905)
--cycle;

\path [fill=color3, fill opacity=0.5]
(axis cs:100,0.0132744607400364)
--(axis cs:100,0.01363504955456)
--(axis cs:200,0.016937933200962)
--(axis cs:500,0.0401580722984762)
--(axis cs:1000,0.0925724595130578)
--(axis cs:1500,0.225180382798268)
--(axis cs:1500,0.215236642026087)
--(axis cs:1500,0.215236642026087)
--(axis cs:1000,0.0898351845468651)
--(axis cs:500,0.0384852496924906)
--(axis cs:200,0.0166207003914317)
--(axis cs:100,0.0132744607400364)
--cycle;

\path [fill=color4, fill opacity=0.5]
(axis cs:100,0.0134084128377424)
--(axis cs:100,0.0140657256340729)
--(axis cs:200,0.0184324323293347)
--(axis cs:500,0.0448838696174679)
--(axis cs:1000,0.114188568897326)
--(axis cs:1500,0.272655629066306)
--(axis cs:1500,0.257117977128084)
--(axis cs:1500,0.257117977128084)
--(axis cs:1000,0.105224658607828)
--(axis cs:500,0.0424685121947337)
--(axis cs:200,0.0181061156103897)
--(axis cs:100,0.0134084128377424)
--cycle;

\addplot [semithick, color0]
table {%
	100 0.0139619509379069
	200 0.0137843820783827
	500 0.0200579961140951
	1000 0.0354913340674506
	1500 0.0547697279188368
	2000 0.0903689861297607
	3000 0.220613082249959
};
\addlegendentry{$k=10$}
\addplot [semithick, color1]
table {%
	100 0.0137222607930501
	200 0.0155101087358263
	500 0.0256196127997504
	1000 0.0554982821146647
	1500 0.0982381767696804
	2000 0.19277286529541
};
\addlegendentry{$k=20$}
\addplot [semithick, color2]
table {%
	100 0.0137643019358317
	200 0.016398032506307
	500 0.0330374770694309
	1000 0.0705404016706679
	1500 0.16350515683492
	2000 0.316055138905843
};
\addlegendentry{$k=30$}
\addplot [semithick, color3]
table {%
	100 0.0134547551472982
	200 0.0167793167961968
	500 0.0393216609954834
	1000 0.0912038220299615
	1500 0.220208512412177
};
\addlegendentry{$k=40$}
\addplot [semithick, color4]
table {%
	100 0.0137370692359077
	200 0.0182692739698622
	500 0.0436761909061008
	1000 0.109706613752577
	1500 0.264886803097195
};
\addlegendentry{$k=50$}
\end{axis}

\end{tikzpicture}
	\caption{\reviewB{For a fixed amount $k$ of non-zero elements per row of the transition matrix, the iteration time grows quadratically w.r.t. the number of samples. A lower $k$ requires a shorter iteration time.} \label{fig:complexity}}	
\end{figure}

% \subsubsection{An inustrial Application}

\section{Conclusion and Future Work}
In this paper, we presented and analyzed an off-policy gradient technique  \textsl{Nonparametric Off-policy Policy Gradient} (NOPG) \cite{tosatto_nonparametric_2020}. Our estimator overcomes the main issues of the techniques of off-policy gradient estimation. On the one hand, in contrast to semi-gradient approaches, it delivers a full-gradient estimate; and on the other hand, it avoids the high variance of importance sampling by phrasing the problem with nonparametric techniques.
The empirical analysis clearly showed a better gradient estimate in terms of bias, variance, and direction. Our experiments also showed that our method has high sample efficiency and that our algorithm can be behavioral-agnostic and cope with unstructured data.

However, our algorithm, which is built on nonparametric techniques, suffers from the curse of dimensionality. The future work aims to mitigate this problem. We plan to investigate better sparsification techniques united with an adaptive bandwidth. The promising properties of the proposed gradient estimation can, in the future, be adapted to the parametric inference, to extend our approach to the more versatile deep-learning setting.

\ifCLASSOPTIONcompsoc
  \section*{Acknowledgments}
\else
  \section*{Acknowledgment}
\fi

  The research is financially supported by the Bosch-Forschungsstiftung program and the European Union’s Horizon 2020 research and innovation program under grant agreement \#640554 (SKILLS4ROBOTS).

\ifCLASSOPTIONcaptionsoff
  \newpage
\fi



%\bibliographystyle{IEEEtran}

\bibliographystyle{IEEEtran}
\bibliography{zotero}

\begin{thebibliography}{10}
\providecommand{\url}[1]{#1}
\csname url@samestyle\endcsname
\providecommand{\newblock}{\relax}
\providecommand{\bibinfo}[2]{#2}
\providecommand{\BIBentrySTDinterwordspacing}{\spaceskip=0pt\relax}
\providecommand{\BIBentryALTinterwordstretchfactor}{4}
\providecommand{\BIBentryALTinterwordspacing}{\spaceskip=\fontdimen2\font plus
\BIBentryALTinterwordstretchfactor\fontdimen3\font minus
  \fontdimen4\font\relax}
\providecommand{\BIBforeignlanguage}[2]{{%
\expandafter\ifx\csname l@#1\endcsname\relax
\typeout{** WARNING: IEEEtran.bst: No hyphenation pattern has been}%
\typeout{** loaded for the language `#1'. Using the pattern for}%
\typeout{** the default language instead.}%
\else
\language=\csname l@#1\endcsname
\fi
#2}}
\providecommand{\BIBdecl}{\relax}
\BIBdecl

\bibitem{mnih_human-level_2015}
\BIBentryALTinterwordspacing
V.~Mnih, K.~Kavukcuoglu, D.~Silver, A.~A. Rusu, J.~Veness, M.~G. Bellemare,
  A.~Graves, M.~Riedmiller, A.~K. Fidjeland, G.~Ostrovski, S.~Petersen,
  C.~Beattie, A.~Sadik, I.~Antonoglou, H.~King, D.~Kumaran, D.~Wierstra,
  S.~Legg, and D.~Hassabis, ``\BIBforeignlanguage{en}{Human-{Level} {Control}
  {Through} {Deep} {Reinforcement} {Learning}},''
  \emph{\BIBforeignlanguage{en}{Nature}}, vol. 518, no. 7540, pp. 529--533,
  2015. [Online]. Available: \url{http://www.nature.com/articles/nature14236}
\BIBentrySTDinterwordspacing

\bibitem{haarnoja_soft_2018}
T.~Haarnoja, A.~Zhou, P.~Abbeel, and S.~Levine, ``Soft {Actor}-{Critic}:
  {Off}-{Policy} {Maximum} {Entropy} {Deep} {Reinforcement} {Learning} with a
  {Stochastic} {Actor},'' in \emph{Proceeding of the 35th {International}
  {Conference} on {Machine} {Learning}}, 2018, pp. 1856--1865.

\bibitem{schulman_trust_2015}
J.~Schulman, S.~Levine, P.~Moritz, M.~Jordan, and P.~Abbeel, ``Trust {Region}
  {Policy} {Optimization},'' in \emph{Proceedings of the 32nd {International}
  {Conference} on {Machine} {Learning}}, 2015, pp. 1889--1897.

\bibitem{ernst_tree-based_2005}
\BIBentryALTinterwordspacing
D.~Ernst, P.~Geurts, and L.~Wehenkel, ``Tree-{Based} {Batch} {Mode}
  {Reinforcement} {Learning},'' \emph{Journal of Machine Learning Research},
  vol.~6, no. Apr, pp. 503--556, 2005. [Online]. Available:
  \url{http://www.jmlr.org/papers/v6/ernst05a.html}
\BIBentrySTDinterwordspacing

\bibitem{riedmiller_neural_2005}
M.~Riedmiller, ``\BIBforeignlanguage{en}{Neural {Fitted} {Q} {Iteration} –
  {First} {Experiences} with a {Data} {Efficient} {Neural} {Reinforcement}
  {Learning} {Method}},'' in \emph{\BIBforeignlanguage{en}{European
  {Conference} of {Machine} {Learning}}}, ser. Lecture {Notes} in {Computer}
  {Science}.\hskip 1em plus 0.5em minus 0.4em\relax Springer Berlin Heidelberg,
  2005, pp. 317--328.

\bibitem{baird_residual_1995}
L.~Baird, ``Residual {Algorithms}: {Reinforcement} {Learning} with {Function}
  {Approximation},'' \emph{Machine Learning Proceedings}, pp. 30--37, 1995.

\bibitem{lu_non-delusional_2018}
\BIBentryALTinterwordspacing
T.~Lu, D.~Schuurmans, and C.~Boutilier, ``Non-{Delusional} {Q}-learning and
  {Value}-{Iteration},'' in \emph{Advances in {Neural} {Information}
  {Processing} {Systems}}.\hskip 1em plus 0.5em minus 0.4em\relax Curran
  Associates, Inc., 2018, pp. 9949--9959. [Online]. Available:
  \url{http://papers.nips.cc/paper/8200-non-delusional-q-learning-and-value-iteration.pdf}
\BIBentrySTDinterwordspacing

\bibitem{imani_off-policy_2018}
E.~Imani, E.~Graves, and M.~White, ``An {Off}-{Policy} {Policy} {Gradient}
  {Theorem} {Using} {Emphatic} {Weightings},'' in \emph{Advances in {Neural}
  {Information} {Processing} {Systems}}, 2018, pp. 96--106.

\bibitem{degris_off-policy_2012-1}
\BIBentryALTinterwordspacing
T.~Degris, M.~White, and R.~S. Sutton, ``Off-{Policy} {Actor}-{Critic},''
  \emph{arXiv:1205.4839 [cs]}, May 2012, arXiv: 1205.4839. [Online]. Available:
  \url{http://arxiv.org/abs/1205.4839}
\BIBentrySTDinterwordspacing

\bibitem{silver_deterministic_2014}
D.~Silver, G.~Lever, N.~Heess, T.~Degris, D.~Wierstra, and M.~Riedmiller,
  ``\BIBforeignlanguage{en}{Deterministic {Policy} {Gradient} {Algorithms}},''
  in \emph{\BIBforeignlanguage{en}{Proceedings of the 31 st {International}
  {Conference} on {Machine} {Learning}}}, 2014.

\bibitem{lillicrap_continuous_2016}
\BIBentryALTinterwordspacing
T.~P. Lillicrap, J.~J. Hunt, A.~Pritzel, N.~Heess, T.~Erez, Y.~Tassa,
  D.~Silver, and D.~Wierstra, ``Continuous {Control} with {Deep}
  {Reinforcement} {Learning},'' in \emph{International {Conference} on
  {Learning} {Representations}}, 2016, arXiv: 1509.02971. [Online]. Available:
  \url{http://arxiv.org/abs/1509.02971}
\BIBentrySTDinterwordspacing

\bibitem{fujimoto_off-policy_2019}
\BIBentryALTinterwordspacing
S.~Fujimoto, D.~Meger, and D.~Precup, ``Off-{Policy} {Deep} {Reinforcement}
  {Learning} without {Exploration},'' in \emph{Proceeding of the 36th
  {International} {Conference} on {Machine} {Learning}}, 2019, pp. 2052--2062.
  [Online]. Available:
  \url{http://proceedings.mlr.press/v97/fujimoto19a/fujimoto19a.pdf}
\BIBentrySTDinterwordspacing

\bibitem{kumar_conservative_2020}
A.~Kumar, A.~Zhou, G.~Tucker, and S.~Levine, ``Conservative {Q}-{Learning} for
  {Offline} {Reinforcement} {Learning},'' \emph{arXiv preprint
  arXiv:2006.04779}, 2020.

\bibitem{kumar_stabilizing_2019}
A.~Kumar, J.~Fu, G.~Tucker, and S.~Levine, ``Stabilizing {Off}-{Policy}
  {Q}-{Learning} via {Bootstrapping} {Error} {Reduction},'' \emph{arXiv
  preprint arXiv:1906.00949}, 2019.

\bibitem{wu_behavior_2019}
Y.~Wu, G.~Tucker, and O.~Nachum, ``Behavior {Regularized} {Offline}
  {Reinforcement} {Learning},'' \emph{arXiv preprint arXiv:1911.11361}, 2019.

\bibitem{shelton_policy_2001}
\BIBentryALTinterwordspacing
C.~R. Shelton, ``Policy {Improvement} for {POMDPs} {Using} {Normalized}
  {Importance} {Sampling},'' in \emph{Proceedings of the {Seventeenth}
  {Conference} on {Uncertainty} in {Artificial} {Intelligence}}, ser.
  {UAI}'01.\hskip 1em plus 0.5em minus 0.4em\relax Morgan Kaufmann Publishers
  Inc., 2001, pp. 496--503, event-place: Seattle, Washington. [Online].
  Available: \url{http://dl.acm.org/citation.cfm?id=2074022.2074083}
\BIBentrySTDinterwordspacing

\bibitem{meuleau_exploration_2001}
\BIBentryALTinterwordspacing
N.~Meuleau, L.~Peshkin, and K.-E. Kim, ``Exploration in {Gradient}-{Based}
  {Reinforcement} {Learning},'' Massachusetts Institute of Technology, Tech.
  Rep., 2001. [Online]. Available:
  \url{https://dspace.mit.edu/handle/1721.1/6076}
\BIBentrySTDinterwordspacing

\bibitem{peshkin_learning_2002}
\BIBentryALTinterwordspacing
L.~Peshkin and C.~R. Shelton, ``Learning from {Scarce} {Experience},'' in
  \emph{Proceedings of the {Nineteenth} {International} {Conference} on
  {Machine} {Learning}}, 2002, arXiv: cs/0204043. [Online]. Available:
  \url{http://arxiv.org/abs/cs/0204043}
\BIBentrySTDinterwordspacing

\bibitem{yu_mopo_2020}
T.~Yu, G.~Thomas, L.~Yu, S.~Ermon, J.~Zou, S.~Levine, C.~Finn, and T.~Ma,
  ``{MOPO}: {Model}-based {Offline} {Policy} {Optimization},'' in
  \emph{Proceedings of the 33nd {International} {Conference} on {Neural}
  {Information} {Processing} {Systems}}, 2020.

\bibitem{kidambi_morel_2020}
R.~Kidambi, A.~Rajeswaran, P.~Netrapalli, and T.~Joachims, ``{MOReL}:
  {Model}-based {Offline} {Reinforcement} {Learning},'' \emph{arXiv preprint
  arXiv:2005.05951}, 2020.

\bibitem{tosatto_nonparametric_2020}
S.~Tosatto, J.~Carvalho, H.~Abdulsamad, and J.~Peters, ``A {Nonparametric}
  {Off}-{Policy} {Policy} {Gradient},'' in \emph{Proceedings of the 23rd
  {International} {Conference} on {Artificial} {Intelligence} and {Statistics}
  ({AISTATS})}, S.~Chiappa and R.~Calandra, Eds., Palermo, Italy, 2020.

\bibitem{white_unifying_2017}
M.~White, ``Unifying {Task} {Specification} in {Reinforcement} {Learning},'' in
  \emph{Proceedings of the 34th {International} {Conference} on {Machine}
  {Learning}}.\hskip 1em plus 0.5em minus 0.4em\relax JMLR. org, 2017, pp.
  3742--3750.

\bibitem{sutton_policy_2000}
R.~S. Sutton, D.~A. McAllester, S.~P. Singh, and Y.~Mansour, ``Policy
  {Gradient} {Methods} for {Reinforcement} {Learning} with {Function}
  {Approximation},'' in \emph{Advances in {Neural} {Information} {Processing}
  {Systems}}, 2000, pp. 1057--1063.

\bibitem{williams_simple_1992}
R.~J. Williams, ``Simple {Statistical} {Gradient}-{Following} {Algorithms} for
  {Connectionist} {Reinforcement} {Learning},'' \emph{Machine learning},
  vol.~8, no. 3-4, pp. 229--256, 1992.

\bibitem{watkins_q-learning_1992}
\BIBentryALTinterwordspacing
C.~J. C.~H. Watkins and P.~Dayan, ``\BIBforeignlanguage{en}{Q-learning},''
  \emph{\BIBforeignlanguage{en}{Machine Learning}}, vol.~8, no.~3, pp.
  279--292, 1992. [Online]. Available: \url{https://doi.org/10.1007/BF00992698}
\BIBentrySTDinterwordspacing

\bibitem{ormoneit_kernel-based_2002}
\BIBentryALTinterwordspacing
D.~Ormoneit and S.~Sen, ``\BIBforeignlanguage{en}{Kernel-{Based}
  {Reinforcement} {Learning}},'' \emph{\BIBforeignlanguage{en}{Machine
  Learning}}, vol.~49, no.~2, pp. 161--178, 2002. [Online]. Available:
  \url{https://doi.org/10.1023/A:1017928328829}
\BIBentrySTDinterwordspacing

\bibitem{xu_kernel-based_2007}
X.~Xu, D.~Hu, and X.~Lu, ``Kernel-{Based} {Least} {Squares} {Policy}
  {Iteration} for {Reinforcement} {Learning},'' \emph{IEEE Transactions on
  Neural Networks}, vol.~18, no.~4, pp. 973--992, 2007.

\bibitem{engel_reinforcement_2005}
Y.~Engel, S.~Mannor, and R.~Meir, ``Reinforcement {Learning} with {Gaussian}
  {Processes},'' in \emph{Proceedings of the 22nd {International} {Conference}
  {On} {Machine} {Learning}}.\hskip 1em plus 0.5em minus 0.4em\relax ACM, 2005,
  pp. 201--208.

\bibitem{taylor_kernelized_2009}
\BIBentryALTinterwordspacing
G.~Taylor and R.~Parr, ``Kernelized {Value} {Function} {Approximation} for
  {Reinforcement} {Learning},'' in \emph{Proceedings of the 26th
  {International} {Conference} on {Machine} {Learning}}, ser. {ICML} '09.\hskip
  1em plus 0.5em minus 0.4em\relax ACM, 2009, pp. 1017--1024, event-place:
  Montreal, Quebec, Canada. [Online]. Available:
  \url{http://doi.acm.org/10.1145/1553374.1553504}
\BIBentrySTDinterwordspacing

\bibitem{kroemer_non-parametric_2011}
\BIBentryALTinterwordspacing
O.~B. Kroemer and J.~R. Peters, ``A {Non}-{Parametric} {Approach} to {Dynamic}
  {Programming},'' in \emph{Advances in {Neural} {Information} {Processing}
  {Systems}}.\hskip 1em plus 0.5em minus 0.4em\relax Curran Associates, Inc.,
  2011, pp. 1719--1727. [Online]. Available:
  \url{http://papers.nips.cc/paper/4182-a-non-parametric-approach-to-dynamic-programming.pdf}
\BIBentrySTDinterwordspacing

\bibitem{borrelli_predictive_2017}
F.~Borrelli, A.~Bemporad, and M.~Morari,
  \emph{\BIBforeignlanguage{en}{Predictive {Control} for {Linear} and {Hybrid}
  {Systems}}}.\hskip 1em plus 0.5em minus 0.4em\relax Cambridge University
  Press, Jun. 2017, google-Books-ID: 7NUoDwAAQBAJ.

\bibitem{nadaraya_estimating_1964}
E.~A. Nadaraya, ``On {Estimating} {Regression},'' \emph{Theory of Probability
  \& Its Applications}, vol.~9, no.~1, pp. 141--142, 1964.

\bibitem{watson_smooth_1964}
G.~S. Watson, ``Smooth {Regression} {Analysis},'' \emph{Sankhyā: The Indian
  Journal of Statistics, Series A}, pp. 359--372, 1964.

\bibitem{fan_design-adaptive_1992}
J.~Fan, ``Design-{Adaptive} {Nonparametric} {Regression},'' \emph{Journal of
  the American Statistical Association}, vol.~87, no. 420, pp. 998--1004, 1992.

\bibitem{wasserman_all_2006}
\BIBentryALTinterwordspacing
L.~Wasserman, \emph{All of {Nonparametric} {Statistics}}.\hskip 1em plus 0.5em
  minus 0.4em\relax Springer, 2006. [Online]. Available:
  \url{https://books.google.it/books?hl=it&lr=&id=MRFlzQfRg7UC&oi=fnd&pg=PA2&dq=wasserman+2006+all&ots=SPSQp53XJz&sig=R9JPan0NnS8GkezXCj85U2ndFmc#v=onepage&q=wasserman%202006%20all&f=false}
\BIBentrySTDinterwordspacing

\bibitem{tosatto_upper_2020}
S.~Tosatto, R.~Akrour, and J.~Peters, ``An {Upper} {Bound} of the {Bias} of
  {Nadaraya}-{Watson} {Kernel} {Regression} under {Lipschitz} {Assumptions},''
  \emph{arXiv preprint arXiv:2001.10972}, 2020.

\bibitem{peters_relative_2010}
J.~Peters, K.~Mulling, and Y.~Altun, ``Relative {Entropy} {Policy} {Search},''
  in \emph{Twenty-{Fourth} {AAAI} {Conference} on {Artificial} {Intelligence}},
  2010.

\bibitem{schulman_proximal_2017}
J.~Schulman, F.~Wolski, P.~Dhariwal, A.~Radford, and O.~Klimov, ``Proximal
  {Policy} {Optimization} {Algorithms},'' \emph{arXiv preprint
  arXiv:1707.06347}, 2017.

\bibitem{metelli_policy_2018}
\BIBentryALTinterwordspacing
A.~M. Metelli, M.~Papini, F.~Faccio, and M.~Restelli, ``Policy {Optimization}
  via {Importance} {Sampling},'' in \emph{Advances in {Neural} {Information}
  {Processing} {Systems}}.\hskip 1em plus 0.5em minus 0.4em\relax Curran
  Associates, Inc., 2018, pp. 5442--5454. [Online]. Available:
  \url{http://papers.nips.cc/paper/7789-policy-optimization-via-importance-sampling.pdf}
\BIBentrySTDinterwordspacing

\bibitem{chua_deep_2018}
K.~Chua, R.~Calandra, R.~McAllister, and S.~Levine, ``Deep {Reinforcement}
  {Learning} in a {Handful} of {Trials} using {Probabilistic} {Dynamics}
  {Models},'' in \emph{Advances in {Neural} {Information} {Processing}
  {Systems}}.\hskip 1em plus 0.5em minus 0.4em\relax Curran Associates, Inc.,
  2018, pp. 4754--4765.

\bibitem{antos_fitted_2007}
A.~Antos, R.~Munos, and C.~Szepesvari, ``Fitted {Q}-{Iteration} in {Continuous}
  {Action}-{Space} {MDPs},'' in \emph{Neural {Information} {Processing}
  {Systems}}, 2007.

\bibitem{kroemer_kernel-based_2012}
O.~Kroemer, E.~Ugur, E.~Oztop, and J.~Peters, ``A {Kernel}-{Based} {Approach}
  to {Direct} {Action} {Perception},'' in \emph{International {Conference} on
  {Robotics} and {Automation}}.\hskip 1em plus 0.5em minus 0.4em\relax IEEE,
  2012, pp. 2605--2610.

\bibitem{baxter_infinite-horizon_2001}
J.~Baxter and P.~L. Bartlett, ``Infinite-{Horizon} {Policy}-{Gradient}
  {Estimation},'' \emph{Journal of Artificial Intelligence Research}, vol.~15,
  pp. 319--350, 2001.

\bibitem{kakade_natural_2001}
S.~Kakade, ``A {Natural} {Policy} {Gradient},'' in \emph{Proceedings of the
  14th {International} {Conference} on {Neural} {Information} {Processing}
  {Systems}}, 2001, pp. 1531--1538.

\bibitem{peters_natural_2008}
J.~Peters and S.~Schaal, ``Natural {Actor}-{Critic},'' \emph{Neurocomputing},
  vol.~71, no. 7-9, pp. 1180--1190, 2008, publisher: Elsevier.

\bibitem{owen_monte_2013}
A.~B. Owen, \emph{Monte {Carlo} {Theory}, {Methods} and {Examples}}, 2013.

\bibitem{sutton_emphatic_2016}
R.~S. Sutton, A.~R. Mahmood, and M.~White, ``An {Emphatic} {Approach} to the
  {Problem} of {Off}-{Policy} {Temporal}-{Difference} {Learning},'' \emph{The
  Journal of Machine Learning Research}, vol.~17, no.~1, pp. 2603--2631, 2016,
  publisher: JMLR. org.

\bibitem{liu_breaking_2018}
Q.~Liu, L.~Li, Z.~Tang, and D.~Zhou, ``Breaking the {Curse} of {Horizon}:
  {Infinite}-{Horizon} {Off}-{Policy} {Estimation},'' in \emph{Advances in
  {Neural} {Information} {Processing} {Systems}}, 2018, pp. 5356--5366.

\bibitem{liu_off-policy_2019}
\BIBentryALTinterwordspacing
Y.~Liu, A.~Swaminathan, A.~Agarwal, and E.~Brunskill, ``Off-{Policy} {Policy}
  {Gradient} with {State} {Distribution} {Correction},''
  \emph{arXiv:1904.08473}, 2019, arXiv: 1904.08473. [Online]. Available:
  \url{http://arxiv.org/abs/1904.08473}
\BIBentrySTDinterwordspacing

\bibitem{nachum_algaedice:_2019}
O.~Nachum, B.~Dai, I.~Kostrikov, Y.~Chow, L.~Li, and D.~Schuurmans,
  ``{AlgaeDICE}: {Policy} {Gradient} from {Arbitrary} {Experience},''
  \emph{arXiv:1912.02074v1}, 2019.

\bibitem{argenson_model-based_2020}
A.~Argenson and G.~Dulac-Arnold, ``Model-{Based} {Offline} {Planning},'' in
  \emph{Proceeding of the 9th {International} {Conference} on {Learning}
  {Representations}}, 2020.

\bibitem{lowrey2018plan}
K.~Lowrey, A.~Rajeswaran, S.~Kakade, E.~Todorov, and I.~Mordatch, ``Plan
  {Online}, {Learn} {Offline}: {Efficient} {Learning} and {Exploration} via
  {Model}-{Based} {Control},'' \emph{arXiv preprint arXiv:1811.01848}, 2018.

\bibitem{brockman_openai_2016}
\BIBentryALTinterwordspacing
G.~Brockman, V.~Cheung, L.~Pettersson, J.~Schneider, J.~Schulman, J.~Tang, and
  W.~Zaremba, ``{OpenAI} {Gym},'' \emph{arXiv:1606.01540}, 2016, arXiv:
  1606.01540. [Online]. Available: \url{http://arxiv.org/abs/1606.01540}
\BIBentrySTDinterwordspacing

\bibitem{fu_d4rl_2020}
J.~Fu, A.~Kumar, O.~Nachum, G.~Tucker, and S.~Levine, ``D4rl: {Datasets} for
  {Deep} {Data}-{Driven} {Reinforcement} {Learning},'' \emph{arXiv preprint
  arXiv:2004.07219}, 2020.

\bibitem{todorov2012mujoco}
E.~Todorov, T.~Erez, and Y.~Tassa, ``Mu{Jo}{Co}: A {Physics} {Engine} for
  {Model}-{Based} {Control},'' in \emph{2012 IEEE/RSJ International Conference
  on Intelligent Robots and Systems}.\hskip 1em plus 0.5em minus 0.4em\relax
  IEEE, 2012, pp. 5026--5033.

\bibitem{shelton_policy_2013}
C.~R. Shelton, ``Policy {Improvement} for {POMDPs} {Using} {Normalized}
  {Importance} {Sampling},'' \emph{arXiv preprint arXiv:1301.2310}, 2013.

\bibitem{rubinstein_simulation_2016}
R.~Y. Rubinstein and D.~P. Kroese, \emph{Simulation and the {Monte} {Carlo}
  {Method}}.\hskip 1em plus 0.5em minus 0.4em\relax John Wiley \& Sons, 2016,
  vol.~10.

\bibitem{jie_connection_2010}
T.~Jie and P.~Abbeel, ``On a {Connection} {Between} {Importance} {Sampling} and
  the {Likelihood} {Ratio} {Policy} {Gradient},'' in \emph{Advances in {Neural}
  {Information} {Processing} {Systems}}, 2010, pp. 1000--1008.

\bibitem{thomas_bias_2014}
P.~Thomas, ``Bias in {Natural} {Actor}-{Critic} {Algorithms},'' in
  \emph{International {Conference} on {Machine} {Learning}}, 2014, pp.
  441--448.

\bibitem{nota_is_2020}
C.~Nota and P.~S. Thomas, ``Is the {Policy} {Gradient} a {Gradient}?'' in
  \emph{Proceedings of the 19th {International} {Conference} on {Autonomous}
  {Agents} and {Multiagent} {Systems}}, 2020.

\bibitem{deramo_mushroomrl_2020}
\BIBentryALTinterwordspacing
C.~D'Eramo, D.~Tateo, A.~Bonarini, M.~Restelli, and J.~Peters,
  \emph{{MushroomRL}: {Simplifying} {Reinforcement} {Learning} {Research}},
  2020, publication Title: arXiv preprint arXiv:2001.01102. [Online].
  Available: \url{https://github.com/MushroomRL/mushroom-rl}
\BIBentrySTDinterwordspacing

\bibitem{kober_policy_2009}
J.~Kober and J.~R. Peters, ``Policy {Search} for {Motor} {Primitives} in
  {Robotics},'' in \emph{Advances in {Neural} {Information} {Processing}
  {Systems}}, 2009, pp. 849--856.

\end{thebibliography}

%\bibliography{references}
\vspace{-5em}
\begin{IEEEbiography}[{\includegraphics[width=1in,height=1.25in,clip,keepaspectratio]{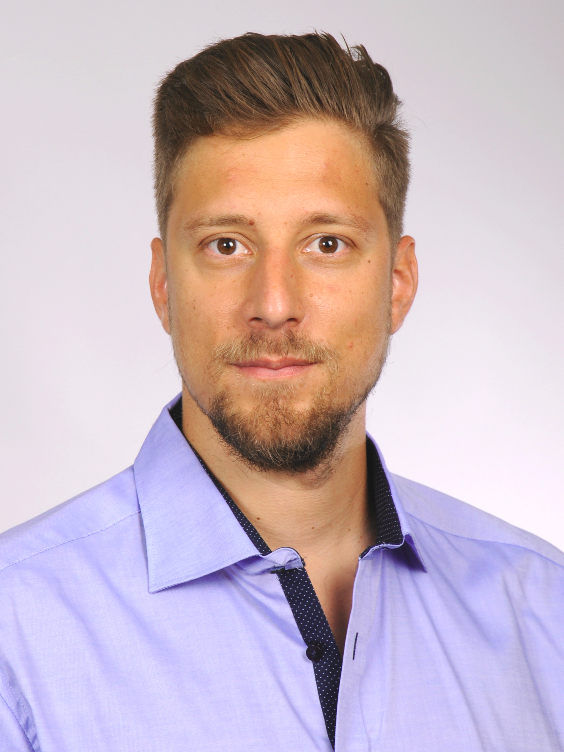}}]{Samuele Tosatto} received his Ph.D. in Computer Science from the Technical University of Darmstadt in 2020. He previously obtained his M.Sc. in the Polytechnic University of Milan. He is currently a post-doc fellow with the University of Alberta, working in the Reinforcement Learning and Artificial Intelligence group. His research interests center around reinforcement learning with a focus on its application to real robotic systems.
\end{IEEEbiography}
\vspace{-5em}
\begin{IEEEbiography}[{\includegraphics[width=1in,height=1.25in,clip,keepaspectratio]{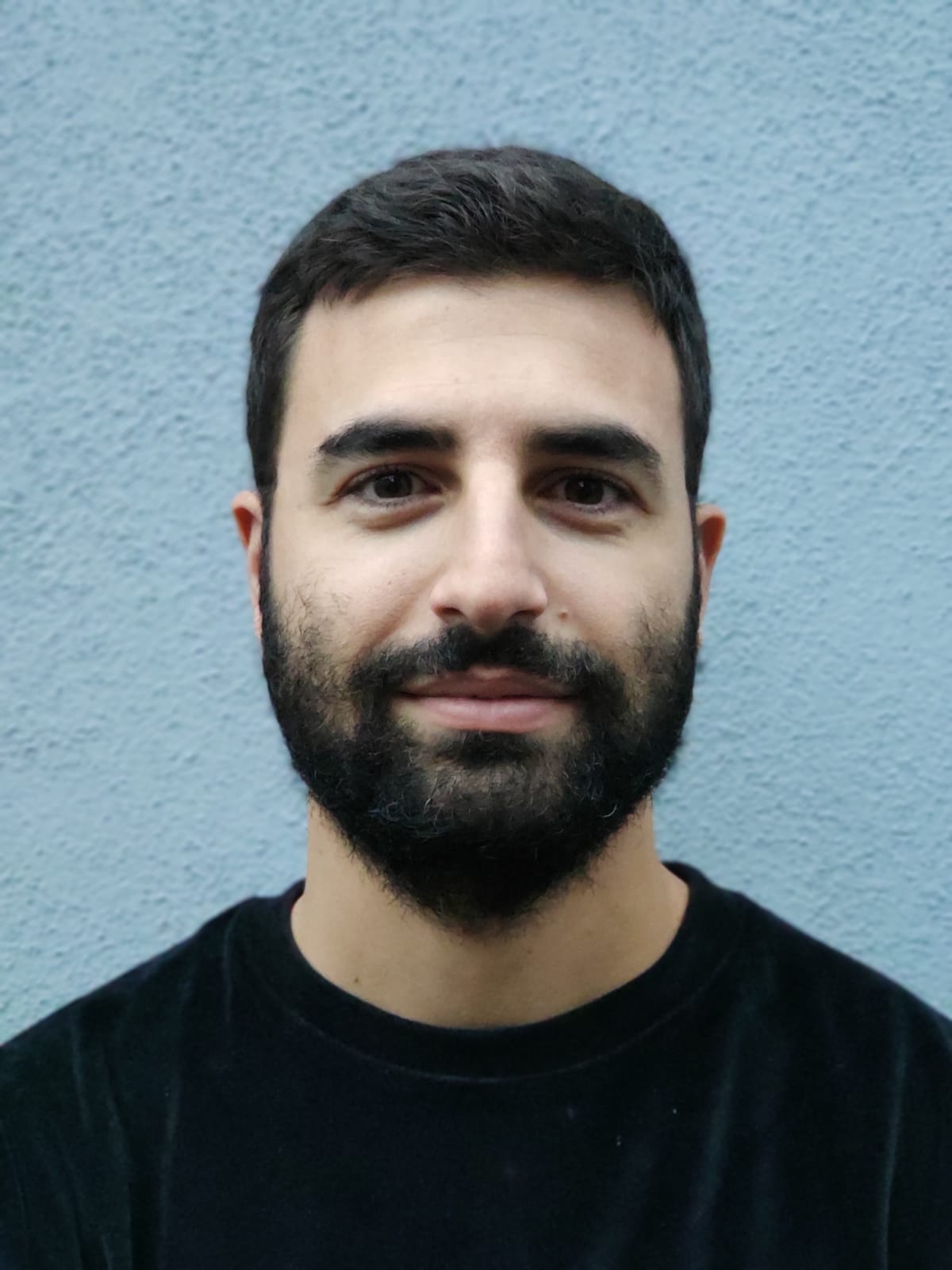}}]{Jo\~ao Carvalho}
is currently a Ph.D. student at the Intelligent Autonomous Systems group of the Technical University of Darmstadt. Previously, he completed a M.Sc. degree in Computer Science from the Albert-Ludwigs-Universit{\"a}t Freiburg, and studied Electrical and Computer Engineering at the Instituto Superior T\'{e}cnico of the University of Lisbon. His research is focused on learning algorithms for control and robotics.
\end{IEEEbiography}
\vspace{-5em}
\begin{IEEEbiography}[{\includegraphics[width=1in,height=1.25in,clip,keepaspectratio]{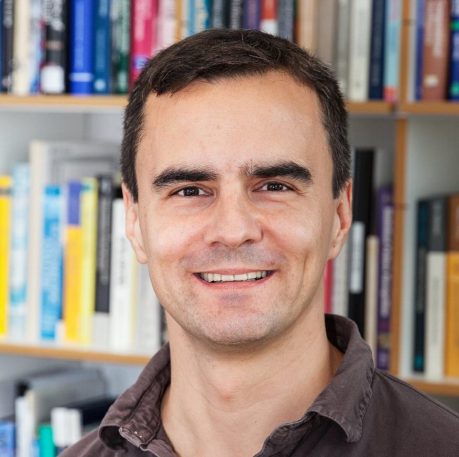}}]{Jan Peters}
is a full professor (W3) for Intelligent Autonomous Systems at the Computer
Science Department of the Technical University
of Darmstadt. He has been a senior
research scientist and group leader at the MaxPlanck Institute for Intelligent Systems, where
he headed the interdepartmental Robot Learning
Group. Jan Peters has received the Dick Volz
Best 2007 US Ph.D. 
Thesis Runner-Up Award,
the Robotics: Science \& Systems - Early Career
Spotlight, the INNS Young Investigator Award,
and the IEEE Robotics \& Automation Society's Early Career Award as
well as numerous best paper awards. In 2015, he received an ERC
Starting Grant and in 2019, he was appointed as an IEEE Fellow.
\end{IEEEbiography}

%\vfill


% UNCOMMENT TO SHOW THE BIBLIOGRAPHY
\newpage
\onecolumn
\section*{\centering \huge{Appendix}}
\vspace{1em}
\setcounter{section}{0}
\section{Support to the Theoretical Analysis}
\subsection{Existence and Uniqueness of the Fixed Point of the Nonparametric Bellman Equation}
In the following, we discuss the existence and uniqueness of the solution of the nonparametric Bellman equation (NPBE).
To do so, we will first notice that the solution must be a linear combination of the responsability vector $\bm{\varepsilon}_\pi(\state)$, and therefore any solution of the NPBE must be bounded. We will also show that such a bound is $\pm R_{\max}/(1-\gamma_c)$. 
We will use this bound to show that the solution must be unique. Subsequently, we will prove Theorem~1.
\begin{proposition}{\emph{Space of the Solution}\\}
	\label{prop:space}
	The solution of the NBPE is $\bm{\varepsilon}_\pi^\intercal(\state)\mathrm{q}$ where $\mathrm{q} \in \mathbb{R}^n$, where $n$ is the number of support points. \\
	Informal Proof: \emph{The NPBE is
	\begin{equation}
		\hat{V}_\pi(\state) = \bm{\varepsilon}_\pi^\intercal(\state) \left(\rvec + \int_\Sset \bm{\phi}_\gamma (\state')\hat{V}_\pi(\state') \de \state' \right) \label{eq:appnpbe}
	\end{equation}
	Both $\rvec$ and $\int_\Sset \bm{\phi}_\gamma (\state')\hat{V}_\pi(\state') \de \state'$ are constant w.r.t. $\state$, therefore $\hat{V}_\pi(\state)$ must be a linear respect to $\bm{\varepsilon}_\pi(\state)$}. 
\end{proposition}
A first consequence of Proposition~\ref{prop:space} is that, since $\bm{\varepsilon}_\pi(\state)$ is a vector of kernels,  $\hat{V}_\pi(\state)$ is bounded.

In Theorem~1 we state that the solution of the NPBE is $\hat{V}_\pi^*(\state) = \bm{\varepsilon}_\pi^\intercal(\state)\bm{\Lambda}_\pi^{-1}\rvec$. It is trivial to show that such a solution is a valid solution of the NPBE; however, the uniqueness of such a solution is non-trivial. In order to prove it, we will before show that if a solution to the NPBE exists, then the solution must be bounded by $[-R_{\max}/(1-\gamma_c), R_{\max}/(1-\gamma_c)]$, where $R_{\max} = \max_i |r_i|$. Note that this nice property is not common for other policy-evaluation algorithms (e.g., Neural Fitted $Q$-Iterations \cite{riedmiller_neural_2005}).
\begin{proposition}{\emph{Bound of NPBE}\\}
	If $V_\pi:\Sset\to\mathbb{R}$ is a solution to the NPBE, then $|V_\pi(\state)| \leq R_{\max}/(1-\gamma_c)$.
	\begin{proof}
		Suppose, by contradiction, that a function $f(\state)$ is the solution of a NPBE, and that $\exists \zvec \in \Sset: f(\zvec) = R_{\max}/(1-\gamma_c) + \epsilon$ where $\epsilon > 0$. Since the solution of the NPBE must be bounded, we can further assume without any loss of generality that $f(\state) \leq f(\zvec)$. Then, 
		\begin{equation}
		\frac{R_{\max}}{1-\gamma_c} + \epsilon	= \bm{\varepsilon}_\pi^\intercal(\zvec)\rvec + \bm{\varepsilon}_\pi^\intercal(\zvec)\int_\Sset \bm{\phi}_\gamma(\state')f(\state')\de \state' \nonumber.
		\end{equation}
		Since by assumption the previous equation must be fulfilled, and $\bm{\varepsilon}_\pi (\zvec)$ is a stochastic vector, $|\bm{\varepsilon}_\pi^\intercal (\zvec) \rvec| \leq R_{\max}$, we have 
		\begin{equation}
		\left| \frac{R_{\max}}{1-\gamma_c} + \epsilon	- \bm{\varepsilon}_\pi^\intercal(\zvec)\int_\Sset \bm{\phi}_\gamma(\state')f(\state')\de \state' \right| \leq R_{\max}. \label{eq:appbound}
		\end{equation}
		However, noticing also that $0 \leq \phi_i^{\pi,\gamma}(\state) \leq \gamma_c$, we have
		\begin{eqnarray}
			\frac{R_{\max}}{1-\gamma_c} + \epsilon	- \bm{\varepsilon}_\pi^\intercal(\zvec)\int_\Sset \bm{\phi}_\gamma(\state')f(\state')\de \state' & \geq & \frac{R_{\max}}{1-\gamma_c} + \epsilon	- \gamma_c \max_\state f(\state) \nonumber \\ 
			& \geq & \frac{R_{\max}}{1-\gamma_c} + \epsilon	- \gamma_c \frac{R_{\max}}{1-\gamma_c} \nonumber \\
			& = & R_{\max} + \epsilon \nonumber,
		\end{eqnarray} 
		which is in contradiction with Equation~\ref{eq:appbound}. A completely symmetric proof can be derived assuming by contradiction that $\exists \zvec \in \Sset: f(\zvec) = -R_{\max}/(1-\gamma_c) - \epsilon$ and $f(\state) \geq f(\zvec)$.
	\end{proof}
\end{proposition}
\begin{proposition}
	\label{proposition:uniqueness}
	If $\rvec$ is bounded by $R_{\text{max}}$ and if $f^*:\Sset \to \mathbb{R}$ satisfies the NPBE, then there is no other function $f:\Sset \to \mathbb{R}$ for which $\exists \zvec \in \Sset$ and $|f^*(\zvec) - f(\zvec)| > 0$.
	\begin{proof}
		Suppose, by contradiction, that exists a function $g:\Sset \to \mathbb{R}$ such that $f^*(\svec) + g(\svec)$ satisfies Equation~\ref{eq:appnpbe}. Furthermore assume that $\exists \zvec: g(\zvec) \neq 0$.
		Note that, since $f:\Sset\to\mathbb{R}$ is a solution of the NPBE, then
		\begin{equation}
		\int_{\Sset}\bm{\varepsilon}_{\pi}^\intercal(\svec)\phivec_\gamma(\svec')f^*(\svec ' )\de   \svec ' \in \mathbb{R}, \label{equation:existenceint1}
		\end{equation}
		and similarly 
		\begin{equation}
		\int_{\Sset}\bm{\varepsilon}_{\pi}^\intercal(\svec)\phivec_\gamma(\svec')\left(f^*(\svec ' ) + g(\svec')\right)\de   \svec ' \in \mathbb{R}. \label{equation:existenceint2}
		\end{equation}
		The existence of the integrals in Equations~\ref{equation:existenceint1} and \ref{equation:existenceint2} implies
		\begin{equation}
		\int_{\Sset}\bm{\varepsilon}_{\pi}^\intercal(\svec)\phivec_\gamma(\svec') g(\svec')\de   \svec ' \in \mathbb{R}. \label{equation:existenceint3}
		\end{equation}
		Note that
		\begin{eqnarray}
		|g(\svec)| & = & |f^*(\svec) - (f^*(\svec) + g(\svec))| \nonumber \\
		& = & \bigg| f^*(\svec) - \bm{\varepsilon}_{\pi}^\intercal(\svec) \bigg(\rvec + \int_{\Sset} \phivec_\gamma(\svec') \big(f^*(\svec ') + g(\svec ')\big)  \de \nonumber \svec'\bigg) \bigg| \nonumber \\
		& = & \bigg| \bm{\varepsilon}_{\pi}^\intercal(\svec) \bigg(\rvec + \int_{\Sset} \phivec_\gamma(\svec')  f^*(\svec' )  \de \nonumber \svec' \bigg)- \bm{\varepsilon}_{\pi}^\intercal(\svec) \bigg(\rvec + \int_{\Sset} \phivec_\gamma(\svec') \big(f^*(\svec ') + g(\svec ')\big)  \de \nonumber \svec'\bigg) \bigg| \nonumber \\
		& = & \bigg| \bm{\varepsilon}_{\pi}^\intercal(\svec)   \int_{\Sset} \phivec_\gamma(\svec') g(\svec ')  \de \nonumber \svec'\bigg| \nonumber \\
		& = & \bigg| \bm{\varepsilon}_{\pi}^\intercal(\svec)   \int_{\Sset} \phivec_\gamma(\svec') g(\svec ' )  \de \nonumber \svec'\bigg|. \nonumber
		\end{eqnarray}
		Using Jensen's inequality
		\begin{eqnarray}
		|g(\svec)|& \leq &  \bm{\varepsilon}_{\pi}^\intercal(\svec)   \int_{\Sset} \phivec_\gamma(\svec')  |g(\svec ') |  \de \nonumber \svec'. \nonumber
		\end{eqnarray}
		
		Since both $f^*$ and $f+g$ are bounded by $\frac{R_{\text{max}}}{1-\gamma_c}$, then $|g(\svec)| \leq A =  \frac{2R_{\text{max}}}{1-\gamma_c}$, then 
		\begin{eqnarray}
		|g(\svec)|& \leq &  \bm{\varepsilon}_{\pi}^\intercal(\svec)   \int_{\Sset} \phivec_\gamma(\svec')  |g(\svec ') |  \de  \svec' \label{equation:recursion} \\
		& \leq &  A \bm{\varepsilon}_{\pi}^\intercal(\svec)   \int_{\Sset} \phivec_\gamma(\svec')    \de \nonumber \svec' \nonumber \\
		& \leq& \gamma_c  A \nonumber .
		\end{eqnarray}
		We can iterate this reasoning now posing $|g(\state)| \leq \gamma_c A$, and eventually we notice that $|g(\state)| \leq 0$, which is in contradiction with the assumption made.
	\end{proof}
\end{proposition}
\noindent \textbf{Proof of Theorem~1}\\
\begin{proof}
	Saying that $\approxx{V}_{\pi}^*$ is a solution for Equation~\ref{eq:appnpbe} is equivalent to say
	\begin{equation}
	\approxx{V}_{\pi}^*(\svec) - \bm{\varepsilon}^\intercal_\pi(\svec) \bigg(\rvec + \gamma \int_{\Sset} \phivec_\gamma(\svec')\approxx{V}_{\pi}^*(\svec')   \de \svec'\bigg) = 0 \quad \quad \forall \svec \in \Sset. \nonumber
	\end{equation}
	We can verify that by simple algebraic manipulation
	\begin{eqnarray}
	& & \approxx{V}_{\pi}^*(\svec) - \bm{\varepsilon}_{\pi}^\intercal(\svec) \bigg(\rvec + \int_{\Sset} \phivec_\gamma(\svec')\approxx{V}_{\pi}^*(\svec')   \de \svec'\bigg) \nonumber \\
	& = & \bm{\varepsilon}_{\pi}^\intercal(\svec) \bm{\Lambda}_{\pi}^{-1} \rvec - \bm{\varepsilon}_{\pi}^\intercal(\svec) \bigg(\rvec + \int_{\Sset} \phivec_\gamma(\svec')\bm{\varepsilon}_{\pi}^\intercal(\svec') \bm{\Lambda}_{\pi}^{-1}\rvec   \de \svec'\bigg) \nonumber \\
	& = & \bm{\varepsilon}_{\pi}^\intercal(\svec)\bigg(\bm{\Lambda}_{\pi}^{-1} \rvec - \rvec - \int_{\Sset} \phivec_\gamma(\svec')\bm{\varepsilon}_{\pi}^\intercal(\svec') \bm{\Lambda}_{\pi}^{-1}\rvec   \de \svec'\bigg) \nonumber \\
	& = & \bm{\varepsilon}_{\pi}^\intercal(\svec)\Bigg(\bigg(I-  \int_{\Sset} \phivec_\gamma(\svec')\bm{\varepsilon}_{\pi}^\intercal(\svec')   \de \svec'\bigg)\bm{\Lambda}_{\pi}^{-1}\rvec - \rvec   \Bigg) \nonumber \\
	& = & \bm{\varepsilon}_{\pi}^\intercal(\svec)\Bigg(\bm{\Lambda}_{\pi}\bm{\Lambda}_{\pi}^{-1}\rvec - \rvec   \Bigg) \nonumber \\
	& = & 0.
	\end{eqnarray}
	Since equation~\ref{eq:appnpbe} has (at least) one solution, Proposition~\ref{proposition:uniqueness} guarantees that the solution ($\approxx{V}_{\pi}^*$) is unique.
\end{proof}
\subsection{Bias of the Nonparametric Bellman Equation}
In this section, we want to show the findings of Theorem~3. To do so, we introduce the infinite-samples extension of the NPBE.
\begin{proposition}
	Let us suppose to have a dataset of infinite samples, and in particular one sample for each state-action pair of the state-action space. 
	In the limit of infinite samples the NPBE defined in Definition~\ref{definition:npbe} with a data-set $\lim_{n\to \infty}D_n$ collected under distribution $\beta$ on the state-action space and MDP $\mathcal{M}$ converges to
	\begin{eqnarray}
	V_D(\svec) & = & \lim_{n \to \infty} \int_{\Aset}\frac{\sum_{i=1}^n\psi_i(\svec)\varphi_i(\avec)\bigg(r_i + \gamma_i \int_{\Sset}\phi_i(\svec')\approxx{V}_{\pi}(\svec')\de \svec\bigg)}{\sum_{j=1}^n\psi_j(\svec)\varphi_j(\avec)} \pi(\avec | \svec) \de \avec \nonumber \\
	& = & \int_{\Aset}\frac{\lim_{n \to \infty}\frac{1}{n}\sum_{i=1}^n\psi_i(\svec)\varphi_i(\avec)\bigg(r_i + \gamma_i \int_{\Sset}\phi_i(\svec')\approxx{V}_{\pi}(\svec')\de \svec\bigg)}{\lim_{n \to \infty} \frac{1}{n}\sum_{j=1}^n\psi_j(\svec)\varphi_j(\avec)} \pi(\avec | \svec) \de \avec \nonumber \\
	& = & \int_{\Aset}\frac{\int_{\Sset \times \Aset}\psi(\svec, \zvec)\varphi(\avec, \bvec)\bigg( R(\zvec, \bvec) +  \int_{\Sset}\phi(\svec', \zvec')p_\gamma(\zvec'|\bvec, \zvec)\approxx{V}_{\pi}(\svec')\de \svec  \bigg) \beta(\zvec, \bvec) \de \zvec \de \bvec}{\int_{\Sset \times \Aset}\psi(\svec, \zvec)\varphi(\avec, \bvec)\beta(\zvec, \bvec) \de \zvec \de \bvec} \pi(\avec | \svec) \de \avec. \label{eq:infinitesamples}
	\end{eqnarray}
\end{proposition}

If we impose the process generating samples to be non-degenerate distribution, and $-R_{\max} \leq R \leq R_{\max}$, we see that Propositions~1-3 remain valid.
Furthermore, from Equation~\ref{eq:infinitesamples} we are able to infer that $\mathbb{E}_D[V_D(\state)] = V_D(\state)$ (since $D$ is an infinite dataset, it does not matter if we re-sample it, the resulting value-function will be always the same).
\begin{proof}[Proof of Theorem~\ref{theorem:ultimate}]
To keep the notation uncluttered, let us introduce
\begin{equation}
	\varepsilon(\state, \action,\zvec, \bvec) = \frac{\psi(\svec, \zvec)\varphi(\avec, \bvec)\beta(\zvec, \bvec)}{\int_{\Sset \times \Aset}\psi(\svec, \zvec)\varphi(\avec, \bvec)\beta(\zvec, \bvec) \de \zvec \de \bvec}.
\end{equation}
We want to bound 
\begin{eqnarray}
\overline{V}(\state) - V^*(\state) & =& \int_\Aset \bigg( \int \varepsilon(\state, \action,\zvec, \bvec)R(\zvec,\bvec)\de\zvec\de\bvec + \int_{\Sset\times\Aset}\varepsilon(\state, \action,\zvec, \bvec) \int_{\Sset}\phi_i(\svec', \zvec')p_\gamma(\zvec'|\bvec, \zvec)\overline{V}(\svec')\de\state' \de \zvec \de \bvec \nonumber  \\
& & \quad \quad  - R(\state, \action) - \int_\Sset V^*(\state')p_\gamma(\state'|\state,\action)\de \state'\bigg)\pi(\action|\state) \de \action \nonumber \\
\implies \left|\overline{V}(\state) - V^*(\state)\right| & \leq & \max_\action\left|\int \varepsilon(\state, \action,\zvec, \bvec)\left(R(\zvec,\bvec)- R(\state, \action)\right)\de\zvec\de\bvec \right| \\
& & \quad \quad + \max_\action\left|\int_{\Sset}p_\gamma(\zvec'|\bvec, \zvec)\left(\phi_i(\svec', \zvec')\overline{V}(\svec') -V^*(\state')\right) \de \zvec \de \bvec \right| \nonumber \\
& \leq & \max_\action\left| \underbrace{ \int \varepsilon(\state, \action,\zvec, \bvec)\left(R(\zvec,\bvec)- R(\state, \action)\right)\de\zvec\de\bvec}_A \right| \\
& & \quad \quad + \gamma_c\max_\action\left| \underbrace{\int_{\Sset\times\Aset}\varepsilon(\state, \action,\zvec, \bvec)\int_{\Sset}p(\zvec'|\bvec, \zvec)\left(\phi_i(\svec', \zvec')\overline{V}(\svec') -V^*(\state')\right)\de \state' \de \zvec \de \bvec}_B \right| \nonumber \\
&= & \text{A}_\text{Bias} + \gamma_c \text{B}_\text{Bias}
\end{eqnarray}
Term A is the bias of Nadaraya-Watson kernel regression, as it is possible to observe in  \cite{tosatto_upper_2020}, therefore Theorem~\ref{theorem:biasnadaraya} applies
\begin{eqnarray}
\text{A}_\text{Bias} & = & \quad \frac{L_R \sum_{k=1}^d \hvec_k \Bigg(\prod_{i\neq k}^d \e{\frac{L_{\beta}^2 \hvec_i^2}{2}}\Bigg(1 +  \erf\bigg(\frac{\hvec_i L_{\beta}}{\sqrt{2}}\bigg) \Bigg)\Bigg)  \Bigg( \frac{1}{\sqrt{2 \pi} } + L_{\beta}\hvec_k \frac{\e{\frac{L_{\beta}^2 \hvec_k^2}{2}}}{2}\Bigg(1 + \erf\bigg(\frac{\hvec_k L_{\beta}}{\sqrt{2}}\bigg) \Bigg)\Bigg)}{\prod_{i=1}^d  \e{\frac{L_{\beta}^2 h_i^2}{2}}\Bigg(1 - \erf\bigg(\frac{\hvec_i L_{\beta}}{\sqrt{2}}\bigg) \Bigg)}\nonumber,  \label{bias}
\end{eqnarray}
where $\hvec = [\hvec_{\psi},\hvec_{\varphi}]$ and $d = d_s + d_a$.
\begin{eqnarray}
\text{B}_\text{bias}
& \leq & \gamma_c \max_{\avec}\Bigg|\frac{\int_{\Sset \times \Aset}\psi(\svec, \zvec)\varphi(\avec, \bvec)\big(\int_{\Sset\times \Sset}\overline{V}(\zvec')\phi(\zvec', \svec')p(\svec'|\svec,\avec)\de \svec'\de \zvec'- \int_{\Sset} V^*(\svec')p(\svec'|\svec, \avec)\de \svec'  \big)\beta(\zvec, \bvec)\de \zvec \de \bvec}{\int_{\Sset, \Aset}\psi(\svec, \zvec)\varphi(\avec, \bvec)\beta(\zvec, \bvec) \de \zvec \de \bvec}\Bigg| \nonumber \\
& = & \gamma_c\max_{\avec}\Bigg|\frac{\int_{\Sset \times \Aset}\psi(\svec, \zvec)\varphi(\avec, \bvec)\big(\int_{\Sset}\int_{\Sset}\big(\overline{V}(\zvec')\phi(\zvec', \svec')-V^*(\svec')\big)p(\svec'|\svec,\avec)\de \svec'\de \zvec'  \big)\beta(\zvec, \bvec)\de \zvec \de \bvec}{\int_{\Sset, \Aset}\psi(\svec, \zvec)\varphi(\avec, \bvec)\beta(\zvec, \bvec) \de \zvec \de \bvec}\Bigg| \nonumber \\
& \leq & \gamma_c\max_{\avec, \svec'}\Bigg|\frac{\int_{\Sset \times \Aset}\psi(\svec, \zvec)\varphi(\avec, \bvec)\big(\int_{\Sset}\overline{V}(\zvec')\phi(\zvec', \svec')-V^*(\svec')\de \zvec'  \big)\beta(\zvec, \bvec)\de \zvec \de \bvec}{\int_{\Sset, \Aset}\psi(\svec, \zvec)\varphi(\avec, \bvec)\beta(\zvec, \bvec) \de \zvec \de \bvec}\Bigg| \nonumber \\
& = & \gamma_c\max_{\avec, \svec'}\Bigg|\frac{\int_{\Sset \times \Aset}\psi(\svec, \zvec)\varphi(\avec, \bvec)\beta(\zvec, \bvec)\de \zvec \de \bvec}{\int_{\Sset, \Aset}\psi(\svec, \zvec)\varphi(\avec, \bvec)\beta(\zvec, \bvec) \de \zvec \de \bvec}\bigg(\int_{\Sset}\overline{V}(\zvec')\phi(\zvec', \svec')-V^*(\svec')\de \zvec'  \bigg) \Bigg|\nonumber \\
& = & \gamma_c\max_{ \svec'}\bigg|\int_{\Sset}\overline{V}(\zvec')\phi(\zvec', \svec')-V^*(\svec')\de \zvec'  \bigg|\nonumber \\
& = & \gamma_c\max_{ \svec'}\bigg|\int_{\Sset}\overline{V}(\svec'+\deltavec)\phi(\svec'+\deltavec, \svec')-V^*(\svec')\de \deltavec  \bigg|. \label{equation:finitesamples}
\end{eqnarray}
Note that
\begin{equation}
\phi(\svec' + \deltavec, \svec') = \prod_{i=1}^{d_s}\frac{e^{- \frac{\delta_i^2}{2h^2_{\phi,i}}}}{\sqrt{2\pi h^2_{\phi,i}}}, \nonumber
\end{equation}
thus, using the Lipschitz inequality,
\begin{eqnarray}
&  &  \max_{\svec'}\bigg|\int_{\Sset}\overline{V}(\svec'+\deltavec)\phi(\svec'+\deltavec, \svec')-V^*(\svec')\de \deltavec  \bigg|\nonumber\\
& \leq & \max_{\svec'}\bigg| \overline{V}(\svec')-V^*(\svec')\bigg| + \int_{\Sset} L_{V}\bigg(\sum_{i=1}^{d_s}|\delta_i|\bigg)\prod_{i=1}^{d_s}\frac{e^{- \frac{\delta_i^2}{2h^2_{\phi,i}}}}{\sqrt{2\pi h^2_{\phi,i}}}\de \deltavec \nonumber \\
& = & \max_{\svec'}\bigg| \overline{V}(\svec')-V^*(\svec')\bigg| + L_{V} \int_{\Sset} \bigg(\sum_{i=1}^{d_s}|\delta_i|\bigg)\prod_{i=1}^{d_s}\frac{e^{- \frac{\delta_i^2}{2h^2_{\phi,i}}}}{\sqrt{2\pi h^2_{\phi,i}}}\de \deltavec \nonumber \\
& =  & \max_{\svec'}\bigg| \overline{V}(\svec')-V^*(\svec')\bigg| +L_{V} \sum_{k=1}^{d_s}\bigg(\prod_{i\neq k}^{d_s}\int_{-\infty}^{+\infty} \frac{e^{- \frac{\delta_i^2}{2h^2_{\phi,i}}}}{\sqrt{2\pi h^2_{\phi,i}}}\de \delta_i\bigg)\int_{-\infty}^{+\infty}|\delta_k| \frac{e^{- \frac{\delta_k^2}{2h^2_{\phi,k}}}}{\sqrt{2\pi h^2_{\phi,k}}} \de \delta_k \nonumber \\
& =  & \max_{\svec'}\bigg| \overline{V}(\svec')-V^*(\svec')\bigg| +L_{V}2 \sum_{k=1}^{d_s}\int_{0}^{+\infty}\delta_k \frac{e^{- \frac{\delta_k^2}{2h^2_{\phi,k}}}}{\sqrt{2\pi h^2_{\phi,k}}} \de \delta_k \nonumber \\
& =  & \max_{\svec'}\bigg| \overline{V}(\svec')-V^*(\svec')\bigg| +L_{V} \sum_{k=1}^{d_s}\frac{h_{\phi,k}}{\sqrt{2 \pi}} \nonumber, 
\end{eqnarray}
which means that
\begin{eqnarray}
& & \bigg| \overline{V}(\svec)-V^*(\svec)\bigg| \leq \text{A}_\text{Bias} + \gamma_c \bigg(\max_{\svec'}\bigg| \overline{V}(\svec')-V^*(\svec')\bigg| +L_{V} \sum_{k=1}^{d_s}\frac{h_{\phi,k}}{\sqrt{2 \pi}}\bigg).  \nonumber
\end{eqnarray}
Since both $\overline{V}(\svec)$ and $V^*(\svec)$ are bounded by $-R_{\max}/(1-\gamma_c)$ and $R_{\max}(1-\gamma_c)$, then  $\big|\overline{V}(\svec)-V^*(\svec)\big| \leq 2\frac{R_{\text{max}}}{1-\gamma_c}$, thus
\begin{eqnarray}
& & \bigg| \overline{V}(\svec)-V^*(\svec)\bigg| \leq \text{A}_\text{Bias} + \gamma_c \bigg(\max_{\svec'}\bigg| \overline{V}(\svec')-V^*(\svec')\bigg| +L_{V} \sum_{k=1}^{d_s}\frac{h_{\phi,k}}{\sqrt{2 \pi}}\bigg)  \label{equation:recursion1}  \\
& & \bigg| \overline{V}(\svec)-V^*(\svec)\bigg| \leq \text{A}_\text{Bias} + \gamma_c \bigg(2\frac{R_{\text{max}}}{1-\gamma_c} +L_{V} \sum_{k=1}^{d_s}\frac{h_{\phi,k}}{\sqrt{2 \pi}}\bigg) \nonumber \\
\implies & &  \bigg| \overline{V}(\svec)-V^*(\svec)\bigg| \leq \text{A}_\text{Bias} + \gamma_c \bigg(\text{A}_\text{Bias} + \gamma_c \bigg(2\frac{R_{\text{max}}}{1-\gamma_c} +L_{V} \sum_{k=1}^{d_s}\frac{h_{\phi,k}}{\sqrt{2 \pi}}\bigg) +L_{V} \sum_{k=1}^{d_s}\frac{h_{\phi,k}}{\sqrt{2 \pi}}\bigg)  \qquad \text{using Equation~\eqref{equation:recursion1}} \nonumber \\
\implies & &  \bigg| \overline{V}(\svec)-V^*(\svec)\bigg| \leq \sum_{t=0}^{\infty}\gamma_c^t \bigg(\text{A}_\text{Bias} + \gamma_c L_{V} \sum_{k=1}^{d_s}\frac{h_{\phi,k}}{\sqrt{2 \pi}} \bigg)  \qquad \text{using Equation~\eqref{equation:recursion1}} \nonumber \\
\implies & &  \bigg| \overline{V}(\svec)-V^*(\svec)\bigg| \leq  \frac{1}{1-\gamma_c} \bigg(\text{A}_\text{Bias} + \gamma_c L_{V} \sum_{k=1}^{d_s}\frac{h_{\phi,k}}{\sqrt{2 \pi}} \bigg). \nonumber 
\end{eqnarray}
\end{proof}

\section{Support to the Empirical Analysis}

\subsection{Gradient Analysis}
The parameters used for the LQG are

\begin{eqnarray}
	A = \left[\begin{array}{cc}
	1.2 & 0 \\
	0 & 1.1
	\end{array}\right]; \quad B = \left[\begin{array}{cc}
	1 & 0 \\
	0 & 1
	\end{array}\right]; \quad Q = \left[\begin{array}{cc}
	1 & 0 \\
	0 & 1
	\end{array}\right]; \quad R = \left[\begin{array}{cc}
	0.1 & 0 \\
	0 & 0.1
	\end{array}\right]; \quad \Sigma = \left[\begin{array}{cc}
	1 & 0 \\
	0 & 1
	\end{array}\right]; \quad \state_0 = [-1, -1]. \nonumber 
\end{eqnarray}
The discount factor is $\gamma=0.9$, and the length of the episodes is $50$ steps.
The parameters of the optimization policy are $\bm{\theta} = [-0.6, -0.8]$ and the off-policy parameters are $\bm{\theta}' = [-0.35, -0.5]$.
The confidence intervals have been computed using \textsl{bootstrapped percentile intervals}. The size of the bootstrapped dataset vary from plot to plot (usually from 1000 to 5000 different seeds). The confidence intervals, instead, have been computed using 10000 bootstraps. 
We used this method instead the more classic standard error (using a $\chi^2$ or a $t$-distribution), because often, due to the importance sampling, our samples are highly non-Gaussian and heavy-tailed. The bootstrapping method relies on less assumptions, and their confidence intervals were more precise in this case. 

\subsection{Policy Improvement analysis}

We use a policy encoded as neural network with parameters $\bm{\theta}$. A deterministic policy is encoded with a neural network $\action = f_{\bm{\theta}}(\state)$. The stochastic policy is encoded as a Gaussian distribution with parameters determined by a neural network with two outputs, the mean and covariance. In this case we represent by $f_{\bm{\theta}}(\state)$ the slice of the output corresponding to the mean and by $g_{\bm{\theta}}(\state)$ the part of the output corresponding to the covariance. 
%In some algorithms we use a latent representation as part of the policy and value function networks, denoted by $\vec{z} = h_{\vec{\theta}}(\state)$, which is then linearly transformed to give the Gaussian policy parameters or to compute the value function.

NOPG can be described with the following hyper-parameters

\begin{longtable}{l l}
	\textbf{NOPG Parameters} & Meaning\\
	\hline
	dataset sizes & number of samples contained in the dataset used for training \\        
	discount factor $\gamma$ & discount factor in infinite horizon MDP \\
	state $\vec{h}_{\textrm{factor}}$ & constant used to decide the bandwidths for the state-space \\
	action $ \vec{h}_{\textrm{factor}}$ & constant used to decide the bandwidths for the action-space \\
	policy & parametrization of the policy\\
	policy output & how is the output of the policy encoded \\               
	learning rate  & the learning rate and the gradient ascent algorithm used \\
	$N_{\pi}^{\textrm{MC}}$ (NOPG-S) & number of samples drawn to compute the integral $\bm{\varepsilon}_{\pi}(\state)$ with MonteCarlo sampling \\
	$N_{\phi}^{\textrm{MC}}$ & number of samples drawn to compute the integral over the next state $\int \phi(\state')\de \state'$\\
	$N_{\mu_0}^{\textrm{MC}}$ & number of samples drawn to compute the integral over the initial \\
	&  distribution $\int \hat{V}_{\pi}(\state) \mu_{0}(\state) \de \state$ \\
	policy updates & number of policy updates before returning the optimized policy \\
	\hline
\end{longtable}

A  few considerations about NOPG parameters. If $N_{\phi}^{\textrm{MC}}=1$ we use the mean of the kernel $\phi$ as a sample to approximate the integral over the next state. When optimizing a stochastic policy represented by a Gaussian distribution, we set and linearly decay the variance over the policy optimization procedure. The kernel bandwidths are computed in two steps: first we find the best bandwidth for each dimension of the state and action spaces using cross validation; second we multiply each bandwidth by an empirical constant factor ($\vec{h}_{\textrm{factor}}$). This second step is important to guarantee that the state and action spaces do not have a zero density. For instance, in a continuous action environment, when sampling actions from a uniform grid we have to guarantee that the space between the grid points have some density. The problem of estimating the bandwidth in kernel density estimation is well studied, but needs to be adapted to the problem at hand, specially with a low number of samples. We found this approach to work well for our experiments but it can be further improved.

\subsubsection{Pendulum with Uniform Dataset}
Tables~\ref{tab:pendulum_uniform} and~\ref{tab:pendulum_uniform_configs} describe the hyper-parameters used to run the experiment shown in the first plot of Figure~\ref{figure:comparison}.

\textbf{Dataset Generation:}
The datasets have been generated using a grid over the state-action spaces $\theta, \dot{\theta}, u$, where $\theta$ and $\dot{\theta}$ are respectively angle and angular velocity of the pendulum, and $u$ is the torque applied. 
In Table~\ref{tab:pendulum_uniform} are enumerated the different datasets used.

\begin{table}[H]
\begin{center}
	\normalsize
    \begin{tabular}{c c c c}
        $\#\theta$ & $\#\dot{\theta}$ & $\#u$ & Sample size \\ \hline
        $10$ & $10$ & $2$ & $200$ \\
        $15$ & $15$ & $2$ & $450$ \\
        $20$ & $20$ & $2$ & $800$ \\
        $25$ & $25$ & $2$ & $1250$ \\
        $30$ & $30$ & $2$ & $1800$ \\
        $40$ & $40$ & $2$ & $3200$ \\
        \hline
    \end{tabular}
    \end{center}

    \caption[Pendulum uniform grid dataset configurations]{\textbf{Pendulum uniform grid dataset configurations} This table shows the level of discretization for each dimension of the state space ($\#\theta$ and $\#\dot{\theta}$) and the action space ($\#u$). Each line corresponds to a uniformly sampled dataset, where $\theta \in [-\pi, \pi]$, $\dot{\theta} \in [-8, 8]$ and $u \in [-2, 2]$. The entries under the states' dimensions and action dimension correspond to how many linearly spaced states or actions are to be queried from the corresponding intervals. The Cartesian product of states and actions dimensions is taken in order to generate the state-action pairs to query the environment transitions. The rightmost column indicates the total number of corresponding samples.}
    \label{tab:pendulum_uniform}
\end{table}

\textbf{Algorithm details:}
The configuration used for NOPG-D and NOPG-S are listed in Table~\ref{tab:pendulum_uniform_configs}.
\begin{table}[H]
\begin{center}
	\normalsize
    \begin{tabular}{l l}
        \textbf{NOPG} & \\
        \hline
        discount factor $\gamma$ & 0.99 \\
        state $\vec{h}_{\textrm{factor}}$ & $1.0$ $1.0$  $1.0$   \\
        action $ \vec{h}_{\textrm{factor}}$ & $50.0$ \\
        policy & neural network parameterized by $\bm{\theta}$\\
               & 1 hidden layer, 50 units, ReLU activations \\
        policy output & $2 \tanh(f_{\bm{\theta}}(\state))$ (NOGP-D) \\
                      & $\mu = 2 \tanh(f_{\bm{\theta}}(\state))$, $\sigma = \textrm{sigmoid}(g_{\bm{\theta}}(\state))$ (NOGP-S) \\               
        learning rate  & $10^{-2}$ with ADAM optimizer \\
        $N_{\pi}^{\textrm{MC}}$ (NOPG-S) & 15 \\
        $N_{\phi}^{\textrm{MC}}$ & 1 \\
        $N_{\mu_0}^{\textrm{MC}}$ & (non applicable) fixed initial state \\           
        policy updates & $1.5 \cdot 10^3$ \\
        \hline
    \end{tabular}
    \end{center}

    \caption[NOPG configurations for the Pendulum uniform grid experiment]{\textbf{NOPG configurations for the Pendulum uniform grid experiment}
%    This table contains the configurations of NOPG-D and NOPG-S needed to replicate the results from the Pendulum uniform grid experiment (\ref{tab:pendulum_uniform_results}).
    }
    \label{tab:pendulum_uniform_configs}
\end{table}

\subsubsection{Pendulum with Random Agent}
The following tables show the hyper-parameters used for generating the second plot starting from the left in Figure~\ref{figure:comparison}
\begin{center}
    \begin{longtable}{l l}
        \textbf{NOPG} & \\
        \hline
        dataset sizes & $10^2$, $5 \cdot 10^2$, $10^3$, $1.5 \cdot 10^3$, $2 \cdot 10^3$, $3 \cdot 10^3$, \\
                      & $5 \cdot 10^3$, $7 \cdot 10^3$, $9 \cdot 10^3$, $10^4$ \\
        discount factor $\gamma$ & $0.99$ \\
        state $\vec{h}_{\textrm{factor}}$ & $1.0$ $1.0$ $1.0$  \\
        action $ \vec{h}_{\textrm{factor}}$ & $25.0$ \\
        policy & neural network parameterized by $\bm{\theta}$\\
               & $1$ hidden layer, $50$ units, ReLU activations \\
        policy output & $2 \tanh(f_{\bm{\theta}}(\state))$ (NOGP-D) \\
                      & $\mu = 2 \tanh(f_{\bm{\theta}}(\state))$, $\sigma = \textrm{sigmoid}(g_{\bm{\theta}}(\state))$ (NOGP-S) \\               
        learning rate  & $10^{-2}$ with ADAM optimizer \\
        $N_{\pi_0}^{\textrm{MC}}$ (NOPG-S) & $10$ \\
        $N_{\phi}^{\textrm{MC}}$ & $1$ \\
        $N_{\mu_0}^{\textrm{MC}}$ & (non applicable) fixed initial state \\        
        policy updates & $2 \cdot 10^3$ \\
        \hline
        & \\ & \\
        
        \textbf{SAC} & \\
		\hline
		discount factor $\gamma$ & $0.99$ \\
		rollout steps & $500$ \\
		actor  & neural network parameterized by $\bm{\theta}_{\textrm{actor}}$\\
		& $1$ hidden layer, $50$ units, ReLU activations \\
		actor output & $2 \tanh(u)$, $u \sim \mathcal{N}(\cdot | \mu = f_{\bm{\theta}_{\textrm{actor}}}(\state), \sigma = g_{\bm{\theta}_{\textrm{actor}}}(\state))$ \\
		actor learning rate & $10^{-3}$ with ADAM optimizer \\
		critic  & neural network parameterized by $\bm{\theta}_{\textrm{critic}}$\\
		& $2$ hidden layers, $50$ units, ReLU activations \\
		critic output & $f_{\bm{\theta}_{\textrm{critic}}}(\state, \action)$ \\
		critic learning rate & $5 \cdot 10^{-3}$ with ADAM optimizer \\        
		max replay size & $5 \cdot 10^5$ \\
		initial replay size & $128$ \\
		batch size & $64$ \\
		soft update & $\tau = 5 \cdot 10^{-3}$ \\
		policy updates & $2.5 \cdot 10^5$ \\
		\hline        
		& \\ & \\        
        
        \textbf{DDPG / TD3} & \\
        \hline
        discount factor $\gamma$ & $0.99$ \\
        rollout steps & $500$ \\
        actor  & neural network parameterized by $\bm{\theta}_{\textrm{actor}}$\\
               & $1$ hidden layer, $50$ units, ReLU activations \\
        actor output & $2 \tanh(f_{\bm{\theta}_{\textrm{actor}}}(\state))$ \\
        actor learning rate & $10^{-3}$ with ADAM optimizer \\
        critic  & neural network parameterized by $\bm{\theta}_{\textrm{critic}}$\\
                & $2$ hidden layers, $50$ units, ReLU activations \\
        critic output & $f_{\bm{\theta}_{\textrm{critic}}}(\state, \action)$ \\
        critic learning rate & $10^{-2}$ with ADAM optimizer \\        
        max replay size & $5 \cdot 10^5$ \\
        initial replay size & $128$ \\
        batch size & $64$ \\
        soft update & $\tau = 5 \cdot 10^{-3}$ \\
        policy updates & $2.5 \cdot 10^5$ \\
        \hline        
        & \\ & \\
%        
%        
%        \textbf{DDPG Offline} & \\
%        \hline
%        dataset sizes & $10^2$, $5 \cdot 10^2$, $10^3$, $2 \cdot 10^3$, $5 \cdot 10^3$, $7.5 \cdot 10^3$, \\
%                      &  $10^4$, $1.2 \cdot 10^4$, $1.5 \cdot 10^4$, $2 \cdot 10^4$, $2.5 \cdot 10^4$ \\
%        discount factor $\gamma$ & $0.97$ \\
%        actor  & neural network parameterized by $\bm{\theta}_{\textrm{actor}}$\\
%               & $1$ hidden layer, $50$ units, ReLU activations \\
%        actor output & $2 \tanh(f_{\bm{\theta}_{\textrm{actor}}}(\state))$ \\
%        actor learning rate & $10^{-2}$ with ADAM optimizer \\
%        critic  & neural network parameterized by $\bm{\theta}_{\textrm{critic}}$\\
%                & $1$ hidden layer, $50$ units, ReLU activations \\
%        critic output & $f_{\bm{\theta}_{\textrm{critic}}}(\state, \action)$ \\
%        critic learning rate & $10^{-2}$ with ADAM optimizer \\
%        soft update & $\tau = 10^{-3}$ \\
%        policy updates & $2 \cdot 10^3$ \\
%        \hline
%        & \\ & \\

        \textbf{PWIS} & \\
        \hline
        dataset sizes & $10^2$, $5 \cdot 10^2$, $10^3$, $2 \cdot 10^3$, $5 \cdot 10^3$, $7.5 \cdot 10^3$, \\
                      &  $10^4$, $1.2 \cdot 10^4$, $1.5 \cdot 10^4$, $2 \cdot 10^4$, $2.5 \cdot 10^4$ \\        
        discount factor $\gamma$ & $0.99$ \\
        policy & neural network parameterized by $\bm{\theta}$\\
               & $1$ hidden layer, $50$ units, ReLU activations \\
        policy output & $\mu = 2 \tanh(f_{\bm{\theta}}(\state))$, $\sigma = \textrm{sigmoid}(g_{\bm{\theta}}(\state))$ \\               
        learning rate  & $10^{-2}$ with ADAM optimizer \\
        policy updates & $2 \cdot 10^3$ \\
        \hline
        & \\ & \\ & \\
        
        \textbf{BEAR} & \\
        \hline
        dataset sizes & $10^2$, $2 \cdot 10^2$ $5 \cdot 10^2$, $10^3$, $2 \cdot 10^3$, $5 \cdot 10^3$, \\
        & $10^4$, $2 \cdot 10^4$, $5 \cdot 10^4$, $10^5$ \\          
        discount factor $\gamma$ & $0.99$ \\
        policy & neural network parameterized by $\bm{\theta}$\\
        & $1$ hidden layer, $50$ units, ReLU activations \\
        policy output & $\mu = 2 \tanh(f_{\bm{\theta}}(\state))$, $\sigma = \textrm{sigmoid}(g_{\bm{\theta}}(\state))$ \\               
        learning rate  & $10^{-4}$ \\
        policy updates & $1 \cdot 10^3$ \\
        \hline
        & \\ & \\ & \\
        
        \textbf{BRAC-(dual)} & \\
        \hline
        dataset sizes & $10^2$, $2 \cdot 10^2$ $5 \cdot 10^2$, $10^3$, $2 \cdot 10^3$, $5 \cdot 10^3$, \\
        & $10^4$, $2 \cdot 10^4$, $5 \cdot 10^4$, $10^5$ \\          
        discount factor $\gamma$ & $0.99$ \\
        policy & neural network parameterized by $\bm{\theta}$\\
        & $1$ hidden layer, $50$ units, ReLU activations \\
        policy output & $\mu = 2 \tanh(f_{\bm{\theta}}(\state))$, $\sigma = \textrm{sigmoid}(g_{\bm{\theta}}(\state))$ \\               
        learning rate  & $10^{-3}$ with Adam \\
        policy updates & $5 \cdot 10^4$ \\
        batch size & $\leq 256$ \\
        soft update & $\tau = 5 \cdot 10^{-3}$ \\
        \hline
        & \\ & \\ & \\
        
        \textbf{MOPO} & \\
        \hline
        dataset sizes & $10^2$, $2 \cdot 10^2$ $5 \cdot 10^2$, $10^3$, $2 \cdot 10^3$, $5 \cdot 10^3$, \\
        & $10^4$, $2 \cdot 10^4$, $5 \cdot 10^4$, $10^5$ \\          
        discount factor $\gamma$ & $0.99$ \\
        policy & neural network parameterized by $\bm{\theta}$\\
        & $1$ hidden layer, $50$ units, ReLU activations \\
        policy output & $\mu = 2 \tanh(f_{\bm{\theta}}(\state))$, $\sigma = \textrm{sigmoid}(g_{\bm{\theta}}(\state))$ \\               
        learning rate  & $3 \cdot 10^{-4}$ with Adam \\
        number of epochs & $5 \cdot 10^2$ \\
        batch size & $256$ \\
        soft update & $\tau = 5 \cdot 10^{-3}$ \\
        BNN hidden dims & $64$ \\
        BNN max epochs & $100$ \\
        BNN ensemble size & $7$ \\
        \hline
        & \\ & \\ & \\

		\textbf{MOReL} & \\
		\hline
		dataset sizes & $10^2$, $2 \cdot 10^2$ $5 \cdot 10^2$, $10^3$, $2 \cdot 10^3$, $5 \cdot 10^3$, \\
		& $10^4$, $2 \cdot 10^4$, $5 \cdot 10^4$, $10^5$ \\          
		discount factor $\gamma$ & $0.99$ \\
		policy & neural network parameterized by $\bm{\theta}$\\
		& $1$ hidden layer, $50$ units, ReLU activations \\
		policy output & $\mu = 2 \tanh(f_{\bm{\theta}}(\state))$, $\sigma = \textrm{sigmoid}(g_{\bm{\theta}}(\state))$ \\               
		step size & $0.02$ \\
		number of iterations & $1 \cdot 10^3$ \\
		dynamics hidden dims & $(128, 128)$, ReLU activations \\
		dynamics lr & $0.001$ \\
		dynamics batch-size & $256$ \\	
		dynamics fit-epochs & $25$ \\	
		dynamics num-models & $4$ \\
		\hline        

        \caption[Algorithms configurations for the Pendulum random data experiment]{\textbf{Algorithms configurations for the Pendulum random data experiment}}
        \label{tab:pendulum_random_configs}
            
    \end{longtable}
\end{center}

\subsubsection{Cart-pole with Random Agent}
The following tables show the hyper-parameters used to generate the third plot in Figure~\ref{figure:comparison}.
\begin{center}
    \begin{longtable}{l l}
        \textbf{NOPG} & \\
        \hline
        dataset sizes & $10^2$, $2.5 \cdot 10^2$, $5 \cdot 10^2$, $10^3$, $1.5 \cdot 10^3$, $2.5 \cdot 10^3$,  \\
                      & $3 \cdot 10^3$, $5 \cdot 10^3$, $6 \cdot 10^3$, $8 \cdot 10^3$, $10^4$ \\        
        discount factor $\gamma$ & $0.99$ \\
        state $\vec{h}_{\textrm{factor}}$ & $1.0$ $1.0$ $1.0$  \\
        action $ \vec{h}_{\textrm{factor}}$ & $20.0$ \\
        policy & neural network parameterized by $\bm{\theta}$\\
               & $1$ hidden layer, $50$ units, ReLU activations \\
        policy output & $5 \tanh(f_{\bm{\theta}}(\state))$ (NOGP-D) \\
                      & $\mu = 5 \tanh(f_{\bm{\theta}}(\state))$, $\sigma = \textrm{sigmoid}(g_{\bm{\theta}}(\state))$ (NOGP-S) \\               
        learning rate  & $ 10^{-2}$ with ADAM optimizer \\
        $N_{\pi}^{\textrm{MC}}$ (NOPG-S) & $10$ \\
        $N_{\phi}^{\textrm{MC}}$ & $1$ \\
        $N_{\mu_0}^{\textrm{MC}}$ & $15$ \\
        policy updates & $2 \cdot 10^3$ \\
        \hline
        & \\ & \\
        
        \textbf{SAC} & \\
        \hline
        discount factor $\gamma$ & $0.99$ \\
        rollout steps & $10000$ \\
        actor  & neural network parameterized by $\bm{\theta}_{\textrm{actor}}$\\
        & $1$ hidden layer, $50$ units, ReLU activations \\
        actor output & $5 \tanh(u)$, $u \sim \mathcal{N}(\cdot | \mu = f_{\bm{\theta}_{\textrm{actor}}}(\state), \sigma = g_{\bm{\theta}_{\textrm{actor}}}(\state))$ \\
        actor learning rate & $10^{-3}$ with ADAM optimizer \\
        critic  & neural network parameterized by $\bm{\theta}_{\textrm{critic}}$\\
        & $2$ hidden layers, $50$ units, ReLU activations \\
        critic output & $f_{\bm{\theta}_{\textrm{critic}}}(\state, \action)$ \\
        critic learning rate & $5 \cdot 10^{-3}$ with ADAM optimizer \\        
        max replay size & $5 \cdot 10^5$ \\
        initial replay size & $128$ \\
        batch size & $64$ \\
        soft update & $\tau = 5 \cdot 10^{-3}$ \\
        policy updates & $2.5 \cdot 10^5$ \\
        \hline        
        & \\ & \\    
        
        \textbf{DDPG / TD3} & \\
        \hline
        discount factor $\gamma$ & $0.99$ \\
        rollout steps & $10000$ \\
        actor  & neural network parameterized by $\bm{\theta}_{\textrm{actor}}$\\
               & $1$ hidden layer, $50$ units, ReLU activations \\
        actor output & $5 \tanh(f_{\bm{\theta}_{\textrm{actor}}}(\state))$ \\
        actor learning rate & $10^{-3}$ with ADAM optimizer \\
        critic  & neural network parameterized by $\bm{\theta}_{\textrm{critic}}$\\
                & $1$ hidden layer, $50$ units, ReLU activations \\
        critic output & $f_{\bm{\theta}_{\textrm{critic}}}(\state, \action)$ \\
        critic learning rate & $10^{-2}$ with ADAM optimizer \\        
        soft update & $\tau = 10^{-3}$ \\
        policy updates & $2 \cdot 10^5$ \\
        \hline        
        & \\ & \\ & \\
%        
%        \textbf{DDPG Offline} & \\
%        \hline
%        dataset sizes & $10^2$, $5 \cdot 10^2$, $10^3$, $2 \cdot 10^3$, $3.5 \cdot 10^3$, $5 \cdot 10^3$, \\
%                      & $8 \cdot 10^3$,  $10^4$, $1.5 \cdot 10^4$, $2 \cdot 10^4$, $2.5 \cdot 10^4$ \\        
%        discount factor $\gamma$ & $0.99$ \\
%        actor  & neural network parameterized by $\bm{\theta}_{\textrm{actor}}$\\
%               & $1$ hidden layer, $50$ units, ReLU activations \\
%        actor output & $5 \tanh(f_{\bm{\theta}_{\textrm{actor}}}(\state))$ \\
%        actor learning rate & $10^{-2}$ with ADAM optimizer \\
%        critic  & neural network parameterized by $\bm{\theta}_{\textrm{critic}}$\\
%                & $1$ hidden layer, $50$ units, ReLU activations \\
%        critic output & $f_{\bm{\theta}_{\textrm{critic}}}(\state, \action)$ \\
%        critic learning rate & $10^{-2}$ with ADAM optimizer \\
%        soft update & $\tau = 10^{-3}$ \\
%        policy updates & $2 \cdot 10^3$ \\
%        \hline
%        & \\ & \\

        \textbf{PWIS} & \\
        \hline
        dataset sizes & $10^2$, $5 \cdot 10^2$, $10^3$, $2 \cdot 10^3$, $3.5 \cdot 10^3$, $5 \cdot 10^3$, \\
                      & $8 \cdot 10^3$,  $10^4$, $1.5 \cdot 10^4$, $2 \cdot 10^4$, $2.5 \cdot 10^4$ \\          
        discount factor $\gamma$ & $0.99$ \\
        policy & neural network parameterized by $\bm{\theta}$\\
               & $1$ hidden layer, $50$ units, ReLU activations \\
        policy output & $\mu = 5 \tanh(f_{\bm{\theta}}(\state))$, $\sigma = \textrm{sigmoid}(g_{\bm{\theta}}(\state))$ \\               
        learning rate  & $10^{-3}$ with ADAM optimizer \\
        policy updates & $2 \cdot 10^3$ \\
        \hline
        & \\ & \\ & \\
        
        \textbf{BEAR} & \\
        \hline
        dataset sizes & $10^2$, $2 \cdot 10^2$ $5 \cdot 10^2$, $10^3$, $2 \cdot 10^3$, $5 \cdot 10^3$, \\
        & $10^4$, $2 \cdot 10^4$, $5 \cdot 10^4$, $10^5$ \\          
        discount factor $\gamma$ & $0.99$ \\
        policy & neural network parameterized by $\bm{\theta}$\\
        & $1$ hidden layer, $50$ units, ReLU activations \\
        policy output & $\mu = 5 \tanh(f_{\bm{\theta}}(\state))$, $\sigma = \textrm{sigmoid}(g_{\bm{\theta}}(\state))$ \\               
        learning rate  & $10^{-4}$ \\
        policy updates & $1 \cdot 10^3$ \\
        \hline
        & \\ & \\ & \\
        
        \textbf{BRAC-(dual)} & \\
        \hline
        dataset sizes & $10^2$, $2 \cdot 10^2$ $5 \cdot 10^2$, $10^3$, $2 \cdot 10^3$, $5 \cdot 10^3$, \\
        & $10^4$, $2 \cdot 10^4$, $5 \cdot 10^4$, $10^5$ \\          
        discount factor $\gamma$ & $0.99$ \\
        policy & neural network parameterized by $\bm{\theta}$\\
        & $1$ hidden layer, $50$ units, ReLU activations \\
        policy output & $\mu = 5 \tanh(f_{\bm{\theta}}(\state))$, $\sigma = \textrm{sigmoid}(g_{\bm{\theta}}(\state))$ \\               
        learning rate  & $10^{-3}$ with Adam \\
        policy updates & $5 \cdot 10^4$ \\
        batch size & $\leq 256$ \\
        soft update & $\tau = 5 \cdot 10^{-3}$ \\
        \hline
        & \\ & \\ & \\        

        \textbf{MOPO} & \\
		\hline
		dataset sizes & $10^2$, $2 \cdot 10^2$ $5 \cdot 10^2$, $10^3$, $2 \cdot 10^3$, $5 \cdot 10^3$, \\
		& $10^4$, $2 \cdot 10^4$, $5 \cdot 10^4$, $10^5$ \\          
		discount factor $\gamma$ & $0.99$ \\
		policy & neural network parameterized by $\bm{\theta}$\\
		& $1$ hidden layer, $50$ units, ReLU activations \\
		policy output & $\mu = 2 \tanh(f_{\bm{\theta}}(\state))$, $\sigma = \textrm{sigmoid}(g_{\bm{\theta}}(\state))$ \\               
		learning rate  & $3 \cdot 10^{-4}$ with Adam \\
		number of epochs & $5 \cdot 10^2$ \\
		batch size & $256$ \\
		soft update & $\tau = 5 \cdot 10^{-3}$ \\
		BNN hidden dims & $32$ \\
		BNN max epochs & $100$ \\ 
		BNN ensemble size & $7$ \\
		\hline
		& \\ & \\ & \\
		
		\textbf{MOReL} & \\
		\hline
		dataset sizes & $10^2$, $2 \cdot 10^2$ $5 \cdot 10^2$, $10^3$, $2 \cdot 10^3$, $5 \cdot 10^3$, \\
		& $10^4$, $2 \cdot 10^4$, $5 \cdot 10^4$, $10^5$ \\          
		discount factor $\gamma$ & $0.99$ \\
		policy & neural network parameterized by $\bm{\theta}$\\
		& $1$ hidden layer, $50$ units, ReLU activations \\
		policy output & $\mu = 2 \tanh(f_{\bm{\theta}}(\state))$, $\sigma = \textrm{sigmoid}(g_{\bm{\theta}}(\state))$ \\               
		step size & $0.02$ \\
		number of iterations & $500$ \\
		dynamics hidden dims & $(128, 128)$, ReLU activations \\
		dynamics lr & $0.001$ \\
		dynamics batch-size & $256$ \\	
		dynamics fit-epochs & $25$ \\	
		dynamics num-models & $4$ \\
		\hline

        \caption[Algorithms configurations for the CartPole random data experiment]{\textbf{Algorithms configurations for the CartPole random data experiment}.}
        \label{tab:cartpole_random_configs}
            
    \end{longtable}
\end{center}

\subsubsection{U-Maze with D4RL}
\begin{center}
    \begin{longtable}{l l}
        \textbf{NOPG} & \\
        \hline
        dataset sizes & $10^2$, $3.5 \cdot 10^2$, $5 \cdot 10^2$, $6.5 \cdot 10^2$, $8 \cdot 10^2$, $1 \cdot 10^3$ \\        
        discount factor $\gamma$ & $0.99$ \\
        state $\vec{h}_{\textrm{factor}}$ & $2.0$  \\
        action $ \vec{h}_{\textrm{factor}}$ & $5.0$ \\
        policy & neural network parameterized by $\bm{\theta}$\\
               & $2$ hidden layers, $64$ and $32$ units, ReLU activations \\
        policy output & $5 \tanh(f_{\bm{\theta}}(\state))$ (NOGP-D) \\
                      & $\mu = 5 \tanh(f_{\bm{\theta}}(\state))$, $\sigma = \textrm{sigmoid}(g_{\bm{\theta}}(\state))$ (NOGP-S) \\               
        learning rate  & $ 3\cdot 10^{-4}$ with ADAM optimizer \\
        $N_{\pi}^{\textrm{MC}}$ (NOPG-S) & $1$ \\
        $N_{\phi}^{\textrm{MC}}$ & $1$ \\
        $N_{\mu_0}^{\textrm{MC}}$ & $50$ \\
        policy updates & $5 \cdot 10^3$ \\
        \hline
        & \\ & \\
        
        \textbf{BEAR} & \\
        \hline
        dataset sizes & $10^2$, $2 \cdot 10^2$ $5 \cdot 10^2$, $10^3$, $2 \cdot 10^3$, $5 \cdot 10^3$, \\
        & $10^4$, $2 \cdot 10^4$, $5 \cdot 10^4$, $10^5$ \\          
        discount factor $\gamma$ & $0.99$ \\
        policy & neural network parameterized by $\bm{\theta}$\\
        & $2$ hidden layer, $64$ and $32$ units, ReLU activations \\
        policy output & $\mu = 5 \tanh(f_{\bm{\theta}}(\state))$, $\sigma = \textrm{sigmoid}(g_{\bm{\theta}}(\state))$ \\               
        learning rate  & $10^{-4}$ \\
        policy updates & $10^3$ \\
        \hline
        & \\ & \\ 
        
        \textbf{BRAC-(dual)} & \\
        \hline
        dataset sizes & $10^2$, $2 \cdot 10^2$ $5 \cdot 10^2$, $10^3$, $2 \cdot 10^3$, $5 \cdot 10^3$, \\
        & $10^4$, $2 \cdot 10^4$, $5 \cdot 10^4$, $10^5$ \\          
        discount factor $\gamma$ & $0.99$ \\
        policy & neural network parameterized by $\bm{\theta}$\\
        & $2$ hidden layer, $64$ and $32$ units, ReLU activations \\
        policy output & $\mu = 5 \tanh(f_{\bm{\theta}}(\state))$, $\sigma = \textrm{sigmoid}(g_{\bm{\theta}}(\state))$ \\               
        learning rate  & $10^{-3}$ with Adam \\
        policy updates & $5 \cdot 10^4$ \\
        batch size & $\leq 256$ \\
        soft update & $\tau = 5 \cdot 10^{-3}$ \\
        \hline
        & \\ & \\ 
                
        \textbf{MOPO} & \\
        \hline
        dataset sizes & $10^2$, $2 \cdot 10^2$ $5 \cdot 10^2$, $10^3$, $2 \cdot 10^3$, $5 \cdot 10^3$, \\
        & $10^4$, $2 \cdot 10^4$, $5 \cdot 10^4$, $10^5$ \\          
        discount factor $\gamma$ & $0.99$ \\
        policy & neural network parameterized by $\bm{\theta}$\\
        & $2$ hidden layer, $64$ and $32$ units, ReLU activations \\
        policy output & $\mu = 5 \tanh(f_{\bm{\theta}}(\state))$, $\sigma = \textrm{sigmoid}(g_{\bm{\theta}}(\state))$ \\               
        learning rate  & $3 \cdot 10^{-4}$ with Adam \\
        number of epochs & $5 \cdot 10^2$ \\
        batch size & $256$ \\
        soft update & $\tau = 5 \cdot 10^{-3}$ \\
        BNN hidden dims & $64$ \\
        BNN max epochs & $100$ \\ 
        BNN ensemble size & $7$ \\
        \hline
		& \\ & \\ & \\

		\textbf{MOReL} & \\
		\hline
		dataset sizes & $10^2$, $2 \cdot 10^2$ $5 \cdot 10^2$, $10^3$, $2 \cdot 10^3$, $5 \cdot 10^3$, \\
		& $10^4$, $2 \cdot 10^4$, $5 \cdot 10^4$, $10^5$ \\          
		discount factor $\gamma$ & $0.99$ \\
		policy & neural network parameterized by $\bm{\theta}$\\
		& $(64, 32)$, ReLU activations \\
		policy output & $\mu = 5 \tanh(f_{\bm{\theta}}(\state))$, $\sigma = \textrm{sigmoid}(g_{\bm{\theta}}(\state))$ \\               
		step size & $0.005$ \\
		number of iterations & $500$ \\
		dynamics hidden dims & $(32, 32)$, ReLU activations \\
		dynamics lr & $0.001$ \\
		dynamics batch-size & $256$ \\	
		dynamics fit-epochs & $25$ \\	
		dynamics num-models & $4$ \\
		\hline

        \caption[Algorithms configurations for the CartPole random data experiment]{\textbf{Algorithms configurations for the CartPole random data experiment}.}
        \label{tab:cartpole_random_configs}
            
    \end{longtable}
\end{center}

\subsubsection{Hopper with D4RL}
\begin{center}
    \begin{longtable}{l l}
        \textbf{NOPG} & \\
        \hline
        dataset & hopper-expert-v2 \\
        dataset sizes & $3 \cdot 10^3$, $5 \cdot 10^3$, $1 \cdot 10^4$ \\        
        discount factor $\gamma$ & $0.99$ \\
        state $\vec{h}_{\textrm{factor}}$ & $4.0$  \\
        action $ \vec{h}_{\textrm{factor}}$ & $10.0$ \\
        policy & neural network parameterized by $\bm{\theta}$\\
               & $2$ hidden layers, $256$ and $256$ units, ReLU activations \\
        policy output & $\tanh(f_{\bm{\theta}}(\state))$ (NOGP-D) \\
                      & $\mu = \tanh(f_{\bm{\theta}}(\state))$, $\sigma = \textrm{sigmoid}(g_{\bm{\theta}}(\state))$ (NOGP-S) \\               
        learning rate  & $ 3 \cdot 10^{-4}$ with ADAM optimizer \\
        $N_{\pi}^{\textrm{MC}}$ (NOPG-S) & $1$ \\
        $N_{\phi}^{\textrm{MC}}$ & $1$ \\
        $N_{\mu_0}^{\textrm{MC}}$ & $500$ \\
        policy updates & $5 \cdot 10^3$ \\
        \hline
        & \\ & \\
    \end{longtable}
\end{center}

\subsubsection{Mountain Car with Human Demonstrator}
The dataset ($10$ trajectories) for the experiment in Figure~\ref{figure:mountain} has been generated by a human demonstrator, and is available in the source code provided.

\begin{center}

    \begin{longtable}{l l}
        \textbf{NOPG} & \\
        \hline
        discount factor $\gamma$ & $0.99$ \\
        state $\vec{h}_{\textrm{factor}}$ & $1.0$ $1.0$ \\
        action $ \vec{h}_{\textrm{factor}}$ & $50.0$ \\
        policy & neural network parameterized by $\vec{\theta}$\\
               & $1$ hidden layer, $50$ units, ReLU activations \\
        policy output & $ \tanh(f_{\vec{\theta}}(\state))$ (NOGP-D) \\
                      & $\mu =  \tanh(f_{\vec{\theta}}(\state))$, $\sigma = \textrm{sigmoid}(g_{\vec{\theta}}(\state))$ (NOGP-S) \\               
        learning rate  & $10^{-2}$ with ADAM optimizer \\
        $N_{\pi}^{\textrm{MC}}$ (NOPG-S) & $15$ \\
        $N_{\phi}^{\textrm{MC}}$ & $1$ \\
        $N_{\mu_0}^{\textrm{MC}}$ & $15$ \\
        policy updates & $1.5 \cdot 10^3$ \\
        \hline
        \caption[NOPG configurations for the MountainCar experiment]{\textbf{NOPG configurations for the MountainCar experiment.}}
        \label{tab:mountaincar_configs}
    \end{longtable}
    
\end{center}

\reviewB{
\subsection{Computational and Memory Complexity}
Here we detail the computational and memory complexity of NOPG. We denote $N_{\mu_0}^{\text{MC}}$ Monte-Carlo samples for expectations under the initial state distribution, $N_{\pi}^{\text{MC}}$ samples for expectations under the stochastic policy (for deterministic it is $1$), and $N_{\phi}^{\text{MC}}$ samples for the expectations of each entry of $\approxx{\mathbf{P}}_{\pi}^\gamma$. We keep the main constants throughout the analysis and drop constant factors, e.g. for the normalization of a matrix row, which involves summing up and diving each element of the row.
\begin{itemize}
	\item Constructing the vector $\bm{\varepsilon}_{\pi,0}$ takes $\bigO\left(N_{\mu_0}^{\text{MC}} N_{\pi}^{\text{MC}} n \right)$ time. Storing $\bm{\varepsilon}_{\pi,0}$ occupies $\bigO\left(N_{\mu_0}^{\text{MC}} N_{\pi}^{\text{MC}} n \right)$ memory.
	\item We compute $\approxx{\mathbf{P}}_{\pi}^\gamma$ row by row by sparsifying each row selecting the largest $k$ elements, amounting to a complexity of $\bigO\left(N_{\phi}^{\text{MC}} N_{\pi}^{\text{MC}} n^2 \log k \right)$, with $k \ll n$, since the cost of processing one row is $\bigO \left( N_{\phi}^{\text{MC}} N_{\pi}^{\text{MC}} n \log k\right)$, e.g. with a Max-Heap, and there are $n$ rows. Storing $\approxx{\mathbf{P}}_{\pi}^\gamma$ needs $\bigO \left( N_{\phi}^{\text{MC}} N_{\pi}^{\text{MC}} n k \right)$ memory.
	\item We solve the linear system of equations $\rvec = \bm{\Lambda}_\pi\mathbf{q}_{\pi}$ and $\bm{\varepsilon}_{\pi,0} = \bm{\Lambda}_{\pi}^{\intercal}\bm{\mu}_{\pi}$ for $\mathbf{q}_{\pi}$ and $\bm{\mu}_{\pi}$ using the conjugate gradient method. Both computations take $\bigO \left( \sqrt{\delta} (k+1) n \right)$, where $\delta$ is the condition number of $\bm{\Lambda}_{\pi}$ after sparsification, and $(k+1)n$ the number of nonzero elements. The \textit{plus one} comes from the computation of $\bm{\Lambda}_\pi$, since subtracting $\approxx{\mathbf{P}}_{\pi}^\gamma$ from the identity matrix leads to an increase of $n$ nonzero elements. The conjugate gradient method is specially advantageous when using sparse matrices. As a side note, computing the condition number $\delta$ is computationally intensive, but there exist methods to compute an upper bound.
	\item In computing the surrogate loss for the gradient computation, the cost of the vector-vector multiplication $\frac{\partial}{\partial \theta} \bm{\varepsilon}_{\pi, 0}^\intercal \bm{q}_{\pi} $ is $\bigO\left( N_{\mu_0}^{\text{MC}} N_{\pi}^{\text{MC}} n  \right)$, and the vector-(sparse) matrix-vector multiplication $\bm{\mu}_{\pi}^{\intercal}\bigg(\pargrad{\theta} \approxx{\mathbf{P}}_{\pi}^\gamma \bigg)\mathbf{q}_{\pi}$ is $\bigO \left(  N_{\phi}^{\text{MC}} N_{\pi}^{\text{MC}} n^2 k \right)$, thus totaling $\bigO \left( N_{\mu_0}^{\text{MC}} N_{\pi}^{\text{MC}} n +  N_{\phi}^{\text{MC}} N_{\pi}^{\text{MC}} n^2 k \right)$. Assuming the number of policy parameters $M$ to be much	lower than the number of samples, $M \ll n$, we ignore the gradient computation, since even when using finite differences, we would have $\bigO(M) \ll \bigO(n)$.
\end{itemize}

Even thought the Monte-Carlo sample terms are fixed constants, and usually set to $1$, we left these terms to emphasize that the policy parameters are inside each entry of $\bm{\varepsilon}_{\pi, 0}$ and $\approxx{\mathbf{P}}_{\pi}^\gamma$, and thus we need to keep these terms until we compute the gradient with automatic differentiation. Since modern frameworks for gradient computation, such as Tensorflow or PyTorch, build a (static or dynamic) computational graph to backpropagate gradients, we cannot simply ignore these constants in terms of time and memory complexity. The only exception is when computing $\bm{q}_{\pi}$  and $\bm{\mu}_{\pi}$ with the conjugate gradient. Since we do not need to backprogagate through this matrix, here we drop the Monte-Carlo terms because $\bm{\Lambda}_{\pi}$ is evaluated at the current policy parameters, leading to a matrix represented by $n \times k$ elements.

Taking all costs into account, we conclude that the computational complexity of NOPG \textit{per policy update} is
\begin{align*}
	& \underbrace{\bigO\left(N_{\mu_0}^{\text{MC}} N_{\pi}^{\text{MC}} n \right)}_{\bm{\varepsilon}_{\pi, 0}} + \underbrace{\bigO\left(N_{\phi}^{\text{MC}} N_{\pi}^{\text{MC}} n^2 \log k \right)}_{\approxx{\mathbf{P}}_{\pi}^\gamma} + \underbrace{\bigO \left( \sqrt{\delta} (k+1) n \right) }_{\text{conj. grad } \bm{q}_\pi, \bm{\mu}_\pi} \\
	& \quad + \underbrace{\bigO\left(N_{\mu_0}^{\text{MC}} N_{\pi}^{\text{MC}} n \right)}_{\frac{\partial}{\partial \theta} \bm{\varepsilon}_{\pi, 0}^\intercal \bm{q}_{\pi}} + \underbrace{\bigO \left(  N_{\phi}^{\text{MC}} N_{\pi}^{\text{MC}} n^2 k \right)}_{\bm{\mu}_{\pi}^{\intercal} \pargrad{\theta} \approxx{\mathbf{P}}_{\pi}^\gamma \mathbf{q}_{\pi}} \\
	& = \bigO\left(N_{\mu_0}^{\text{MC}} N_{\pi}^{\text{MC}} n \right) + \bigO\left(N_{\phi}^{\text{MC}} N_{\pi}^{\text{MC}} n^2 (k + \log k) \right) + \bigO \left( \sqrt{\delta} (k+1) n \right).
\end{align*}

The memory complexity is
\begin{equation*}
	\underbrace{\bigO\left(N_{\mu_0}^{\text{MC}} N_{\pi}^{\text{MC}} n \right)}_{\bm{\varepsilon}_{\pi, 0}}  +  \underbrace{\bigO\left(N_{\phi}^{\text{MC}} N_{\pi}^{\text{MC}} n k \right)}_{\approxx{\mathbf{P}}_{\pi}^\gamma} + \underbrace{\bigO\left( n \right)}_{\bm{q}_\pi} + \underbrace{\bigO\left( n \right)}_{\bm{\mu}_\pi}.
\end{equation*}

The quantities in ``underbrace'' indicate the source of the complexities as described in the previous paragraphs.
Hence, after dropping task-specific constants, the algorithm to implement NOPG has close to quadratic computational complexity and linear memory complexity with respect to the number of samples $n$, per policy update.}

\end{document}